\documentclass[twoside,11pt]{article}

\usepackage{jmlr2e}

\usepackage{amsmath}
\usepackage{mathtools}
\usepackage{thmtools,thm-restate}
\usepackage{booktabs}

\newtheorem{proposition}{Proposition}
\newtheorem{lemma}[proposition]{Lemma} 
\newtheorem{remark}{Remark}
\newtheorem{corollary}{Corollary}

\usepackage{enumitem}
\usepackage{xcolor}
\usepackage{url}

\usepackage{url}

\usepackage[ruled, lined, linesnumbered, commentsnumbered, longend]{algorithm2e}
\usepackage{relsize}
\SetKwInOut{KwIn}{Input}
\SetKwInOut{KwOut}{Output}

\SetCommentSty{mycommfont}

\newcommand{\BlackBox}{\rule{1.5ex}{1.5ex}}
\newenvironment{proof}{\par\noindent{\bf Proof\ }}{\hfill\BlackBox\\[2mm]}

\newcommand{\sq}[2][0]{
  \mbox{$\medmuskip=#1mu\displaystyle#2$}%
}

\DeclareMathAlphabet\mathbfcal{OMS}{cmsy}{b}{n}

\newcommand{\mat}[1]{\mathbf{#1}}

\ShortHeadings{Self-Healing Robust Neural Networks via Closed-Loop Control}{Zhuotong Chen}
\firstpageno{1}

\begin{document}

\title{Self-Healing Robust Neural Networks via Closed-Loop Control}


\author{\name Zhuotong Chen \email ztchen@ucsb.edu \\
      \addr Department of Electrical and Computer Engineering\\
      University of California at Santa Barbara, \\
      Santa Barbara, CA, USA
      \AND
      \name Qianxiao Li \email qianxiao@nus.edu.sg \\
      \addr Department of Mathematics, \\ National University of Singapore \\ Singapore, 119076
      \AND
      \name Zheng Zhang \email zhengzhang@ece.ucsb.edu \\
      \addr Department of Electrical and Computer Engineering\\
      University of California at Santa Barbara, \\
      Santa Barbara, CA, USA}


\maketitle

\begin{abstract}
Despite the wide applications of neural networks, there have been increasing concerns about their vulnerability issue. 
While numerous attack and defense techniques have been developed, this work investigates the robustness issue from a new angle: can we design a self-healing neural network that can automatically detect and fix the vulnerability issue by itself? 
A typical self-healing mechanism is the immune system of a human body. This biology-inspired idea has been used in many engineering designs, but is rarely investigated in deep learning.  
This paper considers the post-training self-healing of a neural network, and proposes a closed-loop control formulation to automatically detect and fix the errors caused by various attacks or perturbations.
We provide a margin-based analysis to explain how this formulation can improve the robustness of a classifier.
To speed up the inference of the proposed self-healing network, we solve the control problem via improving the Pontryagin’s Maximum Principle-based solver.
Lastly, we present an error estimation of the proposed framework for neural networks with nonlinear activation functions.
We validate the performance on several network architectures against various perturbations. 
Since the self-healing method does not need {\it a-priori} information about data perturbations/attacks, it can handle a broad class of unforeseen perturbations.
\footnote{A Pytorch implementation can be found in:\url{https://github.com/zhuotongchen/Self-Healing-Robust-Neural-Networks-via-Closed-Loop-Control.git}}. 
\end{abstract}


\section{Introduction}
Despite their success in massive engineering applications, deep neural networks are found to be vulnerable to perturbations of input data, to fail to ensure fairness across sub-groups of data, and experience performance degradation when implemented on nano-scale integrated circuits with process variations. In order to ensure trustworthy AI, numerous techniques have been reported, including defense techniques such as adverserial training~\citep{ganin2016domain,allen2022feature} to improve robustness, fairness-aware training~\citep{mary2019fairness,zhang2019fairness}, and fault-tolerant computing~\citep{qiao2019fault, wang2020scalable} or hardware design~\citep{moon2019enhancing,reagen2018ares}. Most techniques optimize neural network models and computing hardware in the design phase based on some assumptions about attacks, data imbalance or hardware imperfections, but the performance of a neural network may degrade significantly when these assumptions do not hold in practical deployment.

A fundamental question is: {\it can we design a self-healing process for a given neural network to handle a broad range of unforeseen data or hardware imperfections}? 
In a figurative sense, self-healing properties can be ascribed to systems or processes, which by nature or design tend to correct any disturbances brought into them. 
For instance, in psychology, self-healing often refers to the recovery of a patient from a psychological disturbance guided by instinct only. In 
physiology, the most well-known self-healing mechanism is probably the human's immune system: B cells and T cells can work together to identify and kill many external attackers (e.g., bacteria) to maintain the health of a human body \citep{rajapakse2011emerging}. 
This idea has been applied in semiconductor chip design, where self-healing integrated circuit can automatically detect and fix the errors caused by imperfect nano-scale fabrication, noise or electromagnetic interference~\citep{tang2012low, lee2012self,goyal2011new,liu2011v,chien2012dual,keskin2010statistical,sadhu2013linearized,sun2014indirect}. 
In the context of machine learning, a self-healing process is expected to fix or mitigate some undesired issues by itself, either with or without a performance detector.

In this paper, we show that it is possible to build a self-healing neural network to achieve better robustness.
Specifically, we realize this proposal via a closed-loop control method. It has been well known that an imperceptible perturbation of an input image can cause misclassification in a well-trained neural network \citep{szegedy2013intriguing, goodfellow2014explaining}. 
Many defense methods have been proposed to address this issue, including training-based defense~\citep{madry2017towards, zhang2019theoretically, gowal2020uncovering} (such as adverserial training) which focuses on the classifier itself,
and data-based defense~\citep{song2017pixeldefend, samangouei2018defense, venugopalancountering} that exploits the underlying data information. A main drawback of the training defense techniques is that they assume a specific type of attack/perturbation {\it a-priori},
and existing data-based methods are vulnerable against specifically designed attacks \citep{athalye2018obfuscated}.
In a practical setting, it is often hard (or even impossible) to foresee the possible attacks/perturbations in advance. 
Furthermore, the input attacks/perturbations could be a combination of many types. Significantly differing from the attack-and-defense methods, self-healing does not need attack/perturbation information, and it focuses on detecting and fixing the possible errors by the neural network itself. 
This allows a neural network to handle many types of attacks/perturbations simultaneously. 

\paragraph{Contribution Summary.} The specific contributions of this paper are summarized below:
\begin{itemize}[leftmargin=*]
    \item {\bf Closed-loop control formulation and margin-based analysis for post-training self-healing.} We consider a closed-loop control formulation to achieve self-healing in the post-training stage, with a goal to improve the robustness of a given neural network under a broad class of unforeseen perturbations/attacks. This self-healing formulation has two key components: embedding functions at both input and hidden layers to detect the possible errors, and a control process to adjust the neurons to fix or mitigate these errors before making a prediction.     We investigate the working principle of the proposed control loss function,
    and reveal that it can modify the decision boundary and increase the margin of a classifier.
    
    \item {\bf Fast numerical solver for the control objective function.} The self-healing neural network is implemented via closed-loop control, and this implementation causes computing overhead in the inference. 
    In order to reduce the computing overhead, 
    we solve the Pontryagin's Maximum Principle via the method of successive approximations.
    This numerical solver allows us to handle both deep and wide neural networks.
    
    \item {\bf Theoretical error analysis.}
    We provide an error analysis of the proposed framework in the most general form by considering nonlinear dynamics with nonlinear embedding manifolds.
    The theoretical setup aligns with our algorithm implementation without simplification.
    
    \item {\bf Empirical validation on several datasets.}
    On two standard and one challenging datasets,
    we empirically verify that the proposed closed-loop control implementation of self healing can consistently improve the robustness of the pre-trained models against various perturbations.
\end{itemize}
Our preliminary result was reported in~\citep{chen2021towards}. This extended work includes the following additional contributions: a broader vision of closed-loop control, the margin-based analysis of the loss function, accelerated PMP solver, and more generic error analysis in the nonlinear setting. 

\section{An Optimal Control-based Self-Healing Neural Network Framework}
This section introduces the shared robustness issue in integrated circuits (IC) and in neural networks. 
We show that the self-healing techniques widely used in IC design can be used to improve the robustness of neural networks due to the theoretical similarities of these two seemingly disconnected domains.

\subsection{Self-Healing in IC Design}
In this work, we use ``self-healing" to describe the capability of automatically correcting (possibly after detecting) the possible errors in a neural network. 
This idea has been well studied in the IC design community to fix the errors caused by nano-scale fabrication process variations in analog, mixed-signal and digital system design~\citep{tang2012low, lee2012self,goyal2011new,liu2011v,chien2012dual,keskin2010statistical,sadhu2013linearized,sun2014indirect}.
In practice, it is hard to precisely control the geometric or material parameters in IC fabrication, which causes lots of circuit chips under-performing or even failing to work. 
To address this issue, two techniques are widely used: yield optimization and self healing. 
Yield optimization~\citep{zhang2013yield,wang2017efficient,li2006robust,cui2020chance,he2021pobo} is similar to adversarial training: it chooses the optimal circuit parameters in the design phase to minimize the failure probability assuming that an exact probability density function of the process variation is given. Self-healing, on the other hand, intends to fix the possible circuit errors in the post-design phase, without knowing the distribution of process variations. 
We have similar challenges in trustworthy neural network design: it is hard to foresee what types of attacks/perturbations will occur in the practical deployment of a neural network model, therefore post-training correction can be used to fix many potential errors beyond the capability of adversarial training.
 
 \begin{figure}
    \centering
    \includegraphics[width=\linewidth]{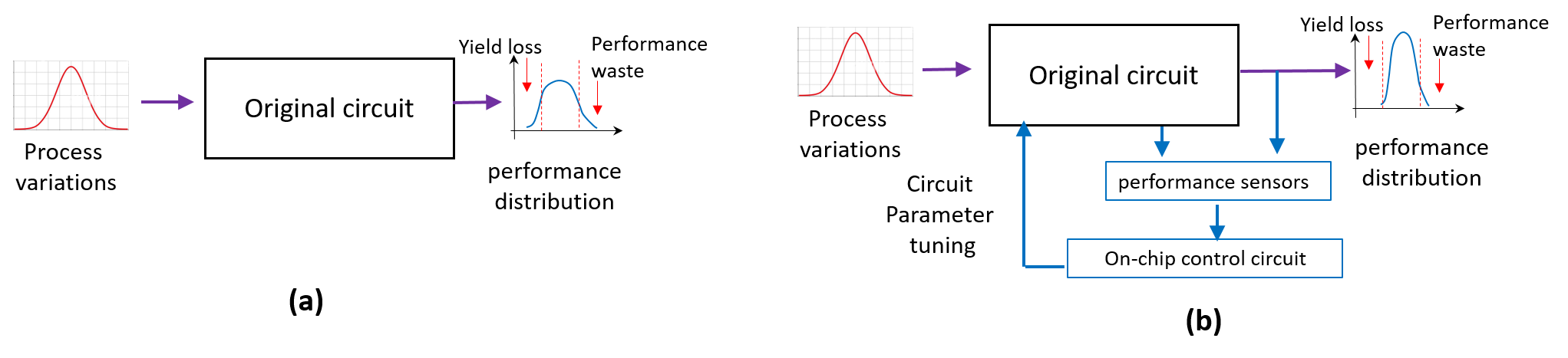}
    \caption{(a) Standard circuit design without self healing. The result can have significant yield loss and performance waste; (b) Self-healing circuit with on-chip performance monitor and control, resulting higher yield and less performance waste.}
    \label{fig: self-healing}
\end{figure}


Among various possible self-healing implementations, closed-loop control has achieved great success in practical chip design~\citep{tang2012low,lee2012self}. The key idea is shown in Fig.~\ref{fig: self-healing}. In Fig.~\ref{fig: self-healing} (a), a normal circuit, possibly after yield optimization (which tries to maximize the success rate of a circuit under various uncertainties), may still suffer from significant yield loss or performance waste due to the unpredictability of practical process variations. As shown in Fig.~\ref{fig: self-healing} (b), in order to address this issue, some on-chip global or local sensors can be added to monitor critical performance metrics. A control circuit is further added on chip to tune some circuit parameters (e.g., bias currents, supply voltage or variable capacitors) to fix the possible errors, such that the output performance distribution is adjusted to center around the desired region with higher circuit yield and less performance waste.

The same idea can be employed to design self-healing neural networks due to the following similarities between electronic circuits and neural networks:
\begin{itemize}[leftmargin=*]
    \item Electronic circuits have similar mathematical formulation with certain types of neural networks. Specifically, an electronic circuit network can be described by an ordinary differential equation (ODE) $d\mat{x}/dt=\mat{f}(\mat{x},t)$ based on modified nodal analysis~\citep{ho1975modified}, where the time-varying state variables $\mat{x}$ denote nodal voltages and branch currents. Recent studies have clearly shown that certain types of neural networks (such as residual neural networks, recurrent neural networks) can be seen as a numerical discretization of continuous ODEs~\citep{weinan2017proposal,li2017maximum,haber2017stable,chen2018neural}, and the hidden states at layer $t$ can be regarded as a time-domain snapshot of the ODE at time point $t$.  
    \item Both integrated circuits and neural networks suffer from some uncertainty issues. In IC design, the circuit performance is highly influenced by noise and process variations $\boldsymbol{\epsilon}$, resulting in a modified governining ODE $d\mat{x}(\boldsymbol{\epsilon})/dt=\mat{f}(\mat{x}(\boldsymbol{\epsilon}),\boldsymbol{\epsilon}, t)$. In neural network design, the prediction accuracy is highly influenced by data corruptions and attacks. As a result, robust design/training become important in both domains.  This issue have been handled in the design phase via robust or stochastic optimization [e.g., yield optimization in IC design~\citep{antreich1994circuit, li2004robust,cui2020chance,he2021pobo} or adversarial training in neural network design~\citep{madry2017towards, zhang2019theoretically, gowal2020uncovering}] which gets $\boldsymbol{\epsilon}$ involved in the optimization process.  Meanwhile, the reachable set computation~\citep{dang2004verification,althoff2011formal} and SAT solvers~\citep{gupta2006sat} that were widely used in circuit verification recently have achieved great success neural in network verification~\citep{gehr2018ai2,jia2020efficient}. However, many self-healing techniques (including post-design self-healing) in circuit design have not been explored in trustworthy neural network design.  
\end{itemize}
Table~\ref{table:eda_nn_analogy} has summarized the analogy of neural networks and electronic IC design. 

\begin{table}[t]
\small
 \caption{Analogy between IC design and neural networks.}
\begin{tabular}{|l l|} \toprule

Concepts in IC design                                               & Analogy in neural networks          \\ 
\midrule
circuit equation (ODE) via modified nodal analysis   & ordinary neural networks            \\ 
circuit state variables (voltages and currents)   & features at each layer  \\ 
circuit uncertainties (e.g., process variations, noise) & data corruptions, noise and attacks \\ 
circuit yield                                                  & neural network robustness           \\ 
circuit yield optimization                                     & adverserial training                \\ 
circuit verification                                    & neural network verification                \\ 
self-healing circuit                                           & self-healing neural networks       \\ 
\bottomrule
\end{tabular}
 \label{table:eda_nn_analogy}
\end{table}

\subsection{Self-Healing Robust Neural Network via Closed-Loop Control}
This work will implement post-training self healing via closed-loop control to achieve better robustness of neural networks. 
Similar to the self-healing circuit design~\citep{lee2012self,goyal2011new,liu2011v,chien2012dual,keskin2010statistical,sadhu2013linearized,sun2014indirect}, 
some performance monitors and control blocks can be added to a given $T$-layer neural network as shown in Fig.~\ref{fig: SH-NN} (a).  
Specifically, we consider residual neural networks, because they can be regarded as a forward-Euler discretization of a continuous ODE with the $t$-th layer as a time-domain snapshot at time point $t$. 
At every layer, an embedding function $\mathcal{E}_t (\cdot)$ is used to monitor the performance of a hidden layer, and controller $\boldsymbol\pi_t$ is used to adjust the neurons, such that many possible errors can be eliminated or mitigated before they propagate to the output label. 
Note that our proposed neural network architecture in Fig.~\ref{fig: SH-NN} (a) should not be misunderstood as an open-loop control. As shown in Fig.~\ref{fig: SH-NN} (b), in dynamic systems $\mat{x}_0$ (input data of a neural network) is an initial condition, and the excitation input signal is $\mat{u}_t$ 
(which is $0$ in a standard feed-forward network). 
The forward signal path is from $\mat{u}_t$ to internal states $\mat{x}_t$ and then to the label $\mat{y}$. 
The path from $\mat{x}_t$ to the embedding function $\mathcal{E}_t(\mat{x}_t)$ and then to the excitation signal $\mat{u}_t$
forms a feedback and closes the whole loop.

\begin{figure*}[t]
\centering
\includegraphics[width=\linewidth]{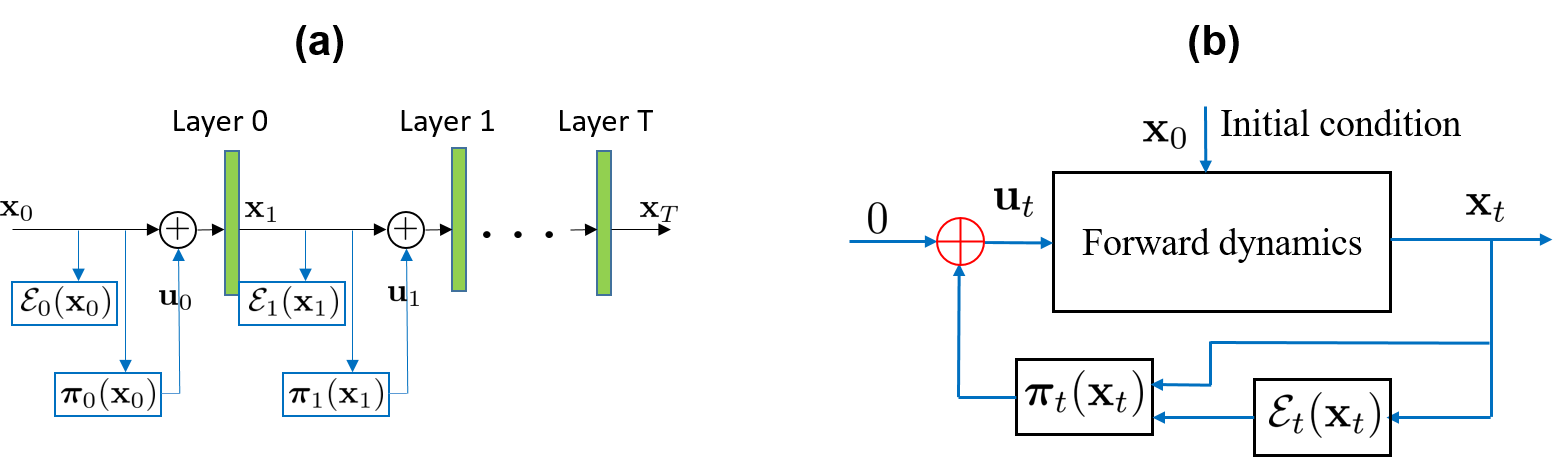}
\caption{(a) Proposed self-healing neural network; (b) Similar to the self-healing IC design shown in Fig.~\ref{fig: self-healing} (b), the proposed network forms a close-loop control from the perspective of dynamic systems.}
\label{fig: SH-NN}
\end{figure*}

Due to the closed-loop structure, the forward propagation of the proposed self-healing neural network at layer $t$ can be written as $\mat{x}_{t+1} = F_t(\mat{x}_t + \mat{u}_t)$.
Compared with standard neural networks, the proposed network needs to compute the control signals $\overline{\mat{u}} = \{ \mat{u}_{t} \}_{t=0}^{T-1}$ during inference by solving an optimal control problem:
\begin{gather}
    \min \limits_{\overline{\mat{u}}} \mathbb{E}_{(\mat{x}_0, \mat{y}) \sim \mathcal{D}} \left[ J(\mat{x}_0, \mat{y}, \overline{\mat{u}}) \right] \vcentcolon = \min \limits_{\overline{\mat{u}}} \mathbb{E}_{(\mat{x}_0, \mat{y}) \sim \mathcal{D}} ~ \Phi(\mat{x}_T, \mat{y}) + \sum_{t=0}^{T-1} {\cal L}(\mat{x}_t, \mat{u}_t, \cdot),
     \nonumber \\
    {\rm s.t.} \; \mat{x}_{t+1} = F_t(\mat{x}_t + \mat{u}_t), \; t = 0,\cdots,T-1.
    \label{eq:closed-loop control cost function}
\end{gather}
where $\Phi$ is the terminal loss, 
and ${\cal L}$ denotes a running loss that possibly depends on state $\mat{x}_t$, control $\mat{u}_t$ and some external functions.

In order to achieve better robustness via the above self-healing closed-loop control, several fundamental questions should be answered:
\begin{itemize}
    \item How shall we design the control objective function \eqref{eq:closed-loop control cost function}, such that the obtained controls can indeed correct the possible errors and improve model robustness?
    \item How can we solve the control problem efficiently, such that the extra latency is minimized in the inference?
    \item What is the working principle and theoretical performance guarantees of the self-healing neural network?
\end{itemize}
These key questions will be answered through Section~\ref{sec:control objective and margin analysis} to Section~\ref{sec:Theoretical Analysis}.



\section{Design of Self-Healing via Optimal Control}
\label{sec:control objective and margin analysis}
In this section, we propose a control objective function for self-healing robust neural networks in solving classification problems. 
With a margin-based analysis, we demonstrate that this control objective function enlarges the classification margin of decision boundary. 

\subsection{Towards Better Robustness: Control Loss via Manifold Projection}

In general, the control objective function Eq.~\eqref{eq:closed-loop control cost function} should have two parts: a terminal loss and a running loss:
\begin{itemize}[leftmargin=*]
    \item In traditional optimal control, the terminal loss $\Phi(\mat{x}_T, \mat{y})$ can be a distance measurement between the terminal state of the underlying trajectory and some destination set given beforehand.
In supervised learning, this corresponds to controlling the underlying hidden states such that the terminal state $\mat{x}_T$ (or some transformation of it) matches the true label.
This is impractical for general machine learning applications since the true label $\mat{y}$ is unknown during inference. Therefore, we ignore the terminal loss by setting it as zero.

\item When considering a deep neural network as a discretization of continuous dynamic system,
the state trajectory (all input and hidden states) governed by this continuous transformation forms a high dimensional structure embedded in the ambient state space.
The set of state trajectories that leads to ideal model performance, in the discretized analogy, can be represented as a sequence of embedding manifolds $\{ \mathcal{M}_t \}_{t=0}^{T-1}$.
The embedding manifold is defined as $\mathcal{M}_t = f_t^{-1}(\mat{0})$ for a submersion $f(\cdot):\mathbb{R}^d \rightarrow \mathbb{R}^{d-r}$.
We can track a trajectory during neural network inference and enforce it onto the desired manifold $\mathcal{M}_t $ to improve model performance.
This motivates us to design the running loss of Eq.~\eqref{eq:closed-loop control cost function} as follows,
\begin{equation}
    {\cal L}(\mat{x}_t, \mat{u}_t, f_t(\cdot)) \vcentcolon = \frac{1} {2} \lVert f_t(\mat{x}_t + \mat{u}_t) \rVert_2^2 + \frac{c} {2} \lVert \mat{u}_t \rVert_2^2.
    \label{eq:running loss}
\end{equation}
The submersion $f_t(\mat{x})$ measures the distance between a state $\mat{x}$ to the embedding manifold $\mathcal{M}_t$,
$f_t(\mat{x}) = \mat{0}$ if $\mat{x} \in \mathcal{M}_t$.
This can be understood based on the ``manifold hypothesis" \citep{fefferman2016testing},
which assumes that real-world high-dimensional data (represented as vectors in $\mathbb{R}^d$) generally lie in a low-dimensional manifold $\mathcal{M} \subset \mathbb{R}^d$. The first term in Eq.~\eqref{eq:running loss} serves as a ``performance monitor" in self healing: it measures the discrepancy between the state variable $\mat{x}_t$ and the desired manifold $\mathcal{M}_t$.
The regularization term with a hyper-parameter $c$ prevents using large controls, which is a common practice in control theory.
\item
The performance monitor can be realized by a manifold projection $\mathcal{E}_t(\cdot)$,
\begin{equation}
    \label{eq:manifold projection}
    \mathcal{E}_t(\mat{x}_t) \vcentcolon = \arg \min_{\mat{z} \in \mathcal{M}_t} \frac{1} {2} \lVert \mat{x}_t - \mat{z} \rVert_2^2.
\end{equation}
The manifold projection can be considered as a constrained optimization.
Given that $\mathcal{M}_t$ is a compact set, the solution of Eq.~\eqref{eq:manifold projection} always exists.
The submersion satisfies $\lVert f_t(\mat{x}_t) \rVert_2 = \lVert \mathcal{E}_t(\mat{x}_t) - \mat{x}_t \rVert_2$.
In practice, 
the manifold projection is realized as an auto-encoder. Specifically,
an encoder embeds a state snapshot into a lower-dimensional space,
then a decoder reconstructs this embedded data back to the ambient state space.
The auto-encoder can be obtained via minimizing the reconstruction loss on a given dataset,
\begin{gather}
    \label{eq:embedding function implementation}
    \mathcal{E}_t^{\ast}(\mathcal{M}_t, \cdot) = \arg \min \limits_{\mathcal{E}_t} \frac{1} {N} \sum_{i=1}^{N} 
    \underbrace{ {\rm CE}(\mat{x}_{i,T}, \mat{y}_i) }_{\rm model~information} + \underbrace{ \lVert \mathcal{E}_t(\mathcal{M}_t, \mat{x}_{i, t}) - \mat{x}_{i, t} \rVert_2^2 }_{\rm data~information}, \\ \nonumber
    {\rm s.t.}\;  \mat{x}_{i, t+1} = F_t(\mat{x}_{i, t}, \boldsymbol\theta_t, \mat{u}_{i, t}), ~ \mat{u}_{i, t} = \mathcal{E}_t(\mat{x}_{i, t}) - \mat{x}_{i, t},
\end{gather}
where ${\rm CE}(\cdot, \cdot)$ denotes cross-entropy loss function.
The objective function Eq.~\eqref{eq:embedding function implementation} defines a attack-agnostic setting, where only clean data and model information are accessible to the control system.
\end{itemize}


Considering the zero terminal loss and non-zero running loss, the overall  control objective function for self healing can be designed as below,
\begin{gather}
    \min \limits_{\overline{\mat{u}}} \mathbb{E}_{(\mat{x}_0, \mat{y}) \sim \mathcal{D}} ~ \sum_{t=0}^{T-1} \lVert f_t(\mat{x}_t + \mat{u}_t) \rVert_2^2 + \frac{c} {2} \lVert \mat{u}_t \rVert_2^2, \nonumber \\
    {\rm s.t.} \; \mat{x}_{t+1} = F_t(\mat{x}_t, \boldsymbol\theta_t, \mat{u}_t), \; t = 0,\cdots,T-1.
    \label{eq:final cost function}
\end{gather}
In neural network inference, the resulting control signals will help to attract the (possibly perturbed) trajectory towards the embedding manifolds.



\subsection{A Margin-based Analysis On the Running Loss}
\label{sec:margin analysis robustness}
We discuss the effectiveness of the running loss in Eq.~\eqref{eq:running loss} by considering robustness issue in deep learning.
To simplify the problem setting,
we consider a special case of the control objective function in Eq.~\eqref{eq:final cost function} where control is only applied at one layer.
Specifically,
we assume that the control is applied to the input ($T = 1$) and the applied control is not penalized ($c = 0$). The analysis in a generic $t$-th layer can be done similarly by seeing $\mat{x}_{t-1}$ as the input data.  
In this simplified setting, 
by choosing $\mathcal{M}$ as the embedding manifold in $\mathbb{R}^d$,
the optimal control results in the solution of the constrained optimization in Eq.~\eqref{eq:manifold projection}.

\paragraph{Manifold projection enlarges decision boundary}

\begin{figure}[t]
	\centering
	\begin{minipage}{.24\linewidth}
	\centering
	    \includegraphics[width = \linewidth]{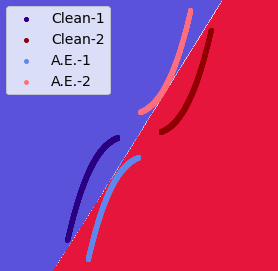}
	    (a) 
	\end{minipage}
		\begin{minipage}{.24\linewidth}
		\centering
	    \includegraphics[width=\linewidth]{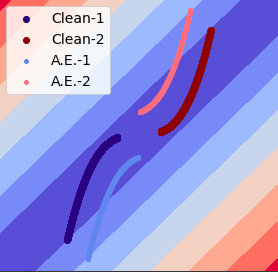}
	    (b) 
	\end{minipage}
	\begin{minipage}{.24\linewidth}
	\centering
		\includegraphics[width=\linewidth]{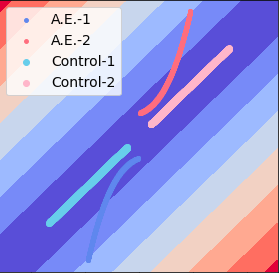}
		(c) 
	\end{minipage}
	\begin{minipage}{.24\linewidth}
	\centering
		\includegraphics[width=\linewidth]{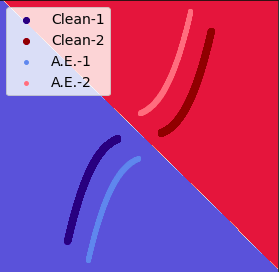}
		(d) 
	\end{minipage}
    \caption{(a): Clean-$1$ and Clean-$2$ represent clean data of class $1$ and $2$, respectively. A.E.-$1$ and A.E.-$2$ are their adversarially perturbed counterparts. (b): the reconstruction loss field. (c): the controlled counterpart. (d): Manifold projection modifies decision boundary.}
    \label{fig:Linear reactive defense}
\end{figure}

For any given input $\tilde{\mat{x}}$,
an optimal control process solves the constrained optimization Eq.~\eqref{eq:manifold projection} by
reconstructing a nearest counterpart $\mat{x} \in \mathcal{M}$.
This seemingly {\it adaptive} control process essentially forms some {\it deterministic} decision boundaries that enlarge the margin of given classifier.
In general,
an accurate classifier can have a small ``classifier margin'' measured by an $\ell_p$ norm, i.e. the minimal perturbation in $\mathbb{R}^d$ required to change the model prediction label.
This small margin can be easily exploited by adversarial attacks, such as PGD \citep{madry2017towards}.
We illustrate these phenomena via an numerical example.
Fig.~\ref{fig:Linear reactive defense} shows a binary classification problem in $\mathbb{R}^1$, where blue and red regions represent the classification predictions (their joint line represents decision boundary of the underlying classifier).
In Fig.~\ref{fig:Linear reactive defense} (a), the given classifier has accuracies of $100 \%$ and $0 \%$ on clean data and against adversarial examples respectively.
Fig.~\ref{fig:Linear reactive defense} (b) shows the reconstruction loss field, computed by $\lVert ( \mat{I} - \mat{P}) \mat{x} \rVert_2^2, \forall \mat{x} \in \mathbb{R}^2$, where $\mat{P}$ is the $\ell_2$ orthogonal projection onto the $1$-d embedding subspace $\mathcal{M}$.
As expected, clean data samples are located in the low loss regions, and 
adversarial examples fall out of $\mathcal{M}$ and have larger reconstruction losses.
In Fig.~\ref{fig:Linear reactive defense} (c), our control process adjusts adversarially perturbed data samples towards the embedding subspace $\mathcal{M}$, and the classifier predicts those with $100 \%$ accuracy.
Essentially, the manifold projection enforces those adjacent out-of-manifold samples to have the same prediction as the clean data in the manifold, and the margin of the decision boundary has been increased as shown in Fig.~\ref{fig:Linear reactive defense} (d).

\begin{remark}
In this simplified linear case, the embedding manifold $\mathcal{M}$ is the $1$-D linear subspace highlighted as the darkest blue in Fig.~\ref{fig:Linear reactive defense} (b) (c).
Specifically, any data point in this subspace incurs zero reconstruction loss.
Therefore,
the constrained optimization problem in Eq.~\eqref{eq:manifold projection} is the orthogonal projection onto a linear subspace $\mathcal{M}$,
The manifold projection reduces the pre-image of a classifier $F(\cdot)$ from $\mathbb{R}^2 \mapsto \mathbb{R}^1$.
Given a data point $\mat{x}$ sampled from this linear subspace, any out-of-manifold data $\tilde{\mat{x}}$ satisfies $\lVert \mat{P} \tilde{\mat{x}} - \mat{x} \rVert_2^2 \leq \lVert \tilde{\mat{x}} - \mat{x} \rVert_2^2$.
Consequently, the margin of $F(\cdot)$ is enlarged.
\end{remark}

\paragraph{A margin-based analysis on the manifold projection.}
Now we formally provide two definitions for margins related to classification problems.
Specifically,
we consider a classification dataset $\mathcal{D}$ belonging to the ground-truth manifold $\mathcal{M}^{\ast}$, $\mathcal{D} \subset \mathcal{M}^{\ast}$,
this enables the formal definitions of different types of margins.

\begin{itemize}[leftmargin=*]
    \item {\bf Manifold margin:}
    We define $\mathcal{R}_{\mathcal{M}}$ as the geodesics
    \begin{equation*}
        \mathcal{R}_{\mathcal{M}}(\mat{a}, \mat{b})
        \vcentcolon = 
        \inf\limits_{
            \gamma \in \Gamma_{\mathcal{M}}(\mat{a}, \mat{b}),
        }
        \int_0^1 \sqrt{ \langle\, \gamma'(t),\gamma'(t) \rangle_{\gamma(t)} } dt,
    \end{equation*}
    where $\gamma \in \Gamma_{\mathcal{M}}(\mat{a}, \mat{b})$ is a continuously differentiable curve 
    $\gamma : [0, 1] \rightarrow \mathcal{M}$ such that $\gamma(0) = \mat{a}$ and $\gamma(1) = \mat{b}$.
    Here, $\langle,\rangle_{p}$ is the positive definite inner product on the tangent space $\mathcal{T}_p \mathcal{M}$ at any point $p$ on the manifold $\mathcal{M}$.
    In other words, the distance $\mathcal{R}_{\mathcal{M}}(\mat{a}, \mat{b})$ between two points $\mat{a}$ and $\mat{b}$ of $\mathcal{M}$ is defined as the length of the shortest path connecting them.
    Given a manifold $\mathcal{M}$ and classifier $F(\cdot)$,
    the manifold margin $d_{\mathcal{M}}(F(\cdot))$ is defined as the shortest distance along $\mathcal{M}$ such that an instance of one class transforms to another.
    \begin{equation}
        \label{eq:manifold margin equation}
        d_{\mathcal{M}}(F(\cdot)) \vcentcolon = \frac{1}{2} \inf\limits_{\mat{x}_1, \mat{x}_2 \in \mathcal{D}} \mathcal{R}_{\mathcal{M}} (\mat{x}_1, \mat{x}_2), \hspace{0.3cm} {\rm s.t.} ~F(\mat{x}_1) \neq F(\mat{x}_2).
    \end{equation}
    \item {\bf Euclidean margin:}
    In practice, data perturbations are any perturbations of a small Euclidean distance (or any equivalent norm).
    The classifier margin $d_e(F(\cdot))$ is the smallest magnitude of a perturbation in $\mathbb{R}^d$ that causes the change of output predictions.
    \begin{equation}
    \label{eq:euclidean margin equation}
    d_e(F(\cdot)) \vcentcolon = \inf\limits_{\mat{x} \in \mathcal{D}} \inf\limits_{\boldsymbol\delta \in \mathbb{R}^d} \lVert \boldsymbol\delta \rVert_2, \;\;\;\; {\rm s.t.}~F(\mat{x}) \neq F(\mat{x} + \boldsymbol\delta).
    \end{equation}
\end{itemize}

In addition,
we introduce ground-truth margin and manifold projection margin from the definitions of manifold and euclidean margins respectively.
\begin{itemize}[leftmargin=*]
    \item {\bf Ground-truth margin:}
    For the ground-truth manifold $\mathcal{M}^{\ast}$ and ground-truth classifier $F^{\ast}(\cdot)$ (population risk minimizer),
    the ground-truth margin $d_{\mathcal{M}^{\ast}} (F^{\ast}(\cdot))$ [according to Eq.~\eqref{eq:manifold margin equation}] is the largest classification margin.
    \item {\bf Manifold projection margin:}
    The manifold projection Eq.~\eqref{eq:manifold projection} modifies a classifier from $F(\cdot)$ to $F\circ \mathcal{E}(\mathcal{M},\cdot)$.
    Therefore, its robustness depends on the ``manifold projection margin'' [according to Eq.~\eqref{eq:euclidean margin equation}] as
    \begin{equation*}
    d_e(F \circ \mathcal{E}(\cdot)) \vcentcolon = \inf\limits_{\mat{x} \in \mathcal{D}} \inf\limits_{\boldsymbol\delta \in \mathbb{R}^d} \lVert \boldsymbol\delta \rVert_2, \;\;\;\; {\rm s.t.}~F(\mathcal{E}(\mat{x})) \neq F(\mathcal{E}(\mat{x} + \boldsymbol\delta)).
    \end{equation*}
    A manifold projection essentially constraints the data space $\mathbb{R}^d$ into a smaller subset according to the embedding manifold $\mathcal{M} \subset \mathbb{R}^d$.
\end{itemize}

In $\mathbb{R}^d$, a binary linear classifier forms a $(d-1)$-dimensional hyperplane that partitions $\mathbb{R}^d$ into two subsets.
Let the range of $\mat{V} \in \mathbb{R}^{d \times (d-1)}$ be this hyperplane,
$\hat{\mat{n}}$ as a $d$-dimensional normal vector such that $\mat{V}^T \hat{\mat{n}} = \mat{0}$.
In general, a linear classifier with random decision boundary can be defined as setting the normal vector $\hat{\mat{n}} \sim \mathcal{N}(\mat{0}, \frac{1} {d} \mat{I})$.
In this simplified linear setting,
the following proposition provides a relationship between the euclidean margin $d_e(F (\cdot))$ and manifold margin $d_{\mathcal{M}}(F(\cdot))$.

\begin{restatable}{proposition}{linearclassifiermargin}
\label{prop: reactive defense condition}
Let $\mathcal{M} \subset \mathbb{R}^d$ be a $r$-dimensional ($r \leq d$) linear subspace that contains the ground-truth manifold $\mathcal{M}^{\ast}$, 
such that $\mathcal{M}^{\ast} \subset \mathcal{M}$,
$F(\cdot)$ a linear classifier with random decision boundary,
then $\mathbb{E} \bigg[ \frac{d_e(F(\cdot))} {d_{\mathcal{M}}(F(\cdot))} \bigg] \leq \sqrt{\frac{r} {d}}$.
\end{restatable}

The detailed proof is shown in Appendix \ref{proof: Manifold Projection Enlarges Classifier Margin In Linear Case}.


\paragraph{A demonstration of margin increase.}
\begin{figure}[t]
	\centering
	\begin{minipage}{.235\linewidth}
	\centering
	    \includegraphics[width = \linewidth]{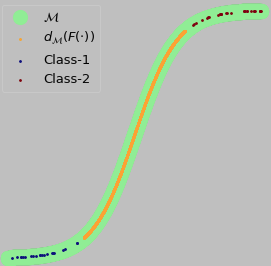}
	    (a): $2\cdot d_{\mathcal{M}}(F(\cdot))$  
	\end{minipage}
	\begin{minipage}{.235\linewidth}
	\centering
	    \includegraphics[width = \linewidth]{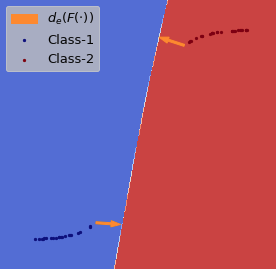}
	    (b): $d_e(F(\cdot))$  
	\end{minipage}
		\begin{minipage}{.235\linewidth}
		\centering
	    \includegraphics[width=\linewidth]{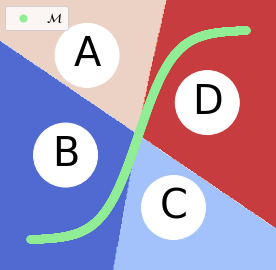}
	    (c): $\mathcal{E}(\cdot)$  
	\end{minipage}
	\begin{minipage}{.235\linewidth}
	\centering
		\includegraphics[width=\linewidth]{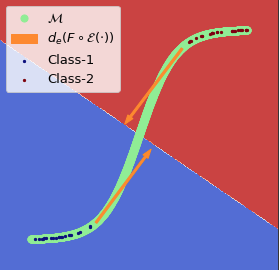}
		(d): $d_e(F \circ \mathcal{E}(\cdot))$ 
	\end{minipage}
    \caption{(a): a binary classification dataset embedded inside of a manifold $\mathcal{M}$.
    The manifold margin with a classifier $F$ is shown.
    (b): the euclidean margin with a classifier $F(\cdot)$.
    (c): the manifold projection and classifier form four partitions,
    regions A and B are projected onto the top portion of $\mathcal{M}$,
    regions C and D are projected onto lower portion of $\mathcal{M}$.
    (d): the manifold projection margin.
    }
    \label{fig:margins}
\end{figure}

Fig.~\ref{fig:margins} (a) shows a binary classification dataset embedded in a $1$-dimensional manifold ($\mathcal{M}$ is shown as green curve).
Given a classifier $F(\cdot)$,
the manifold margin $d_{\mathcal{M}} (F(\cdot))$ (orange curve shows $2 \cdot d_{\mathcal{M}} (F(\cdot))$) is shown as the shortest distance that an instance of one class transforms to another.
The underlying classifier results in a small euclidean margin as shown in Fig.~\ref{fig:margins} (b).
In Fig.~\ref{fig:margins} (c),
subsets-A and B are prediction of class-$1$, subsets-C and D are predictions of class-$2$.
The manifold projection $\mathcal{E} (\cdot)$ projects subsets-A and D onto the top portion of $\mathcal{M}$,
subsets B and C onto lower portion of $\mathcal{M}$.
The decision boundary of classifier and manifold projection form four partitions of $\mathbb{R}^2$.
For the composed classifier $F \circ \mathcal{E} (\cdot)$,
any sample in regions A and D are predicted as class-$2$,
and samples from regions B and C are predicted as class-$1$.
As a result,
Fig.~\ref{fig:margins} (d) shows the decision boundary of $F \circ \mathcal{E} (\cdot)$,
the manifold projection margin (shown in orange) is significantly improved than the euclidean margin.

\section{An Optimal Control Solver for Self-Healing}
\label{sec:control solver}
In this section, we present a general optimal control method to solve the proposed objective function in Eq.~\eqref{eq:final cost function}.
An more efficient method is proposed to reduce the inference overhead caused by generating the controls.

\subsection{Control Solver Based on the Pontryagin's Maximum Principle}
\label{sec:optimal control description}
The proposed self-healing neural network needs to solve the optimal control problem in order to predict its output label $\mat{y}$ given an input data sample $\mat{x}$. We first describe a general solver for the optimal control problem in Eq.~\eqref{eq:closed-loop control cost function} based on the Pontryagin's Maximum Principle~\citep{pontryagin1987mathematical}.


To begin with, we define the Hamiltonian $H(t, \mat{x}_t, \mat{p}_{t+1}, \boldsymbol\theta_t, \mat{u}_t)$ as
\begin{equation*}
    H(t, \mat{x}_t, \mat{p}_{t+1}, \boldsymbol\theta_t, \mat{u}_t) \vcentcolon = \mat{p}_{t+1}^{T} \cdot F_t(\mat{x}_t, \boldsymbol\theta_t, \mat{u}_t) - {\cal L}(\mat{x}_t, \mat{u}_t, f_t(\cdot)).
    \label{eq:hamiltonian}
\end{equation*}
The Pontryagin's maximum principle consists of a two-point boundary value problem,
\begin{align}
    & \mat{x}_{t+1}^{\ast} = \nabla_p H(t, \mat{x}_t^{\ast}, \mat{p}_t^{\ast}, \boldsymbol\theta_t, \mat{u}_t^{\ast}), &(\mat{x}_0^{\ast}, \mat{y}) \sim \mathcal{D},
    \label{eq:pmp forward}\\
    & \mat{p}_t^{\ast} = \nabla_x H(t, \mat{x}_t^{\ast}, \mat{p}_{t+1}^{\ast}, \boldsymbol\theta_t, \mat{u}_t^{\ast}), &\mat{p}_T^{\ast} = \mat{0},
    \label{eq:pmp backward}
\end{align}
plus a maximum condition of the Hamiltonian.
\begin{equation}
    H(t, \mat{x}_t^{\ast}, \mat{p}_t^{\ast}, \boldsymbol\theta_t, \mat{u}_t^{\ast}) \geq H(t, \mat{x}_t^{\ast}, \mat{p}_t^{\ast}, \boldsymbol\theta_t, \mat{u}_t), \; \forall \mat{u}\in\mathbb{R}^{d'} \; \text{ and } \; \forall t \in \mathcal{T}.
    \label{eq:pmp maximization}
\end{equation}
To obtain a numerical solution, one can consider iterating through the forward dynamic Eq.~\eqref{eq:hamiltonian} to obtain all states $\{ \mat{x}_t \}_{t=0}^{T-1}$,
the backward dynamic Eq.~\eqref{eq:pmp backward} to compute the adjoint states $\{ \mat{p}_t \}_{t=0}^{T-1}$,
and updating the Hamiltonian Eq.~\eqref{eq:pmp maximization} with current states and adjoint states via gradient ascent \citep{chen2021towards}.
This iterative process continuous until convergence.


\subsection{A Fast Implementation of the Closed-loop Control}
Now we discuss the computational overhead caused by the closed-loop control,
and propose an accelerated numerical solver based on the unique condition of optimality in the Pontryagin's Maximum Principle. 

\paragraph{Computational Overhead in Inference.} When the closed-loop control module is deployed for inference,
the original feed-forward propagation is now replaced by iterating through the Hamiltonian dynamics.
For each input data, solving the optimal control problems requires us to propagate through both forward Eq.~\eqref{eq:pmp forward} and backward adjoint Eq.~\eqref{eq:pmp backward} dynamics and to maximize the Hamiltonians Eq.~\eqref{eq:pmp maximization} at all layers.
When maximizing the Hamiltonian $n$ times, running the Hamiltonian dynamics approximately increase the time complexity by a factor of $n$ with respect to the standard inference.
The computational overhead  prevents deploying the closed-loop control module in real-world applications.

\paragraph{A Faster PMP Solver.} To address this issue, we consider the method of successive approximation \citep{chernousko1982method} from the optimal condition of the PMP.
For a given input data sample, Eq.~\eqref{eq:pmp forward} and \eqref{eq:pmp backward} generate the state variables and adjoint states respectively for the current controls $\{ \mat{u}_t \}_{t=0}^{T-1}$.
The optimal condition of the objective function in Eq.~\eqref{eq:closed-loop control cost function} is achieved via maximizing all Hamiltonians in Eq.~\eqref{eq:pmp maximization}.
Instead of iterating through all three Hamiltonian dynamics for a single update on the control solutions,
we can consider optimizing the $t^{th}$ Hamiltonian locally  for all $t \in [0,\cdots,T-1]$ with the current state $\mat{x}_t$ and adjoint state $\mat{p}_{t+1}$.
This allows the control solution $\mat{u}_t$ to be updated multiple times within one complete iteration.
Once a locally optimal control $\mat{u}_t^{\ast}$ is achieved by maximizing $H(t, \mat{x}_t, \mat{p}_{t+1}, \boldsymbol\theta_t, \mat{u}_t)$ w.r.t. $\mat{u}_t$,
the adjoint state $\mat{p}_{t + 1}$ is backpropagated to $\mat{p}_t$ via the adjoint dynamic in Eq.~\eqref{eq:pmp backward} followed by maximizing $H(t - 1, \mat{x}_{t-1}, \mat{p}_t, \boldsymbol\theta_{t-1}, \mat{u}_{t-1})$. 
Under this setting, running the Hamiltonian dynamics \eqref{eq:pmp forward}, \eqref{eq:pmp backward} and \eqref{eq:pmp maximization}, $n$ times can be decomposed into $maxItr$ full iterations and $InnerItr$ local updates.
Here $maxItr$ can be significantly smaller than $n$ since the locally optimal control solutions via $InnerItr$ updates can speed up the overall convergence. 
Instead of iterating the full Hamiltonian dynamics $n$ times, the proposed fast implementation iterates $maxItr$ full Hamiltonian dynamics, and $InnerItr$ local updates.

\begin{algorithm}[t]
  \caption{The Method of Successive Approximation.}
  \label{alg:closed_loop control implementation}
  \KwIn{Input $\mat{x}_0$ (possibly perturbed), a trained neural network $F(\cdot)$, embedding functions $\{ \mathcal{E}_t(\cdot) \}_{t=0}^{T-1}$,
  control regularization $c$,
  ${\rm maxItr}$, ${\rm InnerItr}$.} 
  \KwOut{Output state $\mat{x}_T$.}
  Initialize controls $\{ \mat{u}_t \}_{t=0}^{T-1}$ with the greedy solution \;
  \For{$i = 0$ to ${\rm maxItr}$}
  {
  $\mat{x}_0^i = \mat{x}_0 + \mat{u}_0^i$ 
  \tcp*{The controlled initial condition}
  \For{$t=0$ to $T-1$}
  {
  $\mat{x}_{t+1}^i = F_t(\mat{x}_{t}^i + \mat{u}_{t}^i)$
  \tcp*{Controlled forward propagation Eq.~\eqref{eq:pmp forward}}
  }
  $\mat{p}_T^i = \mat{0}$
  \tcp*{The terminal condition of adjoint state is set to $0$}
  \For{$t=T-1$ to $0$}
  {
  \For{$\tau = 0$ to ${\rm InnerItr}$}
  {
  $H(t, \mat{x}_t^i, \mat{p}_{t+1}^i, \boldsymbol\theta_t, \mat{u}_t^{i, \tau}) = \mat{p}_{t+1}^i \cdot F_t(\mat{x}_t^i, \boldsymbol\theta_t, \mat{u}_t^{i, \tau}) - {\cal L}(\mat{x}_t^i, \mat{u}_t^{i, \tau}, \mathcal{E}_t(\mat{x}_t^i))$
  \tcp*{Compute Hamiltonian}
  $\mat{u}_t^{i, \tau + 1} = \mat{u}_t^{i, \tau} + \alpha \cdot \nabla_{\mat{u}} H(t, \mat{x}_t^i, \mat{p}_{t+1}^i, \boldsymbol\theta_t, \mat{u}_t^{i, \tau})$
  \tcp*{Maximize Hamiltonian w.r.t. control $\mat{u}_t$}
  }
  $\mat{p}_t^i = \mat{p}_{t+1}^i \cdot \nabla_{\mat{x}} F(\mat{x}_t^i, \boldsymbol\theta_t, \mat{u}_t^i) - \nabla_{\mat{x}} {\cal L}(\mat{x}_t^i, \mat{u}_t^i, \mathcal{E}_t(\mat{x}_t^i))$
  \tcp*{Backward propagation Eq.~\eqref{eq:pmp backward}}
  }
  }
\end{algorithm}

The detailed implementation is presented in Algorithm~\ref{alg:closed_loop control implementation}.
Here we summarize this efficient implementation.
\begin{enumerate}[leftmargin=*]
    \item To begin with, We initialize all controls with the greedy solution, $\mat{u}_t = \mathcal{E}_t(\mat{x}_t) - \mat{x}_t$, by setting the control regularization $c = 0$.
    This improves the convergence of the Hamiltonian dynamics.
    \item We forward propagate the input data via Eq.~\eqref{eq:pmp forward} to obtain all hidden states.
    \item Since there is no terminal loss, the initial condition of adjoint state $\mat{p}_T = \mat{0}$.
    We backpropagate the adjoint states and maximize the Hamiltonian at each layer as follows:
    \begin{enumerate}[leftmargin=*]
        \item We compute the adjoint state $\mat{p}_t$ from the adjoint dynamics Eq.~\eqref{eq:pmp backward},
        \item Instead of updating control $\mat{u}_t$ once via maximizing the Hamiltonian Eq.~\eqref{eq:pmp maximization},
        we perform {\it multiple updates} ($InnerItr$ iterations) on control $\mat{u}_t$ to achieve the optimal solution $\mat{u}_t^{\ast}$ that satisfies the maximization condition (Notice that any optimization algorithm can be applied).
    \end{enumerate}
    \item The backpropagation terminates when it reaches layer $t=0$.
    This process repeats for an maximum number of iterations ($maxItr$ iterations).
\end{enumerate}


\section{Theoretical Error Analysis}
\label{sec:Theoretical Analysis}
In this section, we formally establish an error analysis for the closed-loop control framework. Let $\mat{x}_t$ be a ``clean" state originated from an unperturbed data sample $\mat{x}_0$, and $\mat{x}_{\epsilon, t}$ be the perturbed states originating from a possible attacked or corrupted data sample $\mat{x}_{\epsilon, 0} = \mat{x}_0 + \mat{z}$. In our proposed self-healing neural network, the controlled state becomes $\overline{\mat{x}}_{\epsilon, t} = \mat{x}_{\epsilon, t} + \mat{u}_t$. We ask this question: how large is $ \lVert \overline{\mat{x}}_{\epsilon, t} - \mat{x}_t \rVert $, i.e., the distance between $\mat{x}_t$ and $\overline{\mat{x}}_{\epsilon, t}$?

We consider a general deep neural network $F = F_T \circ F_{T-1} \circ \dots \circ F_{0}$, where each nonlinear transformation $F_t(\cdot)$ is of class $\mathcal{C}^2$,
and each embedding manifold can be described by a $\mathcal{C}^2$ submersion $f(\cdot) : \mathbb{R}^d \rightarrow \mathbb{R}^{d - r}$, such that $\mathcal{M} = f^{-1} (\mat{0})$.
Given an unperturbed state trajectory $\{ \mat{x}_t \in \mathcal{M}_t \}_{t=0}^{T-1}$,
we denote $\mathcal{T}_{\mat{x}_t}\mathcal{M}_t $ as the tangent space of $\mathcal{M}_t$ at $\mat{x}_t$.

This theoretical result is an extension of the linear closed-control setting in our preliminary work \citep{chen2021towards} where an error estimation in linear setting is derived.
We provide the error estimation between $\overline{\mat{x}}_{\epsilon, t}$ and $\mat{x}_{t}$ in the linear and nonlinear cases in Section~\ref{sec:Error Analysis of Linear Closed-loop Control Method} and Section~\ref{sec:Error Analysis of Closed-loop Control Method} respectively.

\subsection{Error Estimation For The Linearized Case}
\label{sec:Error Analysis of Linear Closed-loop Control Method}
Now we analyze the error of the self-healing neural network for a simplified case with linear activation functions.
We denote $\boldsymbol\theta_t$ as the Jacobian matrix of the nonlinear transformation $F_t(\cdot)$ centered at $\mat{x}_t$, such that $\boldsymbol\theta_t = F_t'(\mat{x}_t)$.
In the linear case,
the solution of the running loss in Eq.~\eqref{eq:running loss} is a projection onto the linear subspace, which admits a closed-form solution.
For a perturbed input, $\mat{q}_0 = \mat{x}_0 + \mat{z}$ with some perturbation $\mat{z}$,
we denote $\{ \mat{q}_{\epsilon, t} \}_{t=0}^{T-1}$ as sequence of states of  the linearized system,
and $\{ \overline{\mat{q}}_{\epsilon, t} \}_{t=0}^{T-1}$ as the states adjusted by the linear control. The perturbation $\mat{z} \in \mathbb{R}^d$ admits a direct sum of two orthogonal components, $\mat{z} = \mat{z}^{\parallel} \oplus \mat{z}^{\perp}$. Here $\mat{z}^{\parallel} \in \mathcal{T}_{\mat{x}_0}\mathcal{M}_0$ is a perturbation within the tangent space,
and $\mat{z}^{\perp}$ lies in the orthogonal complement of $\mathcal{T}_{\mat{x}_0}\mathcal{M}_0$. 

The following theorem \citep{chen2021towards} provides an upper bound of $\lVert \overline{\mat{q}}_{\epsilon, t} - \mat{x}_t \rVert_2^2$.

\begin{restatable}{theorem}{maintheoremLinear}
\label{theorem: error estimation of linear system}
For $t \geq 1$, we have an error estimation for the linearized system
\begin{equation*}
    \lVert \overline{\mat{q}}_{\epsilon, t} - \mat{x}_{t} \rVert_2^2
    \leq \lVert \boldsymbol\theta_{t-1} \cdots \boldsymbol\theta_{0} \rVert_2^2 \cdot \bigg( \alpha^{2t} \lVert \mat{z}^{\perp} \rVert_2^2 + \lVert \mat{z}^{\parallel} \rVert_2^2 + \gamma_t \lVert \mat{z} \rVert_2^2 \big( \gamma_t \alpha^2 (1 - \alpha^{t-1})^2 + 2 (\alpha - \alpha^t) \big) \bigg).
\end{equation*}
where $\gamma_t \vcentcolon = \max \limits_{s \leq t} \big(1 +  \kappa (\boldsymbol\theta_{s-1} \cdots \boldsymbol\theta_0)^2 \big) \lVert \mat{I} - (\boldsymbol\theta_{s-1} \cdots \boldsymbol\theta_0)^T (\boldsymbol\theta_{s-1} \cdots \boldsymbol\theta_0) \rVert_2$,
$\kappa(\boldsymbol\theta)$ is condition number of $\boldsymbol\theta$,
$\alpha = \frac{c} {1 + c}$, and $c$ represents the control regularization. 
In particular, the equality
\begin{equation*}
    \lVert \overline{\mat{q}}_{\epsilon, t} - \mat{x}_{t} \rVert_2^2 =  \alpha^{2t} \lVert \mat{z}^{\perp} \rVert_2^2 + \lVert \mat{z}^{\parallel} \rVert_2^2
\end{equation*}
holds when all $\boldsymbol\theta_t$ are orthogonal.
\end{restatable}
The detailed derivation is presented in Appendix \ref{proof: error estimation of the linear case}.
The error upper bound is tight since it becomes the actual error if all the linear transformations are orthogonal matrices.
Note that the above bound from the greedy control solution is a strict upper bound of the optimal control solution.
The greedy solution does not consider the dynamic, and it optimizes each running loss individually.

\subsection{Error Analysis of Nonlinear Networks with Closed-loop Control}
\label{sec:Error Analysis of Closed-loop Control Method}
Here we provide an error analysis for the self-healing neural network with general nonlinear activation functions.
For a $3$-dimensional tensor, e.g. the Hessian $f''(\mat{x})$, we define the $2$-norm of $f''(\mat{x})$ as
\begin{equation*}
    \lVert f''(\mat{x}) \rVert_{\ast} \vcentcolon = \sup_{\mat{z} \neq \mat{0}} \frac{\lVert f''(\mat{x})^{i,j,k} \mat{z}_j \mat{z}_k \rVert_2} {\lVert \mat{z} \rVert_2^2}.
\end{equation*}
For the nonlinear transformation $F_t(\cdot) \in \mathcal{C}^2$ at layer $t$,
we assume its Hessian $f''_t(\cdot)$ is uniformly bounded, i.e.,
$\sup_{\mat{x} \in \mathbb{R}^d} \lVert F''_t(\mat{x}) \rVert_{\ast} \leq \beta_t$.
Let $f_t \in \mathcal{C}^2 : \mathbb{R}^d \rightarrow \mathbb{R}^{d-r}$ be the submersion of the embedding manifold $\mathcal{M}_t$, we assume its Hessian is uniformly bounded, i.e.,
$\sup_{\mat{x} \in \mathbb{R}^d} \lVert f''_t(\mat{x}) \rVert_{\ast} \leq \sigma_t$.
We use $\mat{x}_t$, $\mat{x}_{\epsilon, t}$ and $\overline{\mat{x}}_{\epsilon, t}$ to denote the clean states, perturbed states without control and the states adjusted with closed-loop control, respectively. 
The initial perturbation $\mat{z} = \epsilon \cdot \mat{v}$, where $\lVert \mat{v} \rVert_2 = 1$ and
$\mat{v} = \mat{v}^{\parallel} \oplus \mat{v}^{\perp}$.
Let
\begin{itemize}[leftmargin=*]
    \item $k_t = 4 \sigma_t \lVert (f_t'(\mat{x}_t) f_t'(\mat{x}_t)^T)^{-1} \rVert_2 \cdot (\lVert f_t'(\mat{x}_t)\rVert_2 + 2 \sigma_t)$,
    \item $\delta_{\mat{x}_t} = \lVert \boldsymbol\theta_{t-1} \cdots \boldsymbol\theta_{0} \rVert_2^2 \cdot
    \bigg( \alpha^{2t} \lVert \mat{v}^{\perp} \rVert_2^2 + \lVert \mat{v}^{\parallel} \rVert_2^2 + \gamma_t \lVert \mat{v} \rVert_2^2 \big( \gamma_t \alpha^2 (1 - \alpha^{t-1})^2 + 2 (\alpha - \alpha^t) \big) \bigg)$.
\end{itemize}

The following theorem provides an error estimation between $\overline{\mat{x}}_{\epsilon, t}$ and $\mat{x}_{t}$.

\begin{restatable}{theorem}{maintheoremNonlinear}
\label{theorem: error estimation of nonlinear system}
If the initial perturbation satisfies
\begin{equation*}
    \epsilon^2 \leq \frac{1} {\bigg( \sum_{i=0}^{T-1} \delta_{\mat{x}_i} (k_{\mat{x}_i} \lVert \boldsymbol\theta_i \rVert_2 + 2 \beta_i) \prod_{j = i+1}^{T-1} (\lVert \boldsymbol\theta_j \rVert_2 + k_{\mat{x}_j} \lVert \boldsymbol\theta_j \rVert_2 + 2 \beta_j) \bigg)}.
\end{equation*}
for $1 \leq t \leq T$, 
we have the following error bound for the closed-loop controlled system
\begin{align*}
    \sq{\lVert \overline{\mat{x}}_{\epsilon, t + 1} - \mat{x}_{t + 1} \rVert_2}
    & \leq \sq{\lVert \boldsymbol\theta_t \cdots \boldsymbol\theta_0 \rVert_2 \bigg( \alpha^{t+1} \lVert \mat{z}^{\perp} \rVert_2 + \lVert \mat{z}^{\parallel} \rVert_2 
    + \lVert \mat{z} \rVert_2 \big( \gamma_{t+1} \alpha (1 - \alpha^{t}) + \sqrt{2 \gamma_{t+1} (\alpha - \alpha^{t+1})} \big) \bigg)} \\
    & \;\;\;\; + \bigg( \sum_{i=0}^t \delta_{\mat{x}_i} (k_{\mat{x}_i} \lVert \boldsymbol\theta_i \rVert_2 + 2 \beta_i) \prod_{j = i+1}^t (\lVert \boldsymbol\theta_j \rVert_2 + k_{\mat{x}_j} \lVert \boldsymbol\theta_j \rVert_2 + 2 \beta_j) \bigg) \epsilon^2.
\end{align*}
\end{restatable}
The detailed proof is provided in Appendix \ref{proof: error estimation of nonlinear system}.
From Theorem \ref{theorem: error estimation of nonlinear system}, we have the following intuitions:
\begin{itemize}[leftmargin=*]
    \item The error estimation has two main components: an linearization error in order of $\mathcal{O}(\epsilon^2)$,
    and the error of $\mathcal{O}(\epsilon)$ of the linearized system.
    Specifically, the linearization error becomes smaller when the activation functions and embedding manifolds behave more linearily ($k_t$ and $\beta_t$ become smaller).
    \item The closed-loop control minimizes the perturbation components $\mat{z}^{\perp}$ within the orthogonal complements of the tangent spaces.
    This is consistent with the manifold hypothesis,
    the robustness improvement is more significant if the underlying data are embedded in a lower dimensional manifold ($\lVert \mat{z}^{\parallel} \rVert_2 \rightarrow 0$).
    \item The above error estimation improves as the control regularization $c$ goes to $0$ (so $\alpha\rightarrow 0$). It is not the sharpest possible as it relies on a greedily optimal control at each layer. The globally optimal control defined by the Ricatti equation may achieve a lower loss when $c\neq 0$.
\end{itemize}

\section{Numerical Experiments}
\label{sec:numerical experiments}
In this section, we test the performance of the proposed self-healing framework.
Specifically, we show that using only one set of embedding functions can improve robustness of many pre-trained models consistently.
Section \ref{sec:cifar10} shows that the proposed method can significantly improve the robustness of both standard and robustly trained models on CIFAR-10 against various perturbations.
Furthermore, in the same experimental setting, sections \ref{sec:cifar100} and \ref{sec:tiny-imagenet} evaluate on CIFAR-100 and Tiny-ImageNet datasets,
which empirically verify effectiveness and generalizability of the self-healing machinery.

\subsection{Experiments On CIFAR-10 Dataset}
\label{sec:cifar10}
We evaluate all controlled models under an ''oblivious attack'' setting \footnote{This consideration is general, e.g. \cite{liao2018defense} has adopted this setting in the previous NIPS competition on defense against adversarial attacks.}. In this setting, the pre-trained models are fully accessible to an attacker, but the control information is not released. Meanwhile, the controllers do not have knowledge about the incoming attack algorithms. We will show that using one set of embedding functions, our self-healing method can improve the robustness of many pre-trained models against a broad class of perturbations. Our experimental setup is summarized below.
\begin{itemize}[leftmargin=*]
    \item {\bf Baseline models.}
    We showcase that {\it one set of controllers } can consistently increase the robustness of many pre-trained ResNets when those models are trained via standard training (momentum SGD) and adversarial training (TRADES \citep{zhang2019theoretically}).
    Specifically,
    we use Pre-activated ResNet-$18$ (\textbf{RN-18}), -$34$ (\textbf{RN-34}), -$50$ (\textbf{RN-50}),
    wide ResNet-$28$-$8$ (\textbf{WRN-28-8}), -$34$-$8$ (\textbf{WRN-34-8}) as the testing benchmarks.
    \item {\bf Robustness evaluations.}
    We evaluate the performance of all models with clean testing data (\textbf{None}), and auto-attack (\textbf{AA}) \citep{croce2020reliablef} that is measured by $\ell_{\infty}$, $\ell_2$ and $\ell_1$ norms.
    Auto-attack that is an ensemble of two gradient-based auto-PGD attacks \citep{croce2020reliablef},
    fast adaptive boundary attack \citep{croce2020minimally}
    and a black-box square attack \citep{andriushchenko2020square}.
    \item {\bf Embedding functions.}
    We choose the fully convolutional networks (FCN) \citep{long2015fully} as an input embedding function,
    and a $2$-layer auto-encoder as an embedding function for the hidden states.
    Specifically, we use one set of embedding functions for all $5$ pre-trained models.
    The training objective function of the $t^{th}$ embedding function follows Eq.~\eqref{eq:embedding function implementation},
    where both model and data information are used.
    \item {\bf PMP hyper-parameters setting.}
    We choose $3$ outer iterations and $10$ inner iterations with $0.001$ as a control regularization parameters in the PMP solver.
    As in Algorithm \ref{alg:closed_loop control implementation}, maxIte=$3$, InnerItr=$10$, and $c=0.001$.
\end{itemize}


\begin{table}[t]
\caption{CIFAR-10 accuracy measure: baseline model / controlled model}
\begin{center}
\begin{tabular}{ p{1.9cm}|p{2.3cm}p{2.3cm}p{2.3cm}p{2.3cm} }
\hline
\multicolumn{5}{c}{$\ell_{\infty}:\epsilon=8 / 255$, $\ell_2:\epsilon=0.5$, $\ell_1:\epsilon=12$}\\
\hline
\multicolumn{5}{c}{Standard models}\\
\hline
& None & AA ($\ell_{\infty}$) & AA ($\ell_2$) & AA ($\ell_1$) \\
\hline
RN-18 & \textbf{94.71} / 92.81 & 0. / \textbf{63.89} & 0. / \textbf{82.1} & 0. / \textbf{75.75} \\
\hline
RN-34 & \textbf{94.91} / 92.84 & 0. / \textbf{64.92} & 0. / \textbf{83.64} & 0. / \textbf{78.05} \\
\hline
RN-50 & \textbf{95.08} / 92.81 & 0. / \textbf{64.31} & 0. / \textbf{83.33} & 0. / \textbf{77.15} \\
\hline
WRN-28-8 & \textbf{95.41} / 92.63 & 0. / \textbf{75.39} & 0. / \textbf{86.71} & 0. / \textbf{84.5} \\
\hline
WRN-34-8 & \textbf{94.05} / 92.77 & 0. / \textbf{64.14} & 0. / \textbf{82.32} & 0. / \textbf{73.54} \\
\hline
\hline
\multicolumn{5}{c}{Robust models (trained with $\ell_{\infty}$ perturbations)}\\
\hline
& None & AA ($\ell_{\infty}$) & AA ($\ell_2$) & AA ($\ell_1$) \\
\hline
RN-$18$ & 82.39 / \textbf{87.51} & 48.72 / \textbf{66.61} & 58.8 / \textbf{79.88} & 9.86 / \textbf{42.85} \\
\hline
RN-$34$ & 84.45 / \textbf{87.93} & 49.31 / \textbf{65.49} & 57.27 / \textbf{78.81} & 7.21 / \textbf{40.74} \\
\hline
RN-$50$ & 83.99 / \textbf{87.57} & 48.68 / \textbf{65.17} & 57.25 / \textbf{78.26} & 6.83 / \textbf{39.44} \\
\hline
WRN-$28$-$8$ & 85.09 / \textbf{87.66} & 48.13 / \textbf{64.44} & 54.38 / \textbf{77.08} & 5.38 / \textbf{41.78} \\
\hline
WRN-$34$-$8$ & 84.95 / \textbf{87.14} & 48.47 / \textbf{64.55} & 54.36 / \textbf{77.15} & 4.67 / \textbf{42.65} \\
\hline
\end{tabular}
\end{center}
\label{cifar10 oblivious control}
\end{table}

As shown in Table \ref{cifar10 oblivious control},
for standard trained baseline models,
despite of the high accuracy on clean data,
their robustness against strong auto-attack degrade to $0 \%$ accuracy under all measurements.
The self-healing process is attack agnostic, and it improves the the robustness against all perturbations with negligible degradation on clean data.
Specifically,
the controlled models have more than $80 \%$ and near $80 \%$ accuracies against perturbations measured by $\ell_2$ and $\ell_1$ norms respectively.


On adversarially trained baseline models.
Since all robust baseline models are pre-trained with $\ell_{\infty}$ measured adversarial examples, 
they show strong robustness against $\ell_{\infty}$ auto-attack.
Surprisingly, 
models that trained using $\ell_{\infty}$ as adversarial training objective preserve strong robustness against $\ell_2$ perturbations.
However, a $\ell_1$ measured perturbation can significantly degrade their robustness.
On average, our proposed control method have achieved $20 \%$ accuracy improvements against $\ell_{\infty}$ and $\ell_{2}$ perturbations,
and near $40 \%$ improvement against $\ell_1$ perturbation.
Surprisingly, by applying the proposed control module, all adversarially trained models have achieved higher accuracy on clean testing data.
In Appendix \ref{exp:additional numerical experiments}, we show more experimental results on CIFAR-10 dataset, including the robustness evaluation of controlled models against white-box attack.

\subsection{Experiments On CIFAR-100 Dataset}
\label{sec:cifar100}
In this section, we investigate the effectiveness of self-healing on the more challenging CIFAR-100 dataset.
We summarize our experiment settings below.

\begin{itemize}[leftmargin=*]
    \item {\bf Baseline models.}
    We consider different variants of Wide-ResNet.
    Specifically, we use Wide-ResNet-28-10 (\textbf{WRN-28-10}), -34-10 (\textbf{WRN-34-10}), -76-10 (\textbf{WRN-76-10}).
    We show that one set of controllers can consistently increase the robustness of all $3$ pre-trained models when those models are trained via adversarial training (TRADES \citep{zhang2019theoretically}).
    \item {\bf Other settings.}
    The embedding functions and PMP settings follow the same.
\end{itemize}


\begin{table}[t]
\caption{CIFAR-100 accuracy measure: baseline model / controlled}
\begin{center}
\begin{tabular}{ p{2.1cm}|p{2.3cm}p{2.3cm}p{2.3cm}p{2.2cm} }
\hline
\multicolumn{5}{c}{$\ell_{\infty}:\epsilon=8 / 255$, $\ell_2:\epsilon=0.5$, $\ell_1:\epsilon=12$}\\
\hline
WRN-28-10 & \textbf{56.96} / 56.84 & 24.97 / \textbf{30.81} & 29.54 / \textbf{39.18} & 3.24 / \textbf{16.43} \\
\hline
WRN-34-10 & \textbf{57.32} / 56.91 & 25.35 / \textbf{31.04} & 29.68 / \textbf{39.64} & 2.99 / \textbf{17.66} \\
\hline
WRN-76-10 & \textbf{57.58} / 57.11 & 24.84 / \textbf{29.96} & 27.81 / \textbf{38.05} & 2.41 / \textbf{19.13} \\
\hline
\end{tabular}
\end{center}
\label{cifar100 oblivious control robust}
\end{table}

In Table \ref{cifar100 oblivious control robust},
the proposed self-healing framework consistently improves the robustness of adversarially trained models on CIFAR-100 dataset.
On average, the self-healing models have achieved $10 \% \sim 20 \%$ accuracy improvement with almost no effects on the clean data performance.
Although the improvements are not as significant as in the CIFAR-10 experiment,
this is due to the hardness of constructing embedding manifolds for this more challenging dataset.
Specifically, it is more difficult to distinguish the controlled data point among $100$ different classes than $10$ classes on an single embedding manifold.

\subsection{Experiments On Tiny-ImageNet}
\label{sec:tiny-imagenet}
Finally, we examine the proposed self-healing framework on Tiny-ImageNet dataset.
Tiny-ImageNet contains $100,000$ and $10,000$ of $64 \times 64$ sized training and validation images with $200$ different classes.
Although over-fitting is more significant on this dataset,
we show that the proposed self-healing framework can consistently improve the robustness of pre-trained models.
The experimental settings are summarized below.

\begin{itemize}[leftmargin=*]
    \item {\bf Baseline models.}
    We consider \textbf{EfficientNet-b0}, \textbf{EfficientNet-b1} and \textbf{EfficientNet-b2} trained via momentum SGD as testing benchmarks.
    \item {\bf Embedding functions.}
    We choose SegNet \citep{badrinarayanan2017segnet} as an input embedding function,
    and a $2$-layer auto-encoder as an embedding function for the hidden states.
    The training objective function of the $t^{th}$ embedding function follows Eq.~\eqref{eq:embedding function implementation},
    where both model and data information are used.
    \item {\bf PMP hyper-parameters setting.}
    The PMP setting follows the same.
\end{itemize}

\begin{table}[t]
\caption{Tiny-ImageNet accuracy measure: baseline model / controlled}
\begin{center}
\begin{tabular}{ p{2.6cm}|p{2.3cm}p{2.2cm}p{2.2cm}p{2.2cm} }
\hline
\multicolumn{5}{c}{$\ell_{\infty}:\epsilon=4 / 255$, $\ell_2:\epsilon=0.8$, $\ell_1:\epsilon=10$}\\
\hline
& None & AA ($\ell_{\infty}$) & AA ($\ell_2$) & AA ($\ell_1$) \\
\hline
EfficientNet-b0 & 57.68 / \textbf{59.92} & 0.21 / \textbf{46.08} & 1.73 / \textbf{49.86}  & 5.86 / \textbf{50.4} \\
\hline
EfficientNet-b1 & 57.99 / \textbf{59.7}2 & 0.13 / \textbf{44.35} & 1.24 / \textbf{48.26} & 4.43 / \textbf{48.86} \\
\hline
EfficientNet-b2 & 58.06 / \textbf{59.3} & 0.25 / \textbf{44.33} & 1.40 / \textbf{47.86} & 4.58 / \textbf{48.39} \\
\hline
\end{tabular}
\end{center}
\label{tiny-imagenet oblivious control robust}
\end{table}

In this task,
we aim to validate the practical applicability of the proposed method on generally large dataset and deep network architectures.
In Table \ref{tiny-imagenet oblivious control robust},
on the challenging Tiny-ImageNet dataset,
despite of the high accuracy on clean data,
as expected,
all pre-trained models result in extremely poor performance against autoattacks.
The proposed framework can improve all three pre-trained EfficientNets consistently against autoattacks.
Specifically,
the controlled models has shown $ 45 \% \sim 50 \%$ robustness improvements against all perturbations.

\subsection{Summary On Numerical Experiments}
Fig.~\ref{fig:radar plot on numerical results} shows the radar plots of accuracy against many perturbations on some chosen baseline models.
Overall, the self healing via close-loop control consistently improve the baseline model performance.
Notice that adversarial training can effectively improve the robustness of baseline models against a certain type of perturbation (e.g. Auto-attack measured in $\ell_{\infty}$).
However, those seemingly robust models are extremely vulnerable against other types of perturbations (e.g. Auto-attack measured in $\ell_1$).
The proposed method is attack-agnostic and can consistently improve robustness of many baseline models against various perturbations.

\begin{figure}[t]
	\centering
	\begin{minipage}{0.95\linewidth}
	\centering
	    \includegraphics[width = \linewidth]{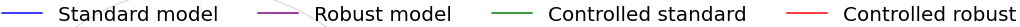}
	\end{minipage}
	\begin{minipage}{.32\linewidth}
	\centering
	    \includegraphics[width = \linewidth]{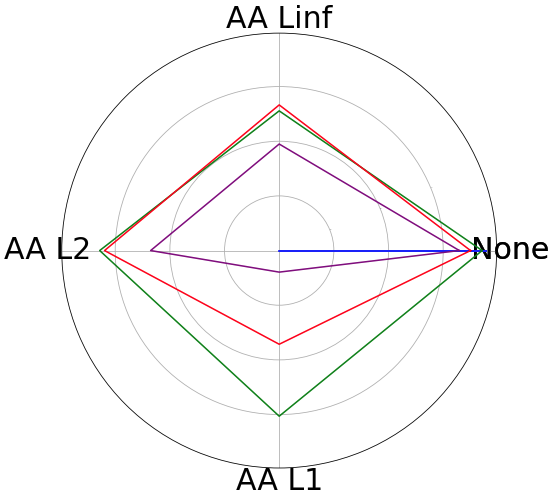}
	    (a): CIFAR-10 
	\end{minipage}
		\begin{minipage}{.32\linewidth}
		\centering
	    \includegraphics[width=\linewidth]{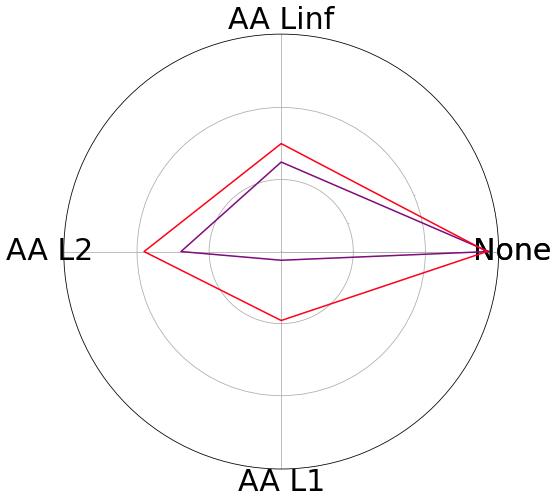}
	    (b): CIFAR-100 
	\end{minipage}
		\begin{minipage}{.32\linewidth}
		\centering
	    \includegraphics[width=\linewidth]{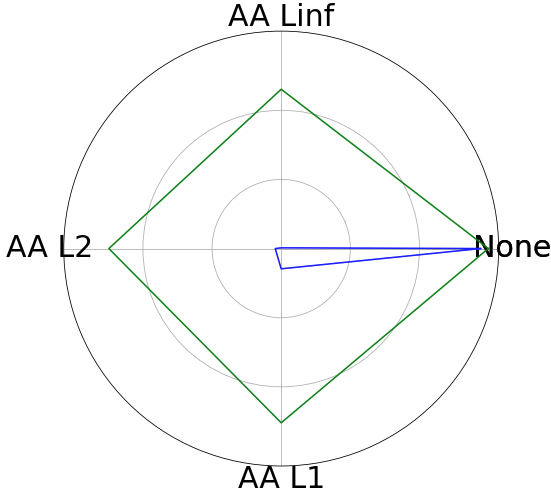}
	    (c): Tiny-Imagenet 
	\end{minipage}
    \caption{(a), (b) and (c)  are radar plots that summarizes RN-18 in Table \ref{cifar10 oblivious control}, WRN-76-10 in Table \ref{cifar100 oblivious control robust}, and EfficientNet-b0 in \ref{tiny-imagenet oblivious control robust} respectively.
    }
    \label{fig:radar plot on numerical results}
\end{figure}

\section{Discussions}
\label{sec:discussion}

\subsection{Limitation of the proposed self-healing framework}
In this section, we discuss the limitations of the current work from both practical and theoretical perspectives.
Those discussions provide insights of the current framework and motivate future research in this direction.

\paragraph{The accuracy of embedding manifolds affects the control performance.}
As shown in Sec. \ref{sec:control objective and margin analysis}, the objective function of the proposed self-healing framework minimizes the distance between a given state trajectory and the embedding manifolds.
The role of embedding manifolds encodes the geometric information of how the clean data behaves in the pre-trained deep neural network.
Therefore, a important question is how accurate those embedding manifolds encode the state trajectories.
In a case that the embedding manifolds do not precisely resemble the data structures, the running loss in Eq.~\eqref{eq:running loss} that measures the distance between a perturbed data and the embedding manifold does not have the true information of this applied perturbation.
Then the applied controls might lead to wrong predictions.
Table~\ref{cifar10 bad embedding} compares two control results that use optimal embedding functions as in Table~\ref{cifar10 oblivious control} and SegNet as embedding function.
It shows the importance of constructing accurate embedding functions.

\begin{table}[t]
\caption{Control with sub-optimal embedding function / optimal embedding function}
\begin{center}
\begin{tabular}{ p{1.9cm}|p{2.3cm}p{2.3cm}p{2.3cm}p{2.3cm} }
\hline
\multicolumn{5}{c}{$\ell_{\infty}:\epsilon=8 / 255$, $\ell_2:\epsilon=0.5$, $\ell_1:\epsilon=12$}\\
\hline
& None & AA ($\ell_{\infty}$) & AA ($\ell_2$) & AA ($\ell_1$) \\
\hline
RN-18 & 82.97 / \textbf{92.81} & 40.29 / \textbf{63.89} & 53.61 / \textbf{82.1} & 50.69 / \textbf{75.75} \\
\hline
RN-34 & 83.14 / \textbf{92.84} & 45.55 / \textbf{64.92} & 56.57 / \textbf{83.64} & 54.34 / \textbf{78.05} \\
\hline
RN-50 & 82.39 / \textbf{92.81} & 44.18 / \textbf{64.31} & 56.06 / \textbf{83.33} & 53.89 / \textbf{77.15} \\
\hline
\end{tabular}
\end{center}
\label{cifar10 bad embedding}
\end{table}

In addition to the precision of embedding manifolds, recall the definition of embedding function in Eq.\eqref{eq:manifold projection}, it searches for a closest counterpart of a given data on the embedding manifold.
It implies that if data belong to the embedding manifold, its outcome from the embedding function should stay the same.
This property is guaranteed by well-studied linear projection.
Given a linear projection operator $\mat{P}$, we have $\mat{P} \circ \mat{P} \mat{x} = \mat{P} \mat{x}$.
However, 
in the nonlinear case,
the shortest path projection $ \mathcal{E} (\mat{x})$ defined in Eq.~\eqref{eq:manifold projection} does not necessarily hold the projection property.
The lack of projection property adds more challenge on measuring the running loss in Eq.~\eqref{eq:running loss}.

\paragraph{The role of control regularization is unclear.}
The second limitation is to understand the role of applying control regularization.
As in the conventional optimal control problems, regularizing the applied controls has practical meaning, such as limiting the amount of energy consumption.
In the current framework, we have observed that regularizing the applied controls can alleviate the issue of inaccurate embedding manifold,
in which case, the controls only slightly adjust the perturbed state trajectory.
This observation is not theoretically justified in this work, and we will continue to understand this in future work.

\subsection{A Broader Scope of Self Healing}

\begin{figure}[t]
\centering
\includegraphics[width=0.9\linewidth]{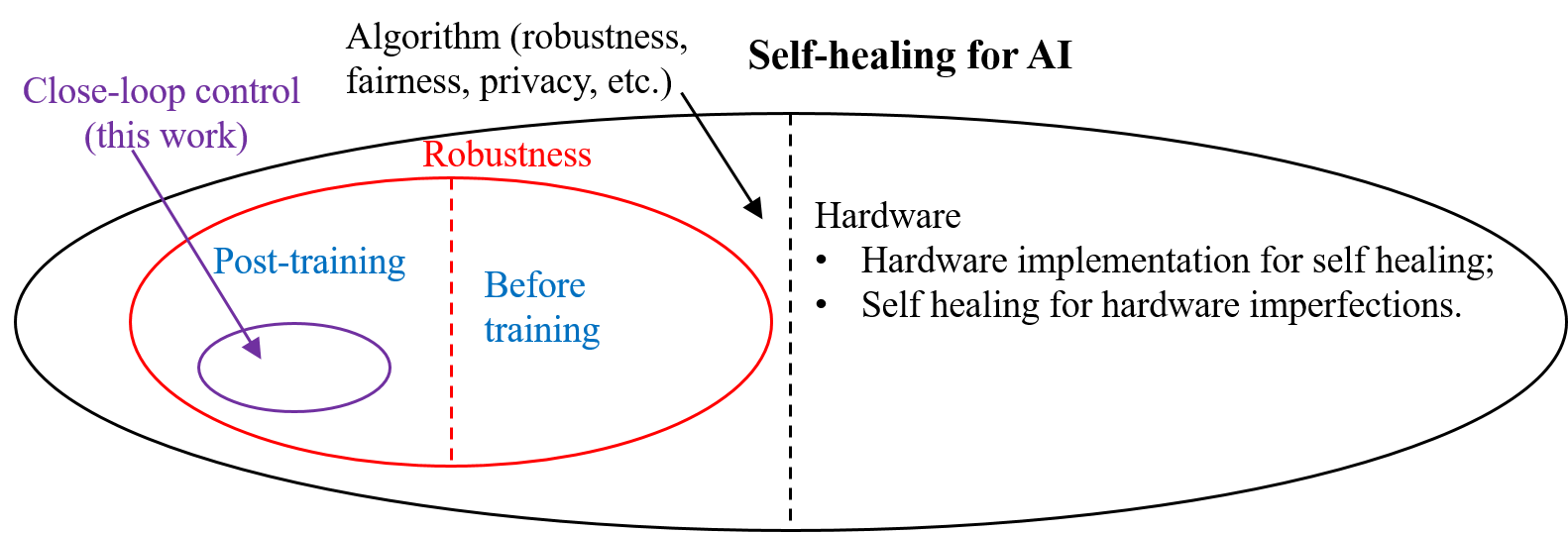}
\caption{A broader scope of self healing for neural networks.}
\label{fig: self-healing_broadview}
\end{figure}

This paper has proposed a self-healing framework implemented with close-loop control to improve the robustness of  given neural network. The control signals are generated and injected into all neurons. Based on this work, many more topics can be investigated in the future. Below we point out some possible directions.
\paragraph{Extension of the Proposed Framework.} An immediate extension of this work is to consider the closed-loop control applied to model parameters (instead of neurons). In order to reduce the control complexity, we may also deploy local performance monitor and local control, instead of monitoring and controlling all neurons. Another more fundamental question is: can we achieve self healing without using closed-loop control? In~\citep{wang2021immuno,wang2021rails}, self healing is achieved by mimicking the immune system of a human body and without using closed-loop control.
\paragraph{Beyond Post-Training Self-Healing.} This work focuses on realizing self healing after a neural network is already trained. It may be possible to achieve self healing in other development stages of a neural network, such as in training and in data acquisition/preparation.  For instance, \cite{kang2021stable} built a robust neural network by enforcing Lyapunov stability in the training. Their neural network offers better robustness via automatically attracting some unforeseen perturbed trajectories to a region around the desired equilibrium point associated with unperturbed data, thus it can be regarded to have some ``self-healing" capabilities.

\paragraph{Beyond Robustness.} The key idea of self healing is to automatically fix the possible errors/weakness of a neural network, with or without a performance monitor. This idea may be extended to address other fundamental issues in AI, such as AI fairness and privacy, as well as the safety issue in AI-based decision making. 

\paragraph{Self Healing at the Hardware Level/Computing Platforms.} The self-healing perspectives bring in many opportunities and challenges at the hardware level. On one hand, the proposed self healing can cause extra hardware cost in the inference. Therefore, it is important to investigate hardware-efficient self-healing mechanism, which can provide self-healing capability with minimal hardware overhead. On the other hand, many imperfections of AI hardware may also be addressed via self healing. Examples include process variations in AI ASIC chip design, and software/hardware errors in distributed AI platforms.

Our vision is visualized in Fig. ~\ref{fig: self-healing_broadview}. This work is a proof-of-concept demonstration of self healing for AI robustness, and many more research problems need to be investigated in the future. 

\section{Conclusion}
This paper has improved the robustness of neural network from a new self-healing perspective. By formulating the problem as a closed-loop control, we show that it is possible for a neural network to automatically detect and fix the possible errors caused by various perturbations and attacks. We have provided margin-based analysis to explain why the designed control loss function can improve robustness. We have also presented efficient numerical solvers to mitigate the computational overhead in inference. Our theoretical analysis has also provided a strict error bound of the neural network trajectory error under data perturbations. The numerical experiments have shown that this method can significantly increase the robustness of neural networks under various types of perturbations/attacks that were unforeseen in the training process. As pointed out in Section~\ref{sec:discussion}, 
this self-healing method may be extended to investigate other fundamental issues (such as fairness, privacy and hardware reliability) of neural networks in the future.

\section*{Acknowledgement}

Zhuotong Chen and Zheng Zhang are supported by NSF Grant \# 2107321 and DOE Grant \# DE-SC0021323.
Qianxiao Li is supported by the National Research Foundation of Singapore, under the NRF Fellowship (NRF-NRFF13-2021-0005).




\newpage
\appendix
\section{Manifold Projection On Classifier Margin}
\label{proof: Manifold Projection Enlarges Classifier Margin In Linear Case}
\linearclassifiermargin*

\begin{proof}
We define the ground-truth manifold as follows,
\begin{equation*}
    \mathcal{M}^{\ast} = \{\mat{x}: \mat{A} \mat{x} = \mat{0}, |\mat{c}^T \mat{x}| \geq d\},
\end{equation*}
that is, the ground-truth manifold $\mathcal{M}^{\ast}$ consists of two half-spaces corresponding to the two classes.
Let $\mathcal{M}$ be an linear subspace $\mathcal{M} = \{\mat{x}:\mat{A} \mat{x} = \mat{0}\}$, in which case, $\mathcal{M}^{\ast} \subset \mathcal{M}$.
We consider a linear classifier with random decision boundary.
Let $\mat{A}$ be a hyperplane that represents the decision boundary of this linear classifier, 
$\hat{\mat{n}}$ as a $d$-dimensional normal vector such that $\hat{\mat{n}}^T \mat{A} = \mat{0}$.
A random linear classifier can be represented by $\hat{\mat{n}} \sim \mathcal{N} (\mat{0}, \frac{1} {d} \mat{I})$.

\paragraph{Manifold and Euclidean margins attain the same $\mat{x}^{\ast}$.}
In this linear case,
the following shows that the manifold margin $d_{\mathcal{M}} (F (\cdot))$ in Eq.~\eqref{eq:manifold margin equation} is equivalent to $d_e(F \circ \mathcal{E}(\cdot))$ in Eq.~\eqref{eq:euclidean margin equation} where $\mathcal{E}(\cdot)$ is the orthogonal projection onto the subspace $\mathcal{M}$.
Since the embedding manifold $\mathcal{M}$ is a linear subspace,
the geodesics defined on the manifold is equivalent as the euclidean norm,
\begin{align*}
    \mathcal{R}_{\mathcal{M}}(\mat{a}, \mat{b})
    \vcentcolon = &
    \inf\limits_{
        \gamma \in \Gamma_{\mathcal{M}}(\mat{a}, \mat{b}),
    }
    \int_0^1 \sqrt{ \langle\, \gamma'(t),\gamma'(t) \rangle_{\gamma(t)} } dt, \\
    & = \lVert \mat{a} - \mat{b} \rVert_2,
\end{align*}
the manifold margin can be shown as follows,
\begin{align*}
    d_{\mathcal{M}}(F(\cdot)) & = \frac{1}{2} \inf\limits_{\mat{x}_1, \mat{x}_2 \in \mathcal{M}^{\ast}} \mathcal{R}_{\mathcal{M}} (\mat{x}_1, \mat{x}_2), \;\;\;\; {\rm s.t.} ~F(\mat{x}_1) \neq F(\mat{x}_2), \\
    & = \frac{1}{2} \inf\limits_{\mat{x}_1, \mat{x}_2 \in \mathcal{M}^{\ast}} \lVert \mat{x}_1 - \mat{x}_2 \rVert_2, \;\;\;\; {\rm s.t.} ~F(\mat{x}_1) \neq F(\mat{x}_2).
\end{align*}
Furthermore,
\begin{align*}
    d_e(F \circ \mathcal{E}(\cdot)) & = \inf\limits_{\mat{x} \in \mathcal{M}^{\ast}} \inf\limits_{\boldsymbol\delta \in \mathbb{R}^d} \lVert \boldsymbol\delta \rVert_2, \;\;\;\; {\rm s.t.}~F \circ \mathcal{E}(\mat{x}) \neq F \circ \mathcal{E}(\mat{x} + \boldsymbol\delta), \\
    & = \inf\limits_{\mat{x} \in \mathcal{M}^{\ast}} \inf\limits_{\boldsymbol\delta \in \mathbb{R}^d} \lVert \boldsymbol\delta \rVert_2, \;\;\;\; {\rm s.t.}~F(\mat{x}) \neq F (\mat{x} + \mathcal{E}(\boldsymbol\delta)), \\
    & = \inf\limits_{\mat{x} \in \mathcal{M}^{\ast}} \inf\limits_{\boldsymbol\delta \in \mathbb{R}^d } \inf\limits_{\boldsymbol\delta' = \mathcal{E} (\boldsymbol\delta)} \lVert \boldsymbol\delta' \rVert_2, \;\;\;\; {\rm s.t.}~F(\mat{x}) \neq F(\mat{x} + \boldsymbol\delta'), \\
    & = \inf\limits_{\mat{x} \in \mathcal{M}^{\ast}} \inf\limits_{\boldsymbol\delta' \in \mathcal{M}} \lVert \boldsymbol\delta' \rVert_2, \;\;\;\; {\rm s.t.}~F(\mat{x}) \neq F(\mat{x} + \boldsymbol\delta'), \\
    & = \frac{1} {2} \inf\limits_{\mat{x}_1, \mat{x}_2 \in \mathcal{M}^{\ast}} \lVert \mat{x}_1 - \mat{x}_2 \rVert_2, \;\;\;\; {\rm s.t.}~F(\mat{x}_1) \neq F(\mat{x}_2),
    \\
    & = d_{\mathcal{M}}(F(\cdot)).
\end{align*}
where the embedding function $\mathcal{E} (\cdot)$ is replaced by restricting $\boldsymbol\delta \in \mathcal{M}$.

We denote $\hat{\mat{n}}$ as a $d$-dimensional normal vector that is orthogonal to the hyperplane defined by $\mat{A}$, such that $\hat{\mat{n}}^T \mat{A} = \mat{0}$.
The Euclidean margin in Eq.~\eqref{eq:euclidean margin equation} can be shown as follows,
\begin{equation*}
    d_e(F(\cdot)) = \inf\limits_{\mat{x} \in \mathcal{M}^{\ast}} \lVert \mat{x}^T \hat{\mat{n}} \rVert_2.
\end{equation*}
Since $\mathcal{E}(\cdot)$ is a linear orthogonal projection,
recall that $d_{\mathcal{M}}(F(\cdot)) = d_e(F \circ \mathcal{E}(\cdot))$,
\begin{equation*}
    d_{\mathcal{M}}(F(\cdot)) = d_e(F \circ \mathcal{E}(\cdot))= \inf\limits_{\mat{x} \in \mathcal{M}^{\ast}} \frac{\lVert \mat{x}^T \mathcal{E} (\hat{\mat{n}}) \rVert_2} {\lVert \mathcal{E} (\hat{\mat{n}}) \rVert_2} = \inf\limits_{\mat{x} \in \mathcal{M}^{\ast}} \frac{\lVert (\mathcal{E}(\mat{x}))^T \hat{\mat{n}} \rVert_2} {\lVert \mathcal{E} (\hat{\mat{n}}) \rVert_2} = \inf\limits_{\mat{x} \in \mathcal{M}^{\ast}} \frac{\lVert \mat{x}^T \hat{\mat{n}} \rVert_2} {\lVert \mathcal{E} (\hat{\mat{n}}) \rVert_2},
\end{equation*}
since $\mat{x} \in \mathcal{M}^{\ast} \in \mathcal{M}$, the orthogonal projection $\mathcal{E}(\mat{x}) = \mat{x}$.
Therefore, the manifold margin is the Euclidean margin divided by a constant scalar $\lVert \mathcal{E}(\hat{\mat{n}}) \rVert$,
$d_{\mathcal{M}}(F(\cdot))$ and $d_e(F(\cdot))$ are achieved at the same optimum $\mat{x}^{\ast}$.

\paragraph{Relationship between manifold and Euclidean margins.}
Let $\mat{V} \in \mathbb{R}^{d \times r}$ be a orthonormal basis of the $r$-dimensional embedding subspace.
An angle $\theta$ between the classifier hyperplane and the embedding subspace describes the relationship between $d_e(F(\cdot))$ and $d_{\mathcal{M}}(F(\cdot))$,
\begin{equation*}
    \mathbb{E} \big[\sin{\theta}\big] = \mathbb{E} \bigg[ \frac{d_e(F(\cdot))} {d_{\mathcal{M}}(F(\cdot))} \bigg].
\end{equation*}
Denote $\omega$ as the angle between $\hat{\mat{n}}$ and the embedding subspace, $\theta = \frac{\pi}{2} - \omega$,
\begin{equation*}
    \sin{\theta} = \cos{\omega} = \lVert \mat{V}^T \mat{n} \rVert.
\end{equation*}
Moreover, when the linear classifier forms a random decision boundary, we consider its orthogonal normal vector $\hat{\mat{n}} \sim \mathcal{N} (\mat{0}, \frac{1} {d} \mat{I})$.
Therefore, $\mat{V}^T \hat{\mat{n}} \sim \mathcal{N} (\mat{0}, \frac{1} {d} \mat{V}^T \mat{V})$.
\begin{equation*}
    \mathbb{E}\big[ \lVert \mat{V}^T \hat{\mat{n}} \rVert_2^2 \big] = \frac{1} {d} Tr(\mat{V}^T \mat{V}) = \frac{r} {d}.
\end{equation*}
Then
\begin{equation*}
    \mathbb{E}\big[ (\sin{\theta})^2 \big] = \mathbb{E}\big[ (\cos{\omega})^2 \big] = \frac{r} {d},
\end{equation*}
and from Jensen's inequality,
\begin{equation*}
    \big( \mathbb{E} \big[ \sin{\theta} \big]] \big)^2 \leq \mathbb{E}\big[ (\sin{\theta})^2 \big].
\end{equation*}
Therefore,
\begin{equation*}
    \mathbb{E} \bigg[ \frac{d_e(F(\cdot))} {d_{\mathcal{M}}(F(\cdot))} \bigg] \leq \sqrt{\frac{r} {d}}.
\end{equation*}

\end{proof}

\section{Error Estimation of Linear System}
\label{proof: error estimation of the linear case}
This section derives the error estimation of closed-loop control framework in linear cases.
Given a sequence of states $\{ \mat{x}_t \}_{t=0}^{T-1}$, such that $\mat{x}_t \in \mathcal{M}_t$ for all $t$,
we denote $\boldsymbol\theta_t$ as the linearized transformation of the nonlinear transformation $F_t(\cdot)$ centered at $\mat{x}_t$.
We represent the $t^{th}$ embedding manifold $\mathcal{M}_t = f_t^{-1} (\mat{0})$,
where $f_t(\cdot) : \mathbb{R}^d \rightarrow \mathbb{R}^{d - r}$ is a submersion of class $\mathcal{C}^2$.
Recall Proposition \ref{prop: manifold tangent space projection},
the kernel of $f_t'(\mat{x}_t)$ is equivalent to $\mathcal{T}_{\mat{x}_t}\mathcal{M}_t$,
and the orthogonal projection onto $\mathcal{T}_{\mat{x}_t}\mathcal{M}_t$ (Eq.~\eqref{eq:tangent space orthogonal projection}) is 
\begin{equation*}
    \mat{P}_t \vcentcolon = \mat{I} - f_t'(\mat{x}_t)^T (f_t'(\mat{x}_t) f_t'(\mat{x}_t)^T)^{-1} f_t'(\mat{x}_t),
\end{equation*}
and the orthogonal projection onto orthogonal complement of $\mathcal{T}_{\mat{x}_t}\mathcal{M}_t$ is
\begin{equation*}
    \mat{Q}_t = \mat{I} - \mat{P}_t = f_t'(\mat{x}_t)^T (f_t'(\mat{x}_t) f_t'(\mat{x}_t)^T)^{-1} f_t'(\mat{x}_t).
\end{equation*}
For simplicity, a orthonormal basis of $\mathcal{T}_{\mat{x}_t}\mathcal{M}_t$ is denoted as $\mat{V}_t \in \mathbb{R}^{d \times d}$,
in which case,
the orthogonal projection $\mat{P}_t = \mat{V}_t \mat{V}_t^T$,
and $\mat{Q}_t = \mat{I} - \mat{V}_t \mat{V}_t^T$.

We consider a set of tangent spaces $\{ \mathcal{T}_{\mat{x}_t}\mathcal{M}_t \}_{t=0}^{T-1}$, that is, each $\mathcal{T}_{\mat{x}_t}\mathcal{M}_t$ is the tangent space of $\mathcal{M}_t$ at $\mat{x}_t$.
Recall the running loss in Eq.~\eqref{eq:running loss},
the linear setting uses projection onto a tangent space rather than an nonlinear embedding manifold.
\begin{equation}
    J(\mat{x}_t, \mat{u}_t) = \frac{1} {2} \lVert \mat{Q}_t(\mat{x}_t + \mat{u}_t) \rVert_2^2 + \frac{c} {2} \lVert \mat{u}_t \rVert_2^2,
    \label{eq:prop_cost_function}
\end{equation}
it measures the magnitude of the controlled state $\mat{x}_t + \mat{u}_t$ within the orthogonal complement of $\mathcal{T}_{\mat{x}_t}\mathcal{M}_t$, and the magnitude of applied control $\mat{u}_t$.

The optimal feedback control under Eq.~\eqref{eq:prop_cost_function} is defined as
\begin{equation*}
    \mat{u}_t^{P} (\mat{x}_t) = \arg \min \limits_{\mat{u}_t} J(\mat{x}_t, \mat{u}_t),
\end{equation*}
it admits an exact solution by setting the gradient of performance index (Eq.~\eqref{eq:prop_cost_function}) to $\mat{0}$.
\begin{align*}
    \nabla_{\mat{u}} J(\mat{x}_t, \mat{u}_t) & = \nabla_{\mat{u}} \bigg( \frac{1} {2} \lVert \mat{Q}_t (\mat{x}_t + \mat{u}_t) \rVert_2^2 + \frac{c} {2} \lVert \mat{u}_t \rVert_2^2 \bigg), \nonumber \\
    & = \mat{Q}_t^T \mat{Q}_t \mat{x}_t + \mat{Q}_t^T \mat{Q}_t \mat{u}_t + c \cdot \mat{u}_t,
\end{align*}
which leads to the exact solution of $\mat{u}_t^{P}$ (Eq.~\eqref{eq:solution up}) as
\begin{equation}
    \mat{u}_t^{P} = - (c \cdot \mat{I} + \mat{Q}_t^T \mat{Q}_t)^{-1} \mat{Q}_t^T \mat{Q}_t \mat{x}_t = - \mat{K}_t \mat{x}_t,
    \label{eq:analytic_solution_u}
\end{equation}
where the feedback gain matrix $\mat{K}_t = (c \cdot \mat{I} + \mat{Q}_t^T \mat{Q}_t)^{-1} \mat{Q}_t^T \mat{Q}_t$.
Thus, the one-step feedback control can be represented as $\mat{u}_t^{P} = - \mat{K}_t \mat{x}_t$.

Given a sequence $\{ \mat{x}_t \}_{t=0}^{T-1}$,
we denote $\{ \mat{q}_{\epsilon, t} \}_{t=0}^{T-1}$ as another sequence of states resulted from the linearized system,
$\mat{q}_{\epsilon, 0} = \mat{x}_0 + \mat{z}$, for some perturbation $\mat{z}$,
and $\{ \overline{\mat{q}}_{\epsilon, t} \}_{t=0}^{T-1}$ as the adjusted states by the linear control,
\begin{align*}
    \overline{\mat{q}}_{\epsilon, t+1} & = \boldsymbol\theta_t (\overline{\mat{q}}_{\epsilon, t} + \mat{u}_t^P), \\
    & = \boldsymbol\theta_t (\mat{I} - \mat{K}_t ) \overline{\mat{q}}_{\epsilon, t}.
\end{align*}

The difference between the controlled system applied with perturbation at initial condition and the uncontrolled system without perturbation is follows,
\begin{align}
    \overline{\mat{q}}_{\epsilon, t+1} - \mat{x}_{t+1} & = \boldsymbol\theta_t (\overline{\mat{q}}_{\epsilon, t} + \mat{u}_t - \mat{x}_t), \nonumber \\
    & = \boldsymbol\theta_t (\overline{\mat{q}}_{\epsilon, t} - \mat{K}_t \overline{\mat{q}}_{\epsilon, t} - \mat{x}_t).
    \label{eq:state_difference}
\end{align}
The control objective is to minimize the state components that lie in the orthogonal complement of the tangent space.
When the data locates on the embedding manifold, $\mat{x}_t \in \mathcal{M}_t$,
this results in $\mat{Q}_t \mat{x}_t = \mat{0}$,
consequently, its feedback control $\mat{K}_t \mat{x}_t = \mat{0}$.
The state difference of Eq.~\eqref{eq:state_difference} can be further shown by adding a $\mat{0}$ term of $(\boldsymbol\theta_t \mat{K}_t \mat{x}_t)$
\begin{align}
\label{eq:controlled_dynamics}
    \overline{\mat{q}}_{\epsilon, t+1} - \mat{x}_{t+1} & = \boldsymbol\theta_t (\mat{I} - \mat{K}_t) \overline{\mat{q}}_{\epsilon, t} - \boldsymbol\theta_t \mat{x}_t + \boldsymbol\theta_t \mat{K}_t \mat{x}_t, \nonumber \\
    & = \boldsymbol\theta_t (\mat{I} - \mat{K}_t) (\overline{\mat{q}}_{\epsilon, t} - \mat{x}_t).
\end{align}
In the following, we show a transformation on $(\mat{I} - \mat{K}_t)$ based on its definition.
\begin{lemma}
\label{lemma:control_matrix}
For $t \geq 0$, we have
\begin{equation*}
    \mat{I} - \mat{K}_t = \alpha \cdot \mat{I} + (1 - \alpha) \cdot \mat{P}_t,
\end{equation*}
where $\mat{P}_t \vcentcolon = \mat{V}_t^r (\mat{V}_t^r)^T$, which is the orthogonal projection onto $Z_{\parallel}^{t}$, and $\alpha \vcentcolon = \frac{c} {1 + c}$ such that $\alpha \in [0, 1]$.
\end{lemma}
\begin{proof}
Recall that $\mat{K}_t = (c \cdot \mat{I} + \mat{Q}_t^T \mat{Q}_t)^{-1} \mat{Q}_t^T \mat{Q}_t$, and $\mat{Q}_t = \mat{I} - \mat{V}_t^r (\mat{V}_t^r)^T$, $\mat{Q}_t$ can be diagonalized as following
\begin{gather*}
 \mat{Q}_t
 = \mat{V}_t
  \begin{bmatrix}
   0 & 0 & \cdots & 0 & 0 \\
   0 & 0 & \cdots & 0 & 0 \\
   \vdots & \vdots & \ddots & 0 & 0 \\
   0 & 0 & \cdots & 1 & 0 \\
   0 & 0 & \cdots & 0 & 1 \\
   \end{bmatrix}
   \mat{V}_t^T,
\end{gather*}
where the first $r$ diagonal elements have common value of $0$ and the last ($d-r$) diagonal elements have common value of $1$. Furthermore, the feedback gain matrix $\mat{K}_t$ can be diagonalized as
\begin{gather*}
 \mat{K}_t
 = \mat{V}_t
  \begin{bmatrix}
   0 & 0 & \cdots & 0 & 0 \\
   0 & 0 & \cdots & 0 & 0 \\
   \vdots & \vdots & \ddots & 0 & 0 \\
   0 & 0 & \cdots & \frac{1} {1+c} & 0 \\
   0 & 0 & \cdots & 0 & \frac{1} {1+c} \\
   \end{bmatrix}
   \mat{V}_t^T,
\end{gather*}
where the last ($d-r$) diagonal elements have common value of $\frac{1} {1+c}$. The control term $(\mat{I} - \mat{K}_t)$ thus can be represented as
\begin{gather*}
 \mat{I} - \mat{K}_t
 = \mat{V}_t
  \begin{bmatrix}
   1 & 0 & \cdots & 0 & 0 \\
   0 & 1 & \cdots & 0 & 0 \\
   \vdots & \vdots & \ddots & 0 & 0 \\
   0 & 0 & \cdots & \frac{c} {1+c} & 0 \\
   0 & 0 & \cdots & 0 & \frac{c} {1+c} \\
   \end{bmatrix}
   \mat{V}_t^T,
\end{gather*}
where the first $r$ diagonal elements have common value of $1$ and the last ($d-r$) diagonal elements have common value of $\frac{c} {1+c}$. By denoting the projection of first $r$ columns as $\mat{V}_t^r$ and last $(d-r)$ columns as $\hat{\mat{V}}_t^r$, it can be further shown as
\begin{align*}
    \mat{I} - \mat{K}_t & = \mat{V}_t^r (\mat{V}_t^r)^T + \frac{c} {1+c} \big( \hat{\mat{V}}_t^r (\hat{\mat{V}}_t^r)^T \big), \\
    & = \mat{P}_t + \alpha \big( \mat{I} - \mat{P}_t \big), \\
    & = \alpha \cdot \mat{I} + (1 - \alpha) \cdot \mat{P}_t.
\end{align*}
\end{proof}

\begin{lemma}
\label{lemma:pts}
Define for $t \geq 0$
\[ \begin{cases}
      \mat{P}_t^0 \vcentcolon = \mat{P}_t, \\
      \mat{P}_t^{(s+1)} \vcentcolon = \boldsymbol\theta_{t-s-1}^{-1} \mat{P}_t^s \boldsymbol\theta_{t-s-1} , \hspace{0.3cm} s = 0, 1, \ldots, t-1,
   \end{cases}
\]
for $0 \leq s \leq t$. Then
\begin{enumerate}[leftmargin=*]
   \item $\mat{P}_t^s$ is a projection.
   \item $\mat{P}_t^s$ is a projection onto $Z_{\parallel}^{t-s}$, i.e. $range(\mat{P}_t^s) = Z_{\parallel}^{t-s}$.
\end{enumerate}
\end{lemma}

\begin{proof}
\begin{enumerate}[leftmargin=*]
   \item We prove it by induction on $s$ for each $t$. For $s=0$, $\mat{P}_t^0 = \mat{P}_t$, which is a projection by its definition. Suppose it is true for $s$ such that $\mat{P}_t^s = \mat{P}_t^s \mat{P}_t^s$, then for $(s+1)$,
   \begin{align*}
       (\mat{P}_t^{s+1})^2 & = \big( \boldsymbol\theta_{t-s-1}^{-1} \mat{P}_t^s \boldsymbol\theta_{t-s-1} \big)^2, \\
       & = \boldsymbol\theta_{t-s-1}^{-1} \big( \mat{P}_t^s \big)^2 \boldsymbol\theta_{t-s-1}, \\
       & = \boldsymbol\theta_{t-s-1}^{-1} \mat{P}_t^s \boldsymbol\theta_{t-s-1}, \\
       & = \mat{P}_t^{s+1}.
   \end{align*}
   \item We prove it by induction on $s$ for each $t$. For $s=0$, $\mat{P}_t^0 = \mat{P}_t$, which is the orthogonal projection onto $Z_{\parallel}^{t}$. Suppose that it is true for $s$ such that $\mat{P}_t^s$ is a projection onto $Z_{\parallel}^{t-s}$, then for $(s+1)$, $\mat{P}_t^{s+1} = \boldsymbol\theta_{t-s-1}^{-1} \mat{P}_t^s \boldsymbol\theta_{t-s-1}$, which implies
   \begin{align*}
       range(\mat{P}_t^{s+1}) & = range(\boldsymbol\theta_{t-s-1}^{-1} \mat{P}_t^s), \\
       & = \{ \boldsymbol\theta_{t-s-1}^{-1} \mat{x}: \mat{x} \in Z_{\parallel}^{t-s} \}, \\
       & = Z_{\parallel}^{t-s-1}.
   \end{align*}
\end{enumerate}
\end{proof}
The following Lemma reformulates the state difference equation.
\begin{lemma}
\label{lemma:gts}
Define for $0 \leq s \leq t$, 
\begin{equation*}
    \mat{G}_t^s \vcentcolon = \alpha \cdot \mat{I} + (1 - \alpha) \mat{P}_t^s.
\end{equation*}
The state difference equation, $\overline{\mat{q}}_{\epsilon, t+1} - \mat{x}_{t+1} = \boldsymbol\theta_t (\mat{I} - \mat{K}_t) (\overline{\mat{q}}_{\epsilon, t} - \mat{x}_t)$, can be written as
\begin{equation*}
    \overline{\mat{q}}_{\epsilon, t} - \mat{x}_{t} = (\boldsymbol\theta_{t-1} \boldsymbol\theta_{t-2} \cdots \boldsymbol\theta_{0}) (\mat{G}_{t-1}^{t-1} \mat{G}_{t-2}^{t-2} \cdots \mat{G}_{0}^{0}) (\overline{\mat{q}}_{\epsilon, 0} - \mat{x}_{0}), \;\; t \geq 1.
\end{equation*}
\end{lemma}

\begin{proof}
We prove it by induction on $t$. For $t=1$, 
\begin{align*}
    \overline{\mat{q}}_{\epsilon, 1} - \mat{x}_{1} & = \boldsymbol\theta_0 (\mat{I} - \mat{K}_0) (\overline{\mat{q}}_{\epsilon, 0} - \mat{x}_{0}), & \\
    & = \boldsymbol\theta_0 (\alpha \cdot \mat{I} + (1 - \alpha) \cdot \mat{P}_0) (\overline{\mat{q}}_{\epsilon, 0} - \mat{x}_{0}), & \textnormal{Lemma}~ \ref{lemma:control_matrix}, \\
    & = \boldsymbol\theta_0 \mat{G}_0^0 (\overline{\mat{q}}_{\epsilon, 0} - \mat{x}_{0}).
\end{align*}
Recall the definitions of $\mat{P}_t^{(s+1)} \vcentcolon = \boldsymbol\theta_{t-s-1}^{-1} \mat{P}_t^s \boldsymbol\theta_{t-s-1}$, and $\mat{G}_t^s \vcentcolon = \alpha \cdot \mat{I} + (1 - \alpha) \mat{P}_t^s$,
\begin{align*}
    \mat{G}_t^{s+1} & = \alpha \cdot \mat{I} + (1 - \alpha) \cdot \mat{P}_t^{(s+1)}, \\
    & = \alpha \cdot \mat{I} + (1 - \alpha) \cdot \boldsymbol\theta_{t-s-1}^{-1} \mat{P}_t^s \boldsymbol\theta_{t-s-1}, \\
    & = \boldsymbol\theta_{t-s-1}^{-1} \big( \alpha \cdot \mat{I} + (1 - \alpha) \cdot \mat{P}_t^s \big) \boldsymbol\theta_{t-s-1}, \\
    & = \boldsymbol\theta_{t-s-1}^{-1} \mat{G}_t^s \boldsymbol\theta_{t-s-1},
\end{align*}
which results in $\boldsymbol\theta_{t-s-1} \mat{G}_t^{(s+1)} = \mat{G}_t^s \boldsymbol\theta_{t-s-1}$.
Suppose that it is true for $(\overline{\mat{q}}_{\epsilon, t} - \mat{x}_{t})$,
\begin{align*}
\label{eq:lemma_gts}
    \overline{\mat{q}}_{\epsilon, t+1} - \mat{x}_{t+1} & = \boldsymbol\theta_t (\mat{I} - \mat{K}_t) (\overline{\mat{q}}_{\epsilon, t} - \mat{x}_{t}), &  \nonumber \\
    & = \boldsymbol\theta_t (\alpha \cdot \mat{I} + (1 - \alpha) \cdot \mat{P}_t) (\overline{\mat{q}}_{\epsilon, t} - \mat{x}_{t}), & \textnormal{Lemma}~ \ref{lemma:control_matrix}, \nonumber \\
    & = \boldsymbol\theta_t \mat{G}_t^0 (\boldsymbol\theta_{t-1} \boldsymbol\theta_{t-2} \cdots \boldsymbol\theta_{0}) (\mat{G}_{t-1}^{t-1} \mat{G}_{t-2}^{t-2} \cdots \mat{G}_{0}^{0}) (\overline{\mat{q}}_{\epsilon, 0} - \mat{x}_{0}), \nonumber \\
    & = (\boldsymbol\theta_t \boldsymbol\theta_{t-1}) \mat{G}_t^1 (\boldsymbol\theta_{t-2} \boldsymbol\theta_{t-3} \cdots \boldsymbol\theta_0) (\mat{G}_{t-1}^{t-1} \mat{G}_{t-2}^{t-2} \cdots \mat{G}_{0}^{0}) (\overline{\mat{q}}_{\epsilon, 0} - \mat{x}_{0}), \nonumber \\
    & = (\boldsymbol\theta_t \boldsymbol\theta_{t-1} \cdots \boldsymbol\theta_{0}) (\mat{G}_{t}^{t} \mat{G}_{t-1}^{t-1} \cdots \mat{G}_{0}^{0}) (\overline{\mat{q}}_{\epsilon, 0} - \mat{x}_{0}).
\end{align*}
\end{proof}

\begin{lemma}
\label{lemma:Ft}
For $t \geq 1$,
\begin{equation*}
    \mat{G}_{t-1}^{(t-1)} \mat{G}_{t-2}^{(t-2)} \cdots \mat{G}_{0}^{0} = \alpha^t \cdot \mat{I} + (1 - \alpha) \sum_{s=0}^{t-1} \alpha^s \mat{P}_s^s.
\end{equation*}
\end{lemma}

\begin{proof}
We prove it by induction on $t$. Recall the definition of $\mat{G}_t^s \vcentcolon = \alpha \cdot \mat{I} + (1 - \alpha) \cdot \mat{P}_t^s$. When $t = 1$, 
\begin{equation*}
    \mat{G}_0^0 = \alpha \cdot \mat{I} + (1 - \alpha) \cdot \mat{P}_0^0.
\end{equation*}
Suppose that it is true for $t$ such that
\begin{equation*}
    \mat{G}_{t-1}^{(t-1)} \mat{G}_{t-2}^{(t-2)} \cdots \mat{G}_{0}^{0} = \alpha^t \cdot \mat{I} + (1 - \alpha) \sum_{s=0}^{t-1} \alpha^s \mat{P}_s^s,
\end{equation*}
for $(t+1)$,
\begin{align*}
    & \mat{G}_t^t \mat{G}_{t-1}^{(t-1)} \cdots \mat{G}_0^0 \nonumber \\
    & = \mat{G}_t^t (\alpha^t \cdot \mat{I} + (1 - \alpha) \sum_{s=0}^{t-1} \alpha^s \mat{P}_s^s), \nonumber \\
    & = (\alpha \cdot \mat{I} + (1 - \alpha) \cdot \mat{P}_t^t) (\alpha^t \cdot \mat{I} + (1 - \alpha) \sum_{s=0}^{t-1} \alpha^s \mat{P}_s^s), \\
    & = \alpha^{t+1} \cdot \mat{I} + \alpha^t (1 - \alpha) \mat{P}_t^t + (1 - \alpha)^2 \sum_{s=0}^{t-1} \alpha^s \cdot \mat{P}_t^t \mat{P}_s^s + \alpha (1 - \alpha) \sum_{s=0}^{t-1} \alpha^s \cdot \mat{P}_s^s.
\end{align*}
Recall Lemma \ref{lemma:pts}, $range(\mat{P}_t^t) = range(\mat{P}_s^s) = Z_{\parallel}^{0}$. 
Since $\mat{P}_t^t$ and $\mat{P}_s^s$ are projections onto the same space,
$\mat{P}_t^t \mat{P}_s^s = \mat{P}_s^s$.
Therefore,
\begin{align*}
    \mat{G}_t^t \mat{G}_{t-1}^{(t-1)} \cdots \mat{G}_0^0
    & = \alpha^{t+1} \cdot \mat{I} + \alpha^t (1 - \alpha ) \cdot \mat{P}_t^t + (1 - \alpha) \sum_{s=0}^{t-1} \alpha^s \cdot \mat{P}_s^s, \nonumber \\
    & = \alpha^{t+1} \cdot \mat{I} + (1 - \alpha) \sum_{s=0}^{t} \alpha^s \cdot \mat{P}_s^s.
\end{align*}
\end{proof}

\begin{lemma}
\label{lemma:condi}
Let $\mat{P} = \mat{V} \mat{V}^T$ be the orthogonal projection onto a subspace $\mathcal{D}$,
and $\boldsymbol\theta$ to be invertible.
Denote by $\hat{\mat{P}}$ the orthogonal projection onto $\boldsymbol\theta \mathcal{D} \vcentcolon = \{ \boldsymbol\theta \mat{x}: ~ \mat{x} \in \mathcal{D} \}$.
Then
\begin{equation*}
    \lVert \boldsymbol\theta^{-1} \hat{\mat{P}} \boldsymbol\theta - \mat{P} \rVert_2 \leq \big(1 + \kappa(\boldsymbol\theta)^2 \big) \cdot \lVert \mat{I} - \boldsymbol\theta^T \boldsymbol\theta \rVert_2.
\end{equation*}
\end{lemma}

\begin{proof}
\begin{align*}
   \hat{\mat{P}} & = \boldsymbol\theta \mat{V} \big[ (\boldsymbol\theta \mat{V})^T (\boldsymbol\theta \mat{V}) \big]^{-1} (\boldsymbol\theta \mat{V})^T, \\
   & = \boldsymbol\theta \mat{V} \big[ \mat{V}^T \boldsymbol\theta^T \boldsymbol\theta \mat{V} \big]^{-1} \mat{V}^T \boldsymbol\theta^T. 
\end{align*}
Furthermore, the difference between the oblique projection and orthogonal projection can be bounded by the follows
\begin{align*}
   \lVert \boldsymbol\theta^{-1} \hat{\mat{P}} \boldsymbol\theta - \mat{P} \rVert_2 & = \lVert \mat{V} \big[ \mat{V}^T \boldsymbol\theta^T \boldsymbol\theta \mat{V} \big]^{-1} \mat{V}^T \boldsymbol\theta^T  \boldsymbol\theta - \mat{V} \mat{V}^T \rVert_2, \\
   & \leq \lVert \mat{V} \big[ \mat{V}^T \boldsymbol\theta^T \boldsymbol\theta \mat{V} \big]^{-1} \mat{V}^T \boldsymbol\theta^T  \boldsymbol\theta - \mat{V} \mat{V}^T \boldsymbol\theta^T \boldsymbol\theta \rVert_2 + \lVert \mat{V} \mat{V}^T \boldsymbol\theta^T \boldsymbol\theta - \mat{V} \mat{V}^T \rVert_2, \\
   & \leq \lVert \mat{V} \big( [\mat{V}^T \boldsymbol\theta^T \boldsymbol\theta \mat{V}]^{-1} - \mat{I} \big) \mat{V}^T \rVert_2 \cdot \lVert \boldsymbol\theta^T \boldsymbol\theta \rVert_2 +  \lVert \boldsymbol\theta^T \boldsymbol\theta - \mat{I} \rVert_2, \\
   & \leq \lVert [ \mat{V}^T \boldsymbol\theta^T \boldsymbol\theta \mat{V}]^{-1} \rVert_2 \cdot \lVert \mat{I} - \mat{V}^T \boldsymbol\theta^T \boldsymbol\theta \mat{V} \rVert_2 \cdot \lVert \boldsymbol\theta^T \boldsymbol\theta \rVert_2 + \lVert \boldsymbol\theta^T \boldsymbol\theta - \mat{I} \rVert_2, \\
   & \leq \lVert [\mat{V}^T \boldsymbol\theta^T \boldsymbol\theta \mat{V}]^{-1} \rVert_2 \cdot \lVert \mat{I} - \boldsymbol\theta^T \boldsymbol\theta \rVert_2 \cdot \lVert \boldsymbol\theta^T \boldsymbol\theta \rVert_2 + \lVert \boldsymbol\theta^T \boldsymbol\theta - \mat{I} \rVert_2, \\
   & = \big( \lambda_{min} (\mat{V}^T \boldsymbol\theta^T \boldsymbol\theta \mat{V}) \big)^{-1} \cdot \lVert \mat{I} - \boldsymbol\theta^T \boldsymbol\theta \rVert_2 \cdot \lVert \boldsymbol\theta^T \boldsymbol\theta \rVert_2 + \lVert \boldsymbol\theta^T \boldsymbol\theta - \mat{I} \rVert_2, \\
   & = \big( \inf \limits_{\lVert \mat{x} \rVert_2 = 1} \mat{x}^T \mat{V}^T \boldsymbol\theta^T \boldsymbol\theta \mat{V} \mat{x} \big)^{-1} \cdot \lVert \mat{I} - \boldsymbol\theta^T \boldsymbol\theta \rVert_2 \cdot \lVert \boldsymbol\theta^T \boldsymbol\theta \rVert_2 + \lVert \boldsymbol\theta^T \boldsymbol\theta - \mat{I} \rVert_2, \\
   & \leq \big( \inf \limits_{\lVert \mat{x}' \rVert_2 = 1} (\mat{x}')^T \boldsymbol\theta^T \boldsymbol\theta \mat{x}' \big)^{-1} \cdot \lVert \mat{I} - \boldsymbol\theta^T \boldsymbol\theta \rVert_2 \cdot \lVert \boldsymbol\theta^T \boldsymbol\theta \rVert_2 + \lVert \boldsymbol\theta^T \boldsymbol\theta - \mat{I} \rVert_2, \\
   & = \big( \lambda_{min} (\boldsymbol\theta^T \boldsymbol\theta) \big)^{-1} \cdot \lVert \mat{I} - \boldsymbol\theta^T \boldsymbol\theta \rVert_2 \cdot \lVert \boldsymbol\theta^T \boldsymbol\theta \rVert_2 + \lVert \boldsymbol\theta^T \boldsymbol\theta - \mat{I} \rVert_2, \\
   & = \lVert (\boldsymbol\theta^T \boldsymbol\theta)^{-1} \rVert_2 \cdot \lVert \mat{I} - \boldsymbol\theta^T \boldsymbol\theta \rVert_2 \cdot \lVert \boldsymbol\theta^T \boldsymbol\theta \rVert_2 + \lVert \boldsymbol\theta^T \boldsymbol\theta - \mat{I} \rVert_2, \\
   & = \big( 1 + \kappa(\boldsymbol\theta)^2 \big) \cdot \lVert \mat{I} - \boldsymbol\theta^T \boldsymbol\theta \rVert_2.
\end{align*}
\end{proof}

\begin{corollary}
\label{corollary: difference of projections}
Let $t \geq 1$. Then for each $s = 0, 1, \cdots, t$, we have
\begin{equation*}
    \lVert \mat{P}_s^s - \mat{P}_0 \rVert_2 \leq \big( 1 + \kappa(\overline{\boldsymbol\theta}_s)^2 \big)\cdot \lVert \mat{I} - \overline{\boldsymbol\theta}_s^T \overline{\boldsymbol\theta}_s \rVert_2, 
\end{equation*}
where 
\begin{itemize}
    \item $\overline{\boldsymbol\theta}_s \vcentcolon = \boldsymbol\theta_{s-1} \cdots \boldsymbol\theta_0, ~ s \geq 1$,
    \item $\overline{\boldsymbol\theta}_s \vcentcolon = \mat{I}, ~ s = 0$.
\end{itemize}
\end{corollary}
The following theorem provides an error estimation for the linearized dynamic system with linear controls.
\maintheoremLinear*
\begin{proof}
The input perturbation $\mat{z} = \overline{\mat{q}}_{\epsilon, 0} - \mat{x}_{0}$ can be written as $\mat{z} = \mat{z}^{\parallel} + \cdot \mat{z}^{\perp}$, where $\mat{z}^{\parallel} \in Z_{\parallel}$ and $\mat{z}^{\perp} \in Z_{\perp}$, where $\mat{z}^{\parallel}$ and $\mat{z}^{\perp}$ are vectors such that
\begin{itemize}
    \item $\mat{z}^{\parallel} \cdot \mat{z}^{\perp} = 0$ almost surely.
    \item $\mat{z}^{\parallel}$, $\mat{z}^{\perp}$ have uncorrelated components.
\end{itemize}
Recall Lemma \ref{lemma:gts},
\begin{align}
    \lVert \overline{\mat{q}}_{\epsilon, t} - \mat{x}_{t} \rVert_2^2 & = \lVert (\boldsymbol\theta_{t-1} \boldsymbol\theta_{t-2} \cdots \boldsymbol\theta_{0}) (\mat{G}_{t-1}^{t-1} \cdots \mat{G}_{0}^{0}) \mat{z} \rVert_2^2, \nonumber \\
    & \leq \lVert \boldsymbol\theta_{t-1} \boldsymbol\theta_{t-2} \cdots \boldsymbol\theta_{0} \rVert_2^2 \cdot \lVert (\mat{G}_{t-1}^{t-1} \cdots \mat{G}_{0}^{0}) \mat{z} \rVert_2^2,
    \label{eq:error_estimation}
\end{align}
For the term $ \lVert (\mat{G}_{t-1}^{t-1} \mat{G}_{t-2}^{t-2} \cdots \mat{G}_{0}^{0}) \mat{z} \rVert_2^2$, recall Lemma \ref{lemma:Ft},
\begin{align*}
    \lVert (\mat{G}_{t-1}^{t-1} \cdots \mat{G}_{0}^{0}) \mat{z} \rVert_2^2 & = \lVert \Big( \alpha^t \cdot \mat{I} + (1 - \alpha) \sum_{s=0}^{t-1} \alpha^s \cdot \mat{P}_s^s \Big) \mat{z} \rVert_2^2, \\
    & = \lVert \alpha^t \mat{z} + (1 - \alpha) \sum_{s=0}^{t-1} \alpha^s \mat{P}_0 \mat{z} + (1 - \alpha) \sum_{s=0}^{t-1} \alpha^s (\mat{P}_s^s - \mat{P}_0) \mat{z} \rVert_2^2, \\
    & = \lVert \alpha^t \mat{z} + (1 - \alpha^t) \mat{z}^{\parallel} + (1 - \alpha) \sum_{s=0}^{t-1} \alpha^s (\mat{P}_s^s - \mat{P}_0) \mat{z} \rVert_2^2, 
\end{align*}
in the above, $\mat{P}_0$ is an orthogonal projection on $t=0$ (input data space), therefore, $\mat{P}_0 \mat{z} = \mat{z}^{\parallel}$. Furthermore, when $s = 0$, $\mat{P}_s^s - \mat{P}_0 = \mat{0}$. Thus,
\begin{align*}
    & \lVert (\mat{G}_{t-1}^{t-1} \cdots \mat{G}_{0}^{0}) \mat{z} \rVert_2^2 \\
    & = \alpha^{2t} \lVert \mat{z} \rVert_2^2 + (1 - \alpha^t)^2 \lVert \mat{z}^{\parallel} \rVert_2^2 + (1 - \alpha)^2 \sum_{s, q=1}^{t-1} \alpha^s \alpha^q \mat{z}^T (\mat{P}_s^s - \mat{P}_0)^T (\mat{P}_q^q - \mat{P}_0) \mat{z} \\ 
    & \;\;\;\; + 2 \alpha^t (1 - \alpha^t) \lVert \mat{z}^{\parallel} \rVert_2^2 + 2 \alpha^t (1 - \alpha) \sum_{s=1}^{t-1} \alpha^s \mat{z}^T (\mat{P}_s^s - \mat{P}_0) \mat{z} \\
    & \;\;\;\; + 2 (1 - \alpha^t) (1 - \alpha) \sum_{s=1}^{t-1} \alpha^s (\mat{z}^{\parallel})^T (\mat{P}_s^s - \mat{P}_0) \mat{z}, \\
    & = \alpha^{2t} \lVert \mat{z}^{\perp} \rVert_2^2 + \big( \alpha^{2t} + 2 \alpha^t (1 - \alpha^t) + (1 - \alpha^t)^2 \big) \lVert \mat{z}^{\parallel} \rVert_2^2 \\
    & \;\;\;\; + (1 - \alpha)^2 \sum_{s, q=1}^{t-1} \alpha^s \alpha^q \mat{z}^T (\mat{P}_s^s - \mat{P}_0)^T (\mat{P}_q^q - \mat{P}_0) \mat{z} + 2 \alpha^t (1 - \alpha) \sum_{s=1}^{t-1} \alpha^s \mat{z}^T (\mat{P}_s^s - \mat{P}_0) \mat{z} \\
    & \;\;\;\; + 2 (1 - \alpha^t) (1 - \alpha) \sum_{s=1}^{t-1} \alpha^s (\mat{z}^{\parallel})^T (\mat{P}_s^s - \mat{P}_0) \mat{z}, \\
    & = \alpha^{2t} \lVert \mat{z}^{\perp} \rVert_2^2 + \lVert \mat{z}^{\parallel} \rVert_2^2 + (1 - \alpha)^2 \sum_{s, q=1}^{t-1} \alpha^s \alpha^q \mat{z}^T (\mat{P}_s^s - \mat{P}_0)^T (\mat{P}_q^q - \mat{P}_0) \mat{z} \\
    & \;\;\;\; + 2 \alpha^t (1 - \alpha) \sum_{s=1}^{t-1} \alpha^s \mat{z}^T (\mat{P}_s^s - \mat{P}_0) \mat{z} + 2 (1 - \alpha^t) (1 - \alpha) \sum_{s=1}^{t-1} \alpha^s (\mat{z}^{\parallel})^T (\mat{P}_s^s - \mat{P}_0) \mat{z}.
\end{align*}
Using Corollary \ref{corollary: difference of projections}, we have
\begin{itemize}[leftmargin=*]
   \item 
   \begin{align*}
       \mat{z}^T (\mat{P}_s^s - \mat{P}_0) \mat{z} & \leq \lVert \mat{z} \rVert_2^2 \cdot \lVert \mat{P}_s^s - \mat{P}_0 \rVert, \\
       & \leq \gamma_t \lVert \mat{z} \rVert_2^2.
   \end{align*}
   \item
   \begin{align*}
       \mat{z}^T (\mat{P}_s^s - \mat{P}_0)^T (\mat{P}_q^q - \mat{P}_0) \mat{z} & \leq \lVert \mat{z} \rVert_2^2 \cdot \lVert \mat{P}_s^s - \mat{P}_0 \rVert \cdot \lVert \mat{P}_q^q - \mat{P}_0 \rVert, \\
       & \leq \gamma_t^2 \lVert \mat{z} \rVert_2^2.
   \end{align*}
   \item
   \begin{align*}
       (\mat{z}^{\parallel})^T (\mat{P}_s^s - \mat{P}_0) \mat{z} & \leq \gamma_t \lVert \mat{z}^{\parallel} \rVert_2 \cdot \lVert \mat{z} \rVert_2, \\
       & \leq \gamma_t \lVert \mat{z} \rVert_2^2.
   \end{align*}
\end{itemize}
Thus, we have
\begin{align*}
    \lVert (\mat{G}_{t-1}^{t-1} \cdots \mat{G}_{0}^{0}) \mat{z} \rVert_2^2 & \leq \alpha^{2t} \lVert \mat{z}^{\perp} \rVert_2^2 + \lVert \mat{z}^{\parallel} \rVert_2^2 + \alpha^2 (1 - \alpha^{t-1})^2 \gamma_t^2 \lVert \mat{z} \rVert_2^2 + 2 \alpha^{t+1} (1 - \alpha^{t-1}) \gamma_t \lVert \mat{z} \rVert_2^2 \\
    & \;\;\;\; + 2 \alpha (1 - \alpha^t) (1 - \alpha^{t - 1}) \gamma_t \lVert \mat{z} \rVert_2^2, \\
    & = \alpha^{2t} \lVert \mat{z}^{\perp} \rVert_2^2 + \lVert \mat{z}^{\parallel} \rVert_2^2 + \gamma_t \lVert \mat{z} \rVert_2^2 \big( \gamma_t \alpha^2 (1 - \alpha^{t-1})^2 + 2 (\alpha - \alpha^t) \big).
\end{align*}
Recall the error estimation in Eq.~\eqref{eq:error_estimation}, 
\begin{align*}
    \lVert \overline{\mat{q}}_{\epsilon, t} - \mat{x}_{t} \rVert_2^2 & \leq \lVert \boldsymbol\theta_{t-1} \boldsymbol\theta_{t-2} \cdots \boldsymbol\theta_{0} \rVert_2^2 \cdot \lVert (\mat{G}_{t-1}^{t-1} \cdots \mat{G}_{0}^{0}) \mat{z} \rVert_2^2, \\
    & \leq \lVert \boldsymbol\theta_{t-1} \cdots \boldsymbol\theta_{0} \rVert_2^2 \cdot \bigg( \alpha^{2t} \lVert \mat{z}^{\perp} \rVert_2^2 + \lVert \mat{z}^{\parallel} \rVert_2^2 + \gamma_t \lVert \mat{z} \rVert_2^2 \big( \gamma_t \alpha^2 (1 - \alpha^{t-1})^2 + 2 (\alpha - \alpha^t) \big) \bigg).
\end{align*}
In the specific case, when all $\boldsymbol\theta_t$ are orthogonal, 
\begin{align*}
    \gamma_t \vcentcolon & = \max \limits_{s \leq t} \big(1 + \kappa (\overline{\boldsymbol\theta}_s)^2 \big) \lVert \mat{I} - \overline{\boldsymbol\theta}_s^T \overline{\boldsymbol\theta}_s \rVert_2 \\
    & = 0.
\end{align*}
Thus, 
\begin{align*}
    \lVert \overline{\mat{q}}_{\epsilon, t} - \mat{x}_{t} \rVert_2^2 = \alpha^{2t} \lVert \mat{z}^{\perp} \rVert_2^2 + \lVert \mat{z}^{\parallel} \rVert_2^2.
\end{align*}
\end{proof}

\section{Error Estimation of Nonlinear System}
\label{proof: error estimation of nonlinear system}
In this section, we analyze the error $ \lVert \overline{\mat{x}}_{\epsilon, t} - \mat{x}_t \rVert_2 $ via the following steps:
\begin{itemize}
    \item Appendix~\ref{sec:Analysis On Manifold Projection} considers two solutions of the running loss Eq.~\eqref{eq:running loss} where the projections are defined based on an embedding manifold and a tangent space respectively.
    An $\mathcal{O} (\epsilon^2)$ error estimation is derived for the difference between those two solutions.
    \item Appendix~\ref{sec:Analysis On Linearization Error} provides an $\mathcal{O}(\epsilon^2)$ solution for the linearization error (defined later).
    \item Finally, 
    Appendix~\ref{sec:Error Estimation} derives an upper bound for the total error $\lVert \overline{\mat{x}}_{\epsilon, t} - \mat{x}_t \rVert_2$.
\end{itemize}

\subsection{Analysis On Nonlinear Manifold Projection}
\label{sec:Analysis On Manifold Projection}
\paragraph{Definition for the tangent space $\mathcal{T}_{\mat{x}}\mathcal{M}$ based on the submersion $f(\cdot)$.}

\begin{restatable}{proposition}{tangentspaceofmanifold}
\label{prop: manifold tangent space projection}
Let $\mathcal{M} \subset \mathbb{R}^d$ be an $r$-dimensional smooth manifold and $\mat{x} \in \mathcal{M}$.
Given a submersion $f(\cdot) : \mathbb{R}^d \rightarrow \mathbb{R}^{d - r}$ of class $\mathcal{C}^1$,
such that $\mathcal{M} = f^{-1} (\mat{0})$.
Then the tangent space at any $\mat{x} \in \mathcal{M}$ is the kernel of the linear map $f'(\mat{x})$, i.e.,  $\mathcal{T}_{\mat{x}}\mathcal{M} = {\rm Ker} f'(\mat{x})$.
\end{restatable}

\begin{proof}
For any $\mat{x} \in \mathcal{M}$ and $\mat{v} \in \mathcal{T}_{\mat{x}}\mathcal{M}$,
suppose that there is an open interval $\mathcal{J} \in \mathbb{R}$ such that $0 \in \mathcal{J}$,
and a smooth curve $\gamma: \mathcal{J} \rightarrow \mathcal{M}$ such that $\gamma (0) = \mat{x}$,
$\gamma'(0) = \mat{v}$.
Since $f(\mat{x}) = \mat{0}, ~ \forall \mat{x} \in \mathcal{M}$,
and $\gamma(\lambda) \in \mathcal{M}, ~ \forall \lambda \in \mathcal{J}$,
\begin{equation*}
    f\circ\gamma(\lambda) = \mat{0}, ~ \lambda \in \mathcal{J}.
\end{equation*}
Therefore, $f\circ\gamma(\lambda)$ is a constant map for all $\lambda \in \mathcal{J}$,
\begin{equation*}
    \mat{0} = (f \circ \gamma)' (0) = f'(\gamma(0)) \gamma'(0) = f'(\mat{x}) \mat{v},
\end{equation*}
since $\mat{v} \in \mathcal{T}_{\mat{x}}\mathcal{M}$ is arbitrarily chosen from $\mathcal{T}_{\mat{x}}\mathcal{M}$,
$f'(\mat{x}) \mat{v} = \mat{0}, ~ \forall \mat{v} \in \mathcal{T}_{\mat{x}}\mathcal{M}$.
Therefore, $\mathcal{T}_{\mat{x}}\mathcal{M} \in ker f'(\mat{x})$ (the kernel of linear map $f'(\mat{x})$).

Recall that $f: \mathbb{R}^d \rightarrow \mathbb{R}^{d - r}$ is a submersion, its differential $f'(\mat{x})$ is a surjective linear map with constant rank for all $\mat{x} \in \mathcal{M}$.
\begin{equation*}
    dim (ker f'(\mat{x})) = dim (\mathbb{R}^r) - rank (f'(\mat{x})) = d - (d - r) = r.
\end{equation*}
Since $\mathcal{T}_{\mat{x}}\mathcal{M} \in ker f'(\mat{x})$ and $dim (\mathcal{T}_{\mat{x}}\mathcal{M}) = dim (ker f'(\mat{x}))$, $\mathcal{T}_{\mat{x}}\mathcal{M} = ker f'(\mat{x})$.
\end{proof}

\paragraph{Definitions for the control solutions of running loss.}
Given a smooth manifold $\mathcal{M}$,
we can attach to every point $\mat{x} \in \mathcal{M}$ a tangent space $\mathcal{T}_{\mat{x}}\mathcal{M}$. 
Proposition \ref{prop: manifold tangent space projection} has shown the equivalence between the kernel of $f'(\mat{x})$ and the tangent space $\mathcal{T}_{\mat{x}}\mathcal{M}$.
Therefore, $f'(\mat{x})$ consists a basis of the complement of the tangent space $\mathcal{T}_{\mat{x}}\mathcal{M}$.
For simplicity, we assume the submersion to be normalized such that the columns of $f'(\mat{x})$ consist of a orthonormal basis.
In this case, the orthogonal projection onto $\mathcal{T}_{\mat{x}}\mathcal{M}$ can be defined as follows,
\begin{equation}
    \label{eq:tangent space orthogonal projection}
    \mat{P}_{\mat{x}} \vcentcolon = \mat{I} - f'(\mat{x})^T f'(\mat{x}).
\end{equation}
In general cases, when $f'(\mat{x})$ does not consist orthonormal basis, the orthogonal projection in Eq.~\eqref{eq:tangent space orthogonal projection} can be defined by adding a scaling factor as follows,
\begin{equation*}
    \mat{P}_{\mat{x}} \vcentcolon = \mat{I} - f'(\mat{x})^T (f'(\mat{x}) f'(\mat{x})^T)^{-1}f'(\mat{x}).
\end{equation*}
The orthogonal projection onto the orthogonal complement of $\mathcal{T}_{\mat{x}}\mathcal{M}$ is defined as follows,
\begin{equation*}
    \mat{Q}_{\mat{x}} \vcentcolon = \mat{I} - \mat{P}_{\mat{x}} = f'(\mat{x})^T f'(\mat{x}).
\end{equation*}

Recall that an general embedding manifold is defined by a submersion, such that $\mathcal{M} = f^{-1} (\mat{0})$.
In the linear case, an embedding manifold is considered as a linear sub-space,
this linear sub-space can be defined by a submersion $\mathcal{M} = (f'(\mat{x}))^{-1} \mat{0} = f'(\mat{x})^T \mat{0}$,
in which case, the submersion is a linear operator $f'(\mat{x})$.
In this linear case, we denote $\mat{u}_{\mat{x}}^P (\mat{x}_{\epsilon})$ as the minimizer of running loss ${\cal L}(\mat{x}_{\epsilon}, \mat{u}, \mathcal{E}(\cdot))$ in Eq.~\eqref{eq:running loss}, 
\begin{equation}
    \label{eq:feedback control}
    \mat{u}_{\mat{x}}^P (\mat{x}_{\epsilon}) = \arg \min \limits_{\mat{u} \in \mathbb{R}^d} 
    \frac{1} {2} \cdot \lVert f'(\mat{x}) (\mat{x}_{\epsilon} + \mat{u}) \rVert_2^2 + \frac{c} {2} \cdot \lVert \mat{u} \rVert_2^2.
\end{equation}
Notice $(\mat{x}_{\epsilon} + \mat{u}_{\mat{x}}^P (\mat{x}_{\epsilon})) = \mat{P}_{\mat{x}} (\mat{x}_{\epsilon})$ when the regularization $c = 0$,
$\mat{u}_{\mat{x}}^P (\mat{x}_{\epsilon})$ admits an exact solution
\begin{equation}
    \label{eq:solution up}
    \mat{u}_{\mat{x}}^P (\mat{x}_{\epsilon}) = - (c \cdot \mat{I} + \mat{Q}_{\mat{x}})^{-1} \mat{Q}_{\mat{x}} \mat{x}_{\epsilon} 
    = - (c \cdot \mat{I} + f'(\mat{x})^T f'(\mat{x}))^{-1} f'(\mat{x})^T f'(\mat{x}) \mat{x}_{\epsilon}.
\end{equation}

In the nonlinear case,
let $\mathcal{M} \subset \mathbb{R}^d$ be an embedding manifold such that $\mathcal{M} = f^{-1} (\mat{0})$, for a submersion $f(\cdot)$ of class $\mathcal{C}^2$,
a constant $\sigma$ be a uniform upper bound on the Hessian of $f(\cdot)$, 
such that $\sup_{\mat{x} \in \mathbb{R}^d} \lVert f''(\mat{x}) \rVert_{\ast} \leq \sigma$.
For simplicity, we assume a normalized submersion $f(\cdot)$ to be where $f'(\mat{x})$ is a orthonormal basis for the orthogonal complement of tangent space at $\mat{x} \in \mathcal{M}$.
In this case,
we denote $\mat{u}^{\mathcal{M}} (\mat{x}_{\epsilon})$ as the minimizer of the running loss ${\cal L}(\mat{x}_{\epsilon}, \mat{u}, \mathcal{E}(\cdot))$ in Eq.~\eqref{eq:running loss},
\begin{equation}
    \label{eq:feedback control of manifold}
    \mat{u}^{\mathcal{M}} (\mat{x}_{\epsilon}) = \arg \min \limits_{\mat{u} \in \mathbb{R}^d} 
    \frac{1} {2} \cdot \lVert f(\mat{x}_{\epsilon} + \mat{u}) \rVert_2^2 + \frac{c} {2} \cdot \lVert \mat{u} \rVert_2^2.
\end{equation}
In general, when the submersion is not normalized, we can always normalize it by replacing $f(\mat{x})$ as $f'(\mat{x})^T (f'(\mat{x}) f'(\mat{x})^T)^{-1} f(\mat{x})$, where $f'(\mat{x})^T (f'(\mat{x}) f'(\mat{x})^T)^{-1}$ is a scaling factor.

\paragraph{Error bound for linear and nonlinear control solutions.}
For a $3$-dimensional tensor, e.g. the Hessian $f''(\mat{x})$, we define the $2$-norm of $f''(\mat{x})$ as
\begin{equation*}
    \lVert f''(\mat{x}) \rVert_{\ast} \vcentcolon = \sup_{\mat{z} \neq \mat{0}} \frac{\lVert f''(\mat{x})^{i,j,k} \mat{z}_j \mat{z}_k \rVert_2} {\lVert \mat{z} \rVert_2^2}.
\end{equation*}
The following proposition shows an error bound between $\mat{u}^{\mathcal{M}} (\mat{x}_{\epsilon})$ and $\mat{u}_{\mat{x}}^P (\mat{x}_{\epsilon})$.
\begin{restatable}{proposition}{manifoldtangentspace}
\label{prop: error between manifold and tangent space projection}
Consider a data point $\mat{x}_{\epsilon} = \mat{x} + \epsilon \cdot \mat{v}$, where $\mat{x} \in \mathcal{M}$, $\lVert \mat{v} \rVert_2 = 1$ and $\epsilon$ sufficiently small $ 0 \leq \epsilon \leq 1$.
The difference between the regularized manifold projection $\mat{u}^{\mathcal{M}} (\mat{x}_{\epsilon})$ and the regularized tangent space projection $\mat{u}_{\mat{x}}^P (\mat{x}_{\epsilon})$ is upper bounded as follows,
\begin{equation*}
    \lVert \mat{u}^{\mathcal{M}} (\mat{x}_{\epsilon}) - \mat{u}_{\mat{x}}^P (\mat{x}_{\epsilon}) \rVert_2 \leq 4 \epsilon^2 \sigma (1 + 2 \sigma).
\end{equation*}
\end{restatable}

\begin{proof}
Recall the definition of regularized manifold projection in Eq.~\eqref{eq:feedback control of manifold},
the optimal solution $\mat{u}^{\mathcal{M}} (\mat{x}_{\epsilon})$ admits a exact solution by setting the gradient of Eq.~\eqref{eq:feedback control of manifold} to $\mat{0}$,
\begin{equation}
    \label{eq:gradient running loss}
    \nabla_{\mat{u}} \bigg( \frac{1} {2} \cdot \lVert f(\mat{x}_{\epsilon} + \mat{u}) \rVert_2^2 + \frac{c} {2} \cdot \lVert \mat{u} \rVert_2^2 \bigg) = \bigg(f'(\mat{x} + \epsilon \mat{v} + \mat{u}) \bigg)^T \bigg(f(\mat{x} + \epsilon \mat{v} + \mat{u})  \bigg) + c \cdot \mat{u}.
\end{equation}
The control $\mat{u}$ is in the same order as the perturbation magnitude $\epsilon$,
we parametrize $\mat{u} = \epsilon \cdot \boldsymbol\mu$.
By applying Taylor series expansion centered at $\epsilon = 0$,
and $f(\mat{x}) = \mat{0}$ since $\mat{x} \in \mathcal{M}$,
\begin{align*}
    & \;\;\;\;\; \bigg(f'(\mat{x} + \epsilon \mat{v} + \epsilon \boldsymbol\mu) \bigg)^T \bigg(f(\mat{x} + \epsilon \mat{v} + \epsilon \boldsymbol\mu)  \bigg) + c \cdot \epsilon \cdot \boldsymbol\mu \nonumber \\
    & = \sq{\bigg(f'(\mat{x}) + \epsilon \big(f'' (\mat{x}^{\boldsymbol\mu})^{i,j,k} (\mat{v} + \boldsymbol\mu)_k \big) \bigg)^T \bigg( \epsilon f'(\mat{x}) (\mat{v} + \boldsymbol\mu) + \epsilon^2 \big(f''(\mat{x}^{\boldsymbol\mu})^{i,j,k} (\mat{v} + \boldsymbol\mu)_j (\mat{v} + \boldsymbol\mu)_k \big) \bigg)
    + c \cdot \epsilon \cdot \boldsymbol\mu},
\end{align*}
since $\boldsymbol\mu$ is a variable dependent on $\mat{u}$, the Hessian of $f(\cdot)$ is a function that depends on $\boldsymbol\mu$.
There exits a $\mat{x}^{\boldsymbol\mu}$ satisfying the following,
\begin{equation*}
    f(\mat{x} + \epsilon \mat{v} + \epsilon \boldsymbol\mu)
    =
    f(\mat{x}) + \epsilon f'(\mat{x}) (\mat{v} + \boldsymbol\mu) + 
    f''(\mat{x}^{\boldsymbol\mu})^{i,j,k} (\mat{v} + \boldsymbol\mu)_j (\mat{v} + \boldsymbol\mu)_k.
\end{equation*}
Furthermore, recall that $\mat{u} = \epsilon \cdot \boldsymbol\mu$,
\begin{align*}
    & \sq{\bigg(f'(\mat{x}) + \big( f''(\mat{x}^{\boldsymbol\mu})^{i,j,k} (\epsilon \mat{v} + \mat{u})_k \big) \bigg)^T \bigg(f'(\mat{x}) (\epsilon \mat{v} + \mat{u}) + \big( f''(\mat{x}^{\boldsymbol\mu})^{i,j,k} (\epsilon \mat{v} + \mat{u})_j (\epsilon \mat{v} + \mat{u})_k \big) \bigg) + c\cdot \mat{u},} \nonumber \\
    & = 
    f'(\mat{x})^T f'(\mat{x}) (\epsilon \mat{v} + \mat{u}) + c \cdot \mat{u} + f'(\mat{x})^T \big( f'' (\mat{x}^{\boldsymbol\mu})^{i,j,k} (\epsilon \mat{v} + \mat{u})_j (\epsilon \mat{v} + \mat{u})_k \big) \\
    & \;\;\;\; \sq{+ \bigg(\big( f''(\mat{x}^{\boldsymbol\mu})^{i,j,k} (\epsilon \mat{v} + \mat{u})_k \big) \bigg)^T  \bigg(  f'(\mat{x})(\epsilon \mat{v} + \mat{u}) +
    \big( f'' (\mat{x}^{\boldsymbol\mu})^{i,j,k} (\epsilon \mat{v} + \mat{u})_j (\epsilon \mat{v} + \mat{u})_k \big) \bigg).}
\end{align*}
Setting the above to $\mat{0}$ results in an implicit solution for $\mat{u}^{\mathcal{M}} (\mat{x}_{\epsilon})$,
\begin{equation*}
    \mat{u}^{\mathcal{M}} (\mat{x}_{\epsilon}) = - \bigg( f'(\mat{x})^T f'(\mat{x}) + c \mat{I} \bigg)^{-1} \bigg(\epsilon f'(\mat{x})^T f'(\mat{x}) \mat{v} + \mat{E}_1 + \mat{E}_2 \bigg),
\end{equation*}
where
\begin{gather*}
    \mat{E}_1 = f'(\mat{x})^T \big( f'' (\mat{x}^{\boldsymbol\mu})^{i,j,k} (\epsilon \mat{v} + \mat{u}^{\mathcal{M}} (\mat{x}_{\epsilon}))_j (\epsilon \mat{v} + \mat{u}^{\mathcal{M}} (\mat{x}_{\epsilon}))_k \big), \\
    \mat{E}_2 = \sq{\Big( f''(\mat{x}^{\boldsymbol\mu})^{i,j,k} (\epsilon \mat{v} + \mat{u})_k \Big)^T  \Big( f'(\mat{x})(\epsilon \mat{v} + \mat{u}) +
    f'' (\mat{x}^{\boldsymbol\mu})^{i,j,k} (\epsilon \mat{v} + \mat{u})_j (\epsilon \mat{v} + \mat{u})_k \Big).}
\end{gather*}

Note that $\mat{u}^{\mathcal{M}} (\mat{x}_{\epsilon})$ is an implicit solution since $\mat{E}_1$ and $\mat{E}_2$ both depend on the solution $\mat{u}$.
Recall the definition of $\mat{u}_{\mat{x}}^P (\mat{x}_{\epsilon})$ in Eq.~\eqref{eq:solution up},
\begin{align*}
    \mat{u}_{\mat{x}}^P (\mat{x}_{\epsilon}) & = - (c \cdot \mat{I} + \mat{Q}_{\mat{x}})^{-1} \mat{Q}_{\mat{x}} \mat{x}_{\epsilon}, \\
    & = - (c \cdot \mat{I} + \mat{Q}_{\mat{x}})^{-1} \mat{Q}_{\mat{x}} (\mat{x} + \epsilon \cdot \mat{v}), \\
    & = - \epsilon \bigg(c \cdot \mat{I} +  f'(\mat{x})^T f'(\mat{x}) \bigg)^{-1} f'(\mat{x})^T f'(\mat{x}) \mat{v},
\end{align*}
the difference between $\mat{u}^{\mathcal{M}} (\mat{x}_{\epsilon})$ and $\mat{u}_{\mat{x}}^P (\mat{x}_{\epsilon})$,
\begin{align*}
    \lVert \mat{u}^{\mathcal{M}} (\mat{x}_{\epsilon}) - \mat{u}_{\mat{x}}^P (\mat{x}_{\epsilon}) \rVert_2
    \leq \lVert \big( f'(\mat{x})^T f'(\mat{x}) + c \cdot \mat{I} \big)^{-1} \rVert_2 \cdot \lVert \mat{E}_1 + \mat{E}_2 \rVert_2.
\end{align*}

Let us simplify the above inequality.
\begin{itemize}[leftmargin=*]
    \item For any non-negative $c$, 
    \begin{equation*}
        \lVert \big(f'(\mat{x})^T f'(\mat{x}) + c \cdot \mat{I} \big)^{-1} \rVert_2 = \lVert \big(f'(\mat{x})^T f'(\mat{x}) + c \cdot \mat{I} \big)^{-1} \rVert_2 \leq 1.
    \end{equation*}
    \item Recall the gradient of the running loss (Eq.~\eqref{eq:gradient running loss}),
    \begin{equation*}
        \bigg(f'(\mat{x} + \epsilon \mat{v} + \epsilon \boldsymbol\mu) \bigg)^T \bigg(f(\mat{x} + \epsilon \mat{v} + \epsilon \boldsymbol\mu)  \bigg)  + c \cdot \epsilon \cdot \boldsymbol\mu  = \bigg(f'(\mat{x} + \epsilon \mat{v} + \mat{u}) \bigg)^T
        \bigg( f'(\mat{p})(\epsilon \mat{v} + \mat{u})\bigg) + c \cdot \mat{u},
    \end{equation*}
    where $\mat{p} = \alpha \mat{x} + (1 - \alpha) (\mat{x} + \epsilon \mat{v} + \mat{u}^{\mathcal{M}})$ for $\alpha \in [0, 1]$
    such that
    \begin{equation*}
        f(\mat{x} + \epsilon \mat{v} + \epsilon \boldsymbol\mu) = f(\mat{x}) + \epsilon \cdot f'(\mat{p})(\epsilon \mat{v} + \epsilon \boldsymbol\mu).
    \end{equation*}
    Setting the gradient of running loss to $\mat{0}$ results in the optimal solution $\mat{u}^{\mathcal{M}} (\mat{x}_{\epsilon})$,
    \begin{equation*}
        \mat{u}^{\mathcal{M}} (\mat{x}_{\epsilon})
        = - \bigg(  \big(f'(\mat{x} + \epsilon \mat{v} + \mat{u}) \big)^T f'(\mat{p}) + c \mat{I} \bigg)^{-1} 
        \bigg( \big(f'(\mat{x} + \epsilon \mat{v} + \mat{u}) \big)^T f'(\mat{p}) \bigg) (\epsilon \mat{v}).
    \end{equation*}
    Since $f'(\cdot)$ contains orthonormal basis,
    the solution $\lVert \mat{u}^{\mathcal{M}} (\mat{x}_{\epsilon}) \rVert$ can be upper bounded by the follows,
    \begin{align}
        \label{eq:control less than epsilon}
        \lVert \mat{u}^{\mathcal{M}} (\mat{x}_{\epsilon}) \rVert & \leq
        \big\lVert \bigg(  \big(f'(\mat{x} + \epsilon \mat{v} + \mat{u}) \big)^T f'(\mat{p}) + c \mat{I} \bigg)^{-1} \big\rVert_2
        \cdot 
        \big\lVert \bigg( \big(f'(\mat{x} + \epsilon \mat{v} + \mat{u}) \big)^T f'(\mat{p}) \bigg) \big\rVert_2
        (\epsilon), \nonumber \\
        & \leq \big\lVert \big(f'(\mat{x} + \epsilon \mat{v} + \mat{u}) \big)^T f'(\mat{p}) \big\rVert_2^2 \cdot (\epsilon), \nonumber \\
        & \leq \lVert f'(\mat{x} + \epsilon \mat{v} + \mat{u})^T \rVert_2^2 \cdot \lVert f'(\mat{p}) \rVert_2^2 \cdot (\epsilon), \nonumber \\
        & \leq \epsilon.
    \end{align}
    \item 
    From above,
    \begin{equation*}
        \lVert \epsilon \mat{v} + \mat{u}^{\mathcal{M}} (\mat{x}_{\epsilon}) \rVert_2^2
        = \lVert \epsilon \mat{v} \rVert_2^2 + 2 \lVert \epsilon \mat{v}\rVert_2 \cdot \lVert\mat{u}^{\mathcal{M}} (\mat{x}_{\epsilon}) \rVert_2
        + \lVert \mat{u}^{\mathcal{M}} (\mat{x}_{\epsilon}) \rVert_2^2
        \leq 
        4 \epsilon^2,
    \end{equation*}
    \begin{equation*}
        \lVert \epsilon \mat{v} + \mat{u}^{\mathcal{M}} (\mat{x}_{\epsilon}) \rVert_2^3
        \leq 
        8 \epsilon^3.
    \end{equation*}
    \item Recall the $f'(\mat{x})$ is a orthnormal basis, $\lVert f'(\mat{x})\rVert_2 \leq 1$,
    the error terms can be bounded as follows,
    \begin{align*}
        \lVert \mat{E}_1 \rVert_2 & = \lVert f'(\mat{x})^T \big( f'' (\mat{x}^{\boldsymbol\mu})^{i,j,k} (\epsilon \mat{v} + \mat{u}^{\mathcal{M}} (\mat{x}_{\epsilon}))_j (\epsilon \mat{v} + \mat{u}^{\mathcal{M}} (\mat{x}_{\epsilon}))_k \big) \rVert_2, \\
        & \leq \lVert \epsilon \mat{v} + \mat{u}^{\mathcal{M}} (\mat{x}_{\epsilon}) \rVert_2^2 \cdot \lVert f'' (\mat{x}^{\boldsymbol\mu}) \rVert_{\ast} \cdot \lVert f'(\mat{x})^T \rVert_2, \\
        & \leq 4 \epsilon^2.
    \end{align*}
    \begin{align*}
        \lVert \mat{E}_2 \rVert_2 & = \sq{\Big \lVert \Big(f''(\mat{x}^{\boldsymbol\mu})^{i,j,k} (\epsilon \mat{v} + \mat{u})_k\Big)^T  \Big(f'(\mat{x})(\epsilon \mat{v} + \mat{u}) +
        f'' (\mat{x}^{\boldsymbol\mu})^{i,j,k} (\epsilon \mat{v} + \mat{u})_j (\epsilon \mat{v} + \mat{u})_k \Big) \Big\rVert_2} \\
        & \leq \sq{\lVert \epsilon \mat{v} + \mat{u}^{\mathcal{M}} (\mat{x}_{\epsilon}) \rVert_2^2 \cdot  \lVert f''(\mat{x}^{\boldsymbol\mu}) \rVert_{\ast} \cdot \lVert f'(\mat{x}) \rVert_2 + \lVert \epsilon \mat{v} + \mat{u}^{\mathcal{M}} (\mat{x}_{\epsilon}) \rVert_2^3 \cdot \lVert f'' (\mat{x}^{\boldsymbol\mu}) \rVert_{\ast}^2,} \\
        & \leq 4 \epsilon^2 \sigma + 8 \epsilon^3 \sigma^2.
    \end{align*}
\end{itemize}
Therefore,
for sufficiently small $\epsilon$, such that $\epsilon \leq 1$,
the difference
\begin{equation*}
    \lVert \mat{u}^{\mathcal{M}} (\mat{x}_{\epsilon}) - \mat{u}_{\mat{x}}^P (\mat{x}_{\epsilon}) \rVert_2
    \leq \lVert \mat{E}_1 \rVert_2 + \lVert \mat{E}_2 \rVert_2
    \leq 4 \epsilon^2 \sigma (1 + 2 \sigma).
\end{equation*}
\end{proof}

The above proposition shows that the error between solutions of running loss with tangent space and nonlinear manifold is of order $\mathcal{O}(\epsilon^2)$,
this result will serve to derive the error estimation in the nonlinear case.

\subsection{Analysis On Linearization Error}
\label{sec:Analysis On Linearization Error}
This section derives an $\mathcal{O} (\epsilon^2)$ error from linearizing the nonlinear system $F_t(\mat{x}_t)$ and nonlinear embedding function $\mathcal{E}_t(\mat{x}_t)$.
We represent the $t^{th}$ embedding manifold $\mathcal{M}_t = f_t^{-1} (\mat{0})$,
where $f_t(\cdot) : \mathbb{R}^d \rightarrow \mathbb{R}^{d - r}$ is a submersion of class $\mathcal{C}^2$.
Recall the definition on the 2-norm of a $3$-dimensional tensor,
\begin{equation*}
    \lVert f''(\mat{x}) \rVert_{\ast} \vcentcolon = \sup_{\mat{z} \neq \mat{0}} \frac{\lVert f''(\mat{x})^{i,j,k} \mat{z}_j \mat{z}_k \rVert_2} {\lVert \mat{z} \rVert_2^2},
\end{equation*}
we consider an uniform upper bound on the submersion $\sup_{\mat{x} \in \mathbb{R}^d} \lVert f_t''(\mat{x}) \rVert_{\ast} \leq \sigma_t$,
and an uniform upper bound on the nonlinear transformation $\sup_{\mat{x} \in \mathbb{R}^d} \lVert F_t''(\mat{x}) \rVert_{\ast} \leq \beta_t$.

\paragraph{Recall the definition of control in linear case.}
Recall Proposition \ref{prop: manifold tangent space projection},
the kernel of $f_t'(\mat{x}_t)$ is equivalent to $\mathcal{T}_{\mat{x}_t}\mathcal{M}_t$.
When the submersion $f_t(\cdot)$ is normalized where the columns of $f_t'(\mat{x}_t)$ consist of a orthonormal basis,
the orthogonal projection onto $\mathcal{T}_{\mat{x}_t}\mathcal{M}_t$ (Eq.~\eqref{eq:tangent space orthogonal projection}) is 
\begin{equation*}
    \mat{P}_{\mat{x}_t} \vcentcolon = \mat{I} - f_t'(\mat{x}_t)^T f_t'(\mat{x}_t),
\end{equation*}
and the orthogonal projection onto orthogonal complement of $\mathcal{T}_{\mat{x}_t}\mathcal{M}_t$ is $\mat{Q}_{\mat{x}_t} = \mat{I} - \mat{P}_{\mat{x}_t}$.
In this linear case, the running loss in Eq.~\eqref{eq:running loss} ${\cal L}(\mat{x}_t, \mat{u}_t, \mathcal{E}_t(\cdot))$ is defined as
\begin{equation*}
    {\mathcal{L}}(\mat{x}_{\epsilon}, \mat{u}_t, \mathcal{E}_t (\cdot)) = \frac{1} {2} \lVert f_t'(\mat{x}_t)(\mat{x}_{\epsilon} + \mat{u}_t) \rVert_2^2 + \frac{c} {2} \lVert \mat{u}_t \rVert_2^2.
\end{equation*}
Its optimal solution $\mat{u}_{\mat{x}_t}^P (\mat{x}_{\epsilon})$ (Eq.~\eqref{eq:solution up}) is
\begin{equation}
    \label{eq:definition of up}
    \mat{u}_{\mat{x}_t}^P (\mat{x}_{\epsilon}) = - (c \cdot \mat{I} + f_t'(\mat{x}_t)^T f_t'(\mat{x}_t))^{-1} f_t'(\mat{x}_t)^T f_t'(\mat{x}_t) \mat{x}_{\epsilon} = - \mat{K}_{\mat{x}_t} \mat{x}_{\epsilon},
\end{equation}
where the feedback gain matrix $\mat{K}_{\mat{x}_t} = (c \cdot \mat{I} + f_t'(\mat{x}_t)^T f_t'(\mat{x}_t))^{-1} f_t'(\mat{x}_t)^T f_t'(\mat{x}_t)$.

\paragraph{Definition of linearized system.}
For the nonlinear transformation $F_t(\cdot)$,
the optimal solution $\mat{u}^{\mathcal{M}_t} (\overline{\mat{x}}_{\epsilon, t})$ of running loss in Eq.~\eqref{eq:running loss} equipped with an embedding manifold $\mathcal{M}_t$ is defined in Eq.~\eqref{eq:feedback control of manifold}.
The controlled nonlinear dynamic is
\begin{equation*}
    \overline{\mat{x}}_{\epsilon, t+1} = F_t(\overline{\mat{x}}_{\epsilon, t} + \mat{u}^{\mathcal{M}_t} (\overline{\mat{x}}_{\epsilon, t})).
\end{equation*}
By definition in the running loss of Eq.~\eqref{eq:feedback control of manifold},
$\mat{u}^{\mathcal{M}_t} (\mat{x}_{t}) = \mat{0}$ when $\mat{x}_t \in \mathcal{M}_t$.
Therefore,
we denote a sequence $\{ \mat{x}_t \}_{t=0}^{T-1}$ as the unperturbed states such that 
\begin{equation*}
    \mat{x}_{t + 1} = F_t (\mat{x}_t), \;\;\;\; \mat{x}_t \in \mathcal{M}_t, \;\;\;\; \forall t=0,1,...,T-1.
\end{equation*}
Given the unperturbed sequence $\{ \mat{x}_t \}_{t=0}^{T-1}$, we denote $\{ \boldsymbol\theta_t \}_{t=0}^{T-1}$ as the Jacobians of $\{ F_t(\cdot) \}_{t=0}^{T-1}$ such that
\begin{equation*}
    \boldsymbol\theta_t = F_t'(\mat{x}_t), \;\;\;\; \forall t = 1, 2,...,T-1,
\end{equation*}
and $\{ \mathcal{T}_{\mat{x}_t}\mathcal{M}_t \}_{t=0}^{T-1}$ as the tangent spaces such that $\mathcal{T}_{\mat{x}_t}\mathcal{M}_t$ is the tangent space of $\mathcal{M}_t$ at $\mat{x}_t \in \mathcal{M}_t$.

When a perturbation $\mat{z}$ is applied on initial condition, $\mat{x}_{\epsilon, 0} = \mat{x}_0 + \mat{z}$,
the difference between the controlled system of perturbed initial condition and $\{ \mat{x}_t \}_{t=0}^{T-1}$ is
\begin{equation*}
    \overline{\mat{x}}_{\epsilon, t+1} - \mat{x}_{t+1} = F_t(\overline{\mat{x}}_{\epsilon, t} + \mat{u}^{\mathcal{M}_t} (\overline{\mat{x}}_{\epsilon, t})) - F_t(\mat{x}_t).
\end{equation*}
The linearization of the state difference is defined as follows,
\begin{align*}
    \overline{\mat{x}}_{\epsilon, t + 1} - \mat{x}_{t + 1} 
    & = F_t (\overline{\mat{x}}_{\epsilon, t} + \mat{u}^{\mathcal{M}_t}(\overline{\mat{x}}_{\epsilon, t}) ) - F_t (\mat{x}_t), \nonumber \\
    & = F_t(\mat{x}_t) + \boldsymbol\theta_t (\overline{\mat{x}}_{\epsilon, t} + \mat{u}^{\mathcal{M}_t} (\overline{\mat{x}}_{\epsilon, t}) - \mat{x}_t) 
     \\
    & \;\;\;\; + \frac{1} {2} F''_t(\mat{p})^{i,j,k} (\overline{\mat{x}}_{\epsilon, t} + \mat{u}^{\mathcal{M}_t} (\overline{\mat{x}}_{\epsilon, t}) - \mat{x}_t)_j
    (\overline{\mat{x}}_{\epsilon, t} + \mat{u}^{\mathcal{M}_t} (\overline{\mat{x}}_{\epsilon, t}) - \mat{x}_t)_k - F_t(\mat{x}_t), \nonumber \\
    & = \boldsymbol\theta_t (\overline{\mat{x}}_{\epsilon, t} + \mat{u}^{\mathcal{M}_t}(\overline{\mat{x}}_{\epsilon, t}) - \mat{u}_{\mat{x}_t}^P (\overline{\mat{x}}_{\epsilon, t}) + \mat{u}_{\mat{x}_t}^P (\overline{\mat{x}}_{\epsilon, t}) - \mat{x}_t )
    \\ 
    & \;\;\;\; + \frac{1} {2} F''_t(\mat{p})^{i,j,k} (\overline{\mat{x}}_{\epsilon, t} + \mat{u}^{\mathcal{M}_t}(\overline{\mat{x}}_{\epsilon, t}) - \mat{x}_t)_j
    (\overline{\mat{x}}_{\epsilon, t} + \mat{u}^{\mathcal{M}_t} (\overline{\mat{x}}_{\epsilon, t}) - \mat{x}_t)_k, \nonumber \\
    & = \boldsymbol\theta_t (\overline{\mat{x}}_{\epsilon, t} + \mat{u}_{\mat{x}_t}^P (\overline{\mat{x}}_{\epsilon, t}) - \mat{x}_t) + \boldsymbol\theta_t (\mat{u}^{\mathcal{M}_t} (\overline{\mat{x}}_{\epsilon, t}) - \mat{u}_{\mat{x}_t}^P (\overline{\mat{x}}_{\epsilon, t}))
    \\
    & \;\;\;\; + \frac{1} {2} F''_t(\mat{p})^{i,j,k} (\overline{\mat{x}}_{\epsilon, t} + \mat{u}^{\mathcal{M}_t} (\overline{\mat{x}}_{\epsilon, t}) - \mat{x}_t)_j
    (\overline{\mat{x}}_{\epsilon, t} + \mat{u}^{\mathcal{M}_t} (\overline{\mat{x}}_{\epsilon, t}) - \mat{x}_t)_k,
\end{align*}
where $\mat{p} = \alpha \mat{x}_t + (1 - \alpha) (\overline{\mat{x}}_{\epsilon, t} + \mat{u}^{\mathcal{M}_t})$ for $\alpha \in [0, 1]$,
$F''_t(\mat{p})$ is a third-order tensor such that
\begin{align*}
    F_t(\overline{\mat{x}}_{\epsilon, t} + \mat{u}^{\mathcal{M}_t} (\overline{\mat{x}}_{\epsilon, t})) & = F_t(\mat{x}_t) + \boldsymbol\theta_t(\overline{\mat{x}}_{\epsilon, t} + \mat{u}^{\mathcal{M}_t} (\overline{\mat{x}}_{\epsilon, t}) - \mat{x}_t) \\
    & \;\;\;\; + \frac{1}{2} F''_t(\mat{p})^{i,j,k} (\overline{\mat{x}}_{\epsilon, t} + \mat{u}^{\mathcal{M}_t} (\overline{\mat{x}}_{\epsilon, t}) - \mat{x}_t)_j (\overline{\mat{x}}_{\epsilon, t} + \mat{u}^{\mathcal{M}_t} (\overline{\mat{x}}_{\epsilon, t}) - \mat{x}_t)_k,
\end{align*}
such a $\mat{p}$ always exists according to the mean-field theorem.
Recall the definition of $\mat{u}_{\mat{x}_t}^P (\overline{\mat{x}}_{\epsilon, t})$ in Eq.~\eqref{eq:definition of up},
$\boldsymbol\theta_t (\overline{\mat{x}}_{\epsilon, t} + \mat{u}_{\mat{x}_t}^P (\overline{\mat{x}}_{\epsilon, t}) - \mat{x}_t) = \boldsymbol\theta_t (\mat{I} - \mat{K}_{\mat{x}_t}) (\overline{\mat{x}}_{\epsilon, t} - \mat{x}_t)$,
\begin{align}
    \label{eq:linearized nonlinear system}
    \overline{\mat{x}}_{\epsilon, t + 1} - \mat{x}_{t + 1} & = \boldsymbol\theta_t (\mat{I} - \mat{K}_{\mat{x}_t}) (\overline{\mat{x}}_{\epsilon, t} - \mat{x}_t) + \boldsymbol\theta_t (\mat{u}^{\mathcal{M}_t} (\overline{\mat{x}}_{\epsilon, t}) - \mat{u}_{\mat{x}_t}^P (\overline{\mat{x}}_{\epsilon, t})) \nonumber
    \\
    & \;\;\;\; + \frac{1} {2} F''_t(\mat{p})^{i,j,k} (\overline{\mat{x}}_{\epsilon, t} + \mat{u}^{\mathcal{M}_t}(\overline{\mat{x}}_{\epsilon, t}) - \mat{x}_t)_j
    (\overline{\mat{x}}_{\epsilon, t} + \mat{u}^{\mathcal{M}_t}(\overline{\mat{x}}_{\epsilon, t}) - \mat{x}_t)_k.
\end{align}

\paragraph{Definition of linearization error.}
Given a perturbation $\mat{z}$,
we define the propagation of perturbation via the linearized system as $\boldsymbol\theta_{t-1} (\mat{I} - \mat{K}_{\mat{x}_{t-1}}) \cdots \boldsymbol\theta_0 (\mat{I} - \mat{K}_{\mat{x}_0}) \mat{z}$.
The linearization error is defined as follows,
\begin{equation*}
    e_t \vcentcolon =
    \lVert (\overline{\mat{x}}_{\epsilon, t} - \mat{x}_{t}) - \boldsymbol\theta_{t-1} (\mat{I} - \mat{K}_{\mat{x}_{t-1}}) \boldsymbol\theta_{t-2} (\mat{I} - \mat{K}_{\mat{x}_{t-2}}) \cdots, \boldsymbol\theta_0 (\mat{I} - \mat{K}_{\mat{x}_0}) \mat{z} \rVert_2.
\end{equation*}

The following proposition formulates a difference inequality for $e_t$.
\begin{proposition}
\label{prop: Difference inequality}
For $t \geq 1$,
\begin{gather*}
    e_{t + 1} \leq \lVert \boldsymbol\theta_t \rVert_2 e_t + (k_t \lVert \boldsymbol\theta_t \rVert_2 + 2 \beta_t) e_t^2 + (k_t \lVert \boldsymbol\theta_t \rVert_2 + 2 \beta_t) \cdot \delta_{\mat{x}_t} \cdot \epsilon^2, \nonumber \\
    e_1 \leq (k_{\mat{x}_0} \lVert \boldsymbol\theta_0 \rVert_2 + 2 \beta_0) \cdot \delta_{\mat{x}_0} \cdot \epsilon^2,
\end{gather*}
where
\begin{gather*}
    k_t = 4 \sigma_t (1 + 2 \sigma_t), \\
    \delta_{\mat{x}_t} = \lVert \boldsymbol\theta_{t-1} \cdots \boldsymbol\theta_{0} \rVert_2^2 \cdot
    \bigg( \alpha^{2t} \lVert \mat{v}^{\perp} \rVert_2^2 + \lVert \mat{v}^{\parallel} \rVert_2^2 + \gamma_t \lVert \mat{v} \rVert_2^2 \big( \gamma_t \alpha^2 (1 - \alpha^{t-1})^2 + 2 (\alpha - \alpha^t) \big) \bigg),~ t \geq 1, \\
    \delta_{\mat{x}_0} = 1,
\end{gather*}
$\alpha = \frac{c} {1 + c}$ for a control regularization $c$.
$\gamma_t \vcentcolon = \max \limits_{s \leq t} \big(1 +  \kappa (\overline{\boldsymbol\theta}_s)^2 \big) \lVert \mat{I} - \overline{\boldsymbol\theta}_s^T \overline{\boldsymbol\theta}_s \rVert_2$,
\begin{itemize}
    \item $\overline{\boldsymbol\theta}_t \vcentcolon = \boldsymbol\theta_{t-1} \cdots \boldsymbol\theta_0, ~ t \geq 1$,
    \item $\overline{\boldsymbol\theta}_0 \vcentcolon = \mat{I}, ~ t = 0$.
\end{itemize}
\end{proposition}

\begin{proof}
we subtract both sides of Eq.~\eqref{eq:linearized nonlinear system} by $\boldsymbol\theta_t (\mat{I} - \mat{K}_{\mat{x}_t}) \cdots \boldsymbol\theta_0 (\mat{I} - \mat{K}_{\mat{x}_0}) \mat{z}$, and recall the definition of linearization error $e_t$,
\begin{equation*}
    e_{t+1} \leq \lVert \boldsymbol\theta_t (\mat{I} - \mat{K}_{\mat{x}_t}) \rVert_2 \cdot e_t + \lVert \boldsymbol\theta_t \rVert_2 \cdot \lVert \mat{u}^{\mathcal{M}_t} (\overline{\mat{x}}_{\epsilon, t}) - \mat{u}_{\mat{x}_t}^P (\overline{\mat{x}}_{\epsilon, t}) \rVert_2 + \frac{1} {2} \lVert F''_t(\mat{p}) \rVert_{\ast} \cdot \lVert \overline{\mat{x}}_{\epsilon, t} + \mat{u}^{\mathcal{M}_t} (\overline{\mat{x}}_{\epsilon, t}) - \mat{x}_t \rVert_2^2.
\end{equation*}
Let us simplify the above inequality.
\begin{itemize}[leftmargin=*]
    \item The orthogonal projection admits $\lVert \mat{I} - \mat{K}_{\mat{x}_t} \rVert_2 \leq 1$.
    \item Recall Proposition \ref{prop: error between manifold and tangent space projection},
    \begin{equation*}
        \lVert \mat{u}^{\mathcal{M}_t} (\overline{\mat{x}}_{\epsilon, t}) - \mat{u}_{\mat{x}_t}^P (\overline{\mat{x}}_{\epsilon, t}) \rVert_2
        \leq 4 \sigma_t (1 + 2 \sigma_t) \cdot \lVert \overline{\mat{x}}_{\epsilon, t} - \mat{x}_t \rVert_2^2,
    \end{equation*}
    where $\sigma_t$ is the uniform upper bound on $\lVert f_t''(\mat{x}) \rVert_{\ast}$.
    We denote
    \begin{gather*}
        k_t = 4 \sigma_t (1 + 2 \sigma_t), \\
        \lVert \mat{u}^{\mathcal{M}_t} (\overline{\mat{x}}_{\epsilon, t}) - \mat{u}_{\mat{x}_t}^P (\overline{\mat{x}}_{\epsilon, t}) \rVert_2
        \leq 
        k_t \cdot \lVert \overline{\mat{x}}_{\epsilon, t} - \mat{x}_t \rVert_2^2.
    \end{gather*}
    \item $F_t(\cdot)$ admits an uniform upper bound $\beta_t$ such that
    $\sup_{\mat{x} \in \mathbb{R}^d} \lVert F''_t(\mat{x}) \rVert_{\ast} \leq \beta_t$.
    \item Recall the inequality in Eq.~\eqref{eq:control less than epsilon}, $\lVert \mat{u}^{\mathcal{M}_t}(\overline{\mat{x}}_{\epsilon, t}) \rVert_2 \leq \lVert \overline{\mat{x}}_{\epsilon, t} - \mat{x}_t \rVert_2$,
    \begin{align*}
        \lVert \overline{\mat{x}}_{\epsilon, t} + \mat{u}^{\mathcal{M}_t} (\overline{\mat{x}}_{\epsilon, t}) - \mat{x}_t \rVert_2^2 & \leq 2 \cdot \lVert \overline{\mat{x}}_{\epsilon, t} - \mat{x}_t \rVert_2^2 + 2 \cdot \lVert \mat{u}^{\mathcal{M}_t} (\overline{\mat{x}}_{\epsilon, t}) \rVert_2^2, \nonumber \\
        & \leq 4 \cdot \lVert \overline{\mat{x}}_{\epsilon, t} - \mat{x}_t \rVert_2^2.
    \end{align*}.
\end{itemize}
Therefore,
\begin{equation*}
    e_{t+1} \leq \lVert \boldsymbol\theta_t \rVert_2 e_t + (k_t \lVert \boldsymbol\theta_t \rVert_2 + 2 \beta_t) \cdot \lVert \overline{\mat{x}}_{\epsilon, t} - \mat{x}_t \rVert_2^2.
\end{equation*}
Furthermore,
\begin{align*}
    & \;\;\;\;\; \lVert \overline{\mat{x}}_{\epsilon, t} - \mat{x}_t \rVert_2^2 \\
    & =
    \lVert \overline{\mat{x}}_{\epsilon, t} - \mat{x}_t - \boldsymbol\theta_{t - 1} (\mat{I} - \mat{K}_{\mat{x}_{t-1}}) \cdots \boldsymbol\theta_0 (\mat{I} - \mat{K}_{\mat{x}_0}) \mat{z} + \boldsymbol\theta_{t - 1} (\mat{I} - \mat{K}_{\mat{x}_{t-1}}) \cdots \boldsymbol\theta_0 (\mat{I} - \mat{K}_{\mat{x}_0}) \mat{z} \rVert_2^2, \\
    & \leq e_t^2 + \lVert \boldsymbol\theta_{t - 1} (\mat{I} - \mat{K}_{\mat{x}_{t-1}}) \cdots \boldsymbol\theta_0 (\mat{I} - \mat{K}_{\mat{x}_0}) \mat{z} \rVert_2^2.
\end{align*}
Then, the linearization error can be bounded as follows,
\begin{equation*}
    e_{t+1} \leq \lVert \boldsymbol\theta_t \rVert_2 e_t + (k_t \lVert \boldsymbol\theta_t \rVert_2 + 2 \beta_t) e_t^2 + (k_t \lVert \boldsymbol\theta_t \rVert_2 + 2 \beta_t) \cdot \lVert \boldsymbol\theta_{t - 1} (\mat{I} - \mat{K}_{\mat{x}_{t-1}}) \cdots \boldsymbol\theta_0 (\mat{I} - \mat{K}_{\mat{x}_0}) \mat{z} \rVert_2^2.
\end{equation*}

We can express the initial perturbation as $\mat{z} = \epsilon \mat{v}$, where $\epsilon$ is perturbation magnitude and
$\mat{v}$ is a unit vector that represents the perturbation direction.
The perturbation direction $\mat{v}$ admits a direct sum such that $\mat{v} = \mat{v}^{\parallel} \oplus \mat{v}^{\perp}$,
where $\mat{v}^{\parallel} \in \mathcal{T}_{\mat{x}_0}\mathcal{M}_0$
and $\mat{v}^{\perp}$ lies in the orthogonal complement of $\mathcal{T}_{\mat{x}_0}\mathcal{M}_0$.

Recall Theorem \ref{theorem: error estimation of linear system},
\begin{align*}
    & \;\;\;\;\; \lVert \boldsymbol\theta_{t-1} (\mat{I} - \mat{K}_{\mat{x}_{t-1}}) \boldsymbol\theta_{t-2} (\mat{I} - \mat{K}_{\mat{x}_{t-2}}) \cdots \boldsymbol\theta_0 (\mat{I} - \mat{K}_{\mat{x}_0}) \mat{z} \rVert_2^2, \nonumber \\
    & \leq \lVert \boldsymbol\theta_{t-1} \cdots \boldsymbol\theta_{0} \rVert_2^2 \cdot \bigg( \alpha^{2t} \lVert \mat{z}^{\perp} \rVert_2^2 + \lVert \mat{z}^{\parallel} \rVert_2^2 + \gamma_t \lVert \mat{z} \rVert_2^2 \big( \gamma_t \alpha^2 (1 - \alpha^{t-1})^2 + 2 (\alpha - \alpha^t) \big) \bigg), \nonumber \\
    & \leq \lVert \boldsymbol\theta_{t-1} \cdots \boldsymbol\theta_{0} \rVert_2^2 \cdot
    \bigg( \alpha^{2t} \lVert \mat{v}^{\perp} \rVert_2^2 + \lVert \mat{v}^{\parallel} \rVert_2^2 + \gamma_t \lVert \mat{v} \rVert_2^2 \big( \gamma_t \alpha^2 (1 - \alpha^{t-1})^2 + 2 (\alpha - \alpha^t) \big) \bigg) \epsilon^2,
\end{align*}
where $\alpha = \frac{c} {1 + c}$ for a control regularization $c$.
$\gamma_t \vcentcolon = \max \limits_{s \leq t} \big(1 +  \kappa (\overline{\boldsymbol\theta}_s)^2 \big) \lVert \mat{I} - \overline{\boldsymbol\theta}_s^T \overline{\boldsymbol\theta}_s \rVert_2$,
\begin{itemize}
    \item $\overline{\boldsymbol\theta}_t \vcentcolon = \boldsymbol\theta_{t-1} \cdots \boldsymbol\theta_0, ~ t \geq 1$,
    \item $\overline{\boldsymbol\theta}_0 \vcentcolon = \mat{I}, ~ t = 0$.
\end{itemize}

Let 
$\delta_{\mat{x}_t} = \lVert \boldsymbol\theta_{t-1} \cdots \boldsymbol\theta_{0} \rVert_2^2 \cdot
    \bigg( \alpha^{2t} \lVert \mat{v}^{\perp} \rVert_2^2 + \lVert \mat{v}^{\parallel} \rVert_2^2 + \gamma_t \lVert \mat{v} \rVert_2^2 \big( \gamma_t \alpha^2 (1 - \alpha^{t-1})^2 + 2 (\alpha - \alpha^t) \big) \bigg)$ for $t \geq 1$,
and
$\delta_{\mat{x}_0} = 1$,
the linearization error $e_{t+1}$ can be upper bounded by
\begin{equation*}
    e_{t+1} \leq \lVert \boldsymbol\theta_t \rVert_2 e_t + (k_t \lVert \boldsymbol\theta_t \rVert_2 + 2 \beta_t) e_t^2 + (k_t \lVert \boldsymbol\theta_t \rVert_2 + 2 \beta_t) \cdot \delta_{\mat{x}_t} \cdot \epsilon^2.
\end{equation*}

Since $e_t$ is defined for $t \geq 1$,
the following derives a upper bound on $e_1$.
When $t = 1$,
recall the initial perturbation $\overline{\mat{x}}_{\epsilon, 0} - \mat{x}_0 = \mat{z}$,
\begin{align*}
    & \overline{\mat{x}}_{\epsilon, 1} - \mat{x}_{1} \\
    & = F_0 (\overline{\mat{x}}_{\epsilon, 0} + \mat{u}_{0}^{\mathcal{M}} (\overline{\mat{x}}_{\epsilon, 0}) ) - F_0 (\mat{x}_0), \nonumber \\
    & = \boldsymbol\theta_0 (\overline{\mat{x}}_{\epsilon, 0} + \mat{u}_{\mat{x}_0}^{\mathcal{M}} (\overline{\mat{x}}_{\epsilon, 0}) - \mat{x}_0) 
    + \frac{1} {2} F_0''(\mat{p})^{i,j,k} (\overline{\mat{x}}_{\epsilon, 0} + \mat{u}_{\mat{x}_0}^{\mathcal{M}} (\overline{\mat{x}}_{\epsilon, 0}) - \mat{x}_0)_j
    (\overline{\mat{x}}_{\epsilon, 0} + \mat{u}_0^{\mathcal{M}} - \mat{x}_0)_k, \nonumber \\
    & = \sq{\boldsymbol\theta_0 (\mat{z} + \mat{u}_{\mat{x}_0}^{\mathcal{M}} (\overline{\mat{x}}_{\epsilon, 0}))
    + \frac{1} {2} F_0''(\mat{p})^{i,j,k} (\mat{z} + \mat{u}_0^{\mathcal{M}}(\overline{\mat{x}}_{\epsilon, 0}))_j
    (\mat{z} + \mat{u}_0^{\mathcal{M}}(\overline{\mat{x}}_{\epsilon, 0}))_k,} \nonumber \\
    & \sq{= \boldsymbol\theta_0 (\mat{z} + \mat{u}_0^{\mathcal{M}}(\overline{\mat{x}}_{\epsilon, 0}) - \mat{u}_0^P(\overline{\mat{x}}_{\epsilon, 0}) + \mat{u}_0^P(\overline{\mat{x}}_{\epsilon, 0}))
    + \frac{1} {2} F_0''(\mat{p})^{i,j,k} (\mat{z} + \mat{u}_0^{\mathcal{M}}(\overline{\mat{x}}_{\epsilon, 0}))_j
    (\mat{z} + \mat{u}_0^{\mathcal{M}}(\overline{\mat{x}}_{\epsilon, 0}))_k,} \\
    & \sq{= \boldsymbol\theta_0 (\mat{I} - \mat{K}_{\mat{x}_0}) \mat{z} + \boldsymbol\theta_0 (\mat{u}_0^{\mathcal{M}}(\overline{\mat{x}}_{\epsilon, 0}) - \mat{u}_0^P(\overline{\mat{x}}_{\epsilon, 0}))
    + \frac{1} {2} F_0''(\mat{p})^{i,j,k} (\mat{z} + \mat{u}_0^{\mathcal{M}}(\overline{\mat{x}}_{\epsilon, 0}))_j
    (\mat{z} + \mat{u}_0^{\mathcal{M}}(\overline{\mat{x}}_{\epsilon, 0}))_k.}
\end{align*}
By following the same procedure as the derivation of $e_{t+1}$,
\begin{align*}
    e_1 \leq (k_{\mat{x}_0} \lVert \boldsymbol\theta_0 \rVert_2 + 2 \beta_0) \cdot \delta_{\mat{x}_0} \cdot \epsilon^2.
\end{align*}
\end{proof}

The following proposition solves the difference inequality of linearization error.
\begin{restatable}{proposition}{linearizationerror}
\label{proposition:linearization error}
If the perturbation satisfies
\begin{equation*}
    \epsilon^2 \leq \frac{1} {\bigg( \sum_{i=0}^{T-1} \delta_{\mat{x}_i} (k_{\mat{x}_i} \lVert \boldsymbol\theta_i \rVert_2 + 2 \beta_i) \prod_{j = i+1}^{T-1} (\lVert \boldsymbol\theta_j \rVert_2 + k_{\mat{x}_j} \lVert \boldsymbol\theta_j \rVert_2 + 2 \beta_j) \bigg)}.
\end{equation*}
for $t \leq T$, the linearization error can be upper bounded by
\begin{equation*}
    e_t \leq \bigg( \sum_{i=0}^{t-1} \delta_{\mat{x}_i} (k_{\mat{x}_i} \lVert \boldsymbol\theta_i \rVert_2 + 2 \beta_i) \prod_{j = i+1}^{t-1} (\lVert \boldsymbol\theta_j \rVert_2 + k_{\mat{x}_j} \lVert \boldsymbol\theta_j \rVert_2 + 2 \beta_j) \bigg) \epsilon^2.
\end{equation*}
\end{restatable}

\begin{proof}
We prove it by induction on $t$ up to some $T$, such that $t \leq T$.
We restrict the magnitude of initial perturbation $\lVert \mat{z} \rVert_2^2 \leq \epsilon_T$ for some constant $\epsilon_T$,
such that the error $e_t \leq 1$ for all $t \leq T$.
The expression of $\epsilon_T$ is derived later.

When $t = 1$,
\begin{align*}
    e_1 \leq (k_{\mat{x}_0} \lVert \boldsymbol\theta_0 \rVert_2 + 2 \beta_0) \cdot \delta_{\mat{x}_0} \cdot \epsilon^2,
\end{align*}
which agrees with Proposition \ref{prop: Difference inequality}.

Suppose that it is true for some $t \leq T-1$,
such that
\begin{equation*}
    e_t \leq \bigg( \sum_{i=0}^{t-1} \delta_{\mat{x}_i} (k_{\mat{x}_i} \lVert \boldsymbol\theta_i \rVert_2 + 2 \beta_i) \prod_{j = i+1}^{t-1} (\lVert \boldsymbol\theta_j \rVert_2 + k_{\mat{x}_j} \lVert \boldsymbol\theta_j \rVert_2 + 2 \beta_j) \bigg) \epsilon^2.
\end{equation*}
Then at $t + 1$, recall Proposition \ref{prop: Difference inequality}, given that $e_t \leq 1$ for all $t \leq T$,
\begin{align*}
    e_{t+1} & \leq \lVert \boldsymbol\theta_t \rVert_2 e_t + (k_t \lVert \boldsymbol\theta_t \rVert_2 + 2 \beta_t) e_t^2 + (k_t \lVert \boldsymbol\theta_t \rVert_2 + 2 \beta_t) \cdot \delta_{\mat{x}_t} \cdot \epsilon^2, \\
    & \leq (\lVert \boldsymbol\theta_t \rVert_2 + k_t \lVert \boldsymbol\theta_t \rVert_2 + 2 \beta_t) e_t + (k_t \lVert \boldsymbol\theta_t \rVert_2 + 2 \beta_t) \cdot \delta_{\mat{x}_t} \cdot \epsilon^2, \\
    & \leq (\lVert \boldsymbol\theta_t \rVert_2 + k_t \lVert \boldsymbol\theta_t \rVert_2 + 2 \beta_t) \bigg( \sum_{i=0}^{t-1} \delta_{\mat{x}_i} (k_{\mat{x}_i} \lVert \boldsymbol\theta_i \rVert_2 + 2 \beta_i) \prod_{j = i+1}^{t-1} (\lVert \boldsymbol\theta_j \rVert_2 + k_{\mat{x}_j} \lVert \boldsymbol\theta_j \rVert_2 + 2 \beta_j) \bigg) \epsilon^2
    \\
    & \;\;\;\;\; + (k_t \lVert \boldsymbol\theta_t \rVert_2 + 2 \beta_t) \cdot \delta_{\mat{x}_t} \cdot \epsilon^2, \\
    & = \bigg( \sum_{i=0}^t \delta_{\mat{x}_i} (k_{\mat{x}_i} \lVert \boldsymbol\theta_i \rVert_2 + 2 \beta_i) \prod_{j = i+1}^t (\lVert \boldsymbol\theta_j \rVert_2 + k_{\mat{x}_j} \lVert \boldsymbol\theta_j \rVert_2 + 2 \beta_j) \bigg) \epsilon^2.
\end{align*}

We have restricted the initial perturbation $\lVert \mat{z} \rVert_2^2 = \epsilon^2 \leq \epsilon_T$, for some constant $\epsilon_T$, such that $e_t \leq 1$, for all $t \leq T$.

For $t \leq T$,
\begin{align*}
    e_t & \leq e_T, \\
    & \leq \bigg( \sum_{i=0}^{T-1} \delta_{\mat{x}_i} (k_{\mat{x}_i} \lVert \boldsymbol\theta_i \rVert_2 + 2 \beta_i) \prod_{j = i+1}^{T-1} (\lVert \boldsymbol\theta_j \rVert_2 + k_{\mat{x}_j} \lVert \boldsymbol\theta_j \rVert_2 + 2 \beta_j) \bigg) \epsilon^2, \\
    & \leq \bigg( \sum_{i=0}^{T-1} \delta_{\mat{x}_i} (k_{\mat{x}_i} \lVert \boldsymbol\theta_i \rVert_2 + 2 \beta_i) \prod_{j = i+1}^{T-1} (\lVert \boldsymbol\theta_j \rVert_2 + k_{\mat{x}_j} \lVert \boldsymbol\theta_j \rVert_2 + 2 \beta_j) \bigg) \epsilon_T, \\
    & = 1,
\end{align*}

therefore,
\begin{equation*}
    \epsilon_T = \frac{1} {\bigg( \sum_{i=0}^{T-1} \delta_{\mat{x}_i} (k_{\mat{x}_i} \lVert \boldsymbol\theta_i \rVert_2 + 2 \beta_i) \prod_{j = i+1}^{T-1} (\lVert \boldsymbol\theta_j \rVert_2 + k_{\mat{x}_j} \lVert \boldsymbol\theta_j \rVert_2 + 2 \beta_j) \bigg)}.
\end{equation*}
\end{proof}

Proposition~\ref{proposition:linearization error} provides several intuitions.
\begin{itemize}[leftmargin=*]
    \item the linearization error is of $\mathcal{O}(\epsilon^2)$  when the data  perturbation is small, where $\epsilon$ is the magnitude of the data perturbation.
    \item the linearization error becomes smaller when the nonlinear transformation $F_t(\cdot)$ behaves more linearily ($\beta_t$ decreases), and the curvature of embedding manifold is smoother ($k_t$ decreases).
    Specifically, in the linear case, $\beta_t$ and $k_t$ become $0$, which results in no linearization error.
    \item the linearization becomes smaller when the initial perturbation lies in a lower dimensional manifold ($\delta_{\mat{x}_t}$ decreases).
\end{itemize}

\subsection{Error Estimation}
\label{sec:Error Estimation}
Now we reach the main theorem on the error estimation of $\lVert \overline{\mat{x}}_{\epsilon, t} - \mat{x}_t \rVert$.

\maintheoremNonlinear*
\begin{proof}
recall that $e_{t+1} = \lVert (\overline{\mat{x}}_{\epsilon, t+1} - \mat{x}_{t+1}) - \boldsymbol\theta_t (\mat{I} - \mat{K}_{\mat{x}_t}) \cdots \boldsymbol\theta_0 (\mat{I} - \mat{K}_{\mat{x}_0}) \mat{z} \rVert_2$,
\begin{align*}
    & \;\;\;\;\; \lVert \overline{\mat{x}}_{\epsilon, t + 1} - \mat{x}_{t + 1} \rVert_2 \\
    & = \lVert \overline{\mat{x}}_{\epsilon, t + 1} - \mat{x}_{t + 1} - 
    \boldsymbol\theta_t (\mat{I} - \mat{K}_{\mat{x}_t}) \cdots \boldsymbol\theta_0 (\mat{I} - \mat{K}_{\mat{x}_0}) \mat{z}  + \boldsymbol\theta_t (\mat{I}- \mat{K}_{\mat{x}_t}) \cdots \boldsymbol\theta_0 (\mat{I} - \mat{K}_{\mat{x}_0}) \mat{z} \rVert_2, \nonumber \\
    & \leq \lVert \boldsymbol\theta_t (\mat{I} - \mat{K}_{\mat{x}_t}) \cdots \boldsymbol\theta_0 (\mat{I} - \mat{K}_{\mat{x}_0}) \mat{z} \rVert_2
    + \lVert \overline{\mat{x}}_{\epsilon, t + 1} - \mat{x}_{t + 1} - 
    \boldsymbol\theta_t (\mat{I} - \mat{K}_{\mat{x}_t}) \cdots \boldsymbol\theta_0 (\mat{I} - \mat{K}_{\mat{x}_0}) \mat{z} \rVert_2, \\
    & = \lVert \boldsymbol\theta_t (\mat{I} - \mat{K}_{\mat{x}_t}) \cdots \boldsymbol\theta_0 (\mat{I} - \mat{K}_{\mat{x}_0}) \mat{z} \rVert_2 + e_{t+1}.
\end{align*}

Recall Theorem \ref{theorem: error estimation of linear system},
\begin{align*}
    & \;\;\;\;\; \lVert \boldsymbol\theta_t (\mat{I} - \mat{K}_{\mat{x}_t}) \cdots \boldsymbol\theta_0 (\mat{I} - \mat{K}_{\mat{x}_0}) \mat{z} \rVert_2 \nonumber \\
    & \leq \bigg( \lVert \overline{\boldsymbol\theta}_{t+1} \rVert_2^2 \cdot \bigg( \alpha^{2(t+1)} \lVert \mat{z}^{\perp} \rVert_2^2 + \lVert \mat{z}^{\parallel} \rVert_2^2 + \gamma_{t+1} \lVert \mat{z} \rVert_2^2 \big( \gamma_{t+1} \alpha^2 (1 - \alpha^{t})^2 + 2 (\alpha - \alpha^{t+1}) \big) \bigg) \bigg)^{\frac{1}{2}}, \nonumber \\
    & \leq \lVert \overline{\boldsymbol\theta}_{t+1} \rVert_2
    \cdot \bigg( \alpha^{t+1} \lVert \mat{z}^{\perp} \rVert_2 + \lVert \mat{z}^{\parallel} \rVert_2 
    + \lVert \mat{z} \rVert_2 \big( \gamma_{t+1} \alpha (1 - \alpha^{t}) + \sqrt{2 \gamma_{t+1} (\alpha - \alpha^{t+1})} \big) \bigg),
\end{align*}
where $\overline{\boldsymbol\theta}_{t+1} = \boldsymbol\theta_t \boldsymbol\theta_{t-1} \cdots \boldsymbol\theta_0$.

Recall Proposition \ref{proposition:linearization error} for the linearization error,
\begin{equation*}
    e_{t+1} \leq \bigg( \sum_{i=0}^t \delta_{\mat{x}_i} (k_{\mat{x}_i} \lVert \boldsymbol\theta_i \rVert + 2 \beta_i) \prod_{j = i+1}^t (\lVert \boldsymbol\theta_j \rVert_2 + k_{\mat{x}_j} \lVert \boldsymbol\theta_j \rVert_2 + 2 \beta_j) \bigg) \epsilon^2.
\end{equation*}
Therefore, for $t \geq 1$,
\begin{align*}
    \sq{\lVert \overline{\mat{x}}_{\epsilon, t + 1} - \mat{x}_{t + 1} \rVert_2}
    & \leq \sq{\lVert \overline{\boldsymbol\theta}_{t+1} \rVert_2 \bigg( \alpha^{t+1} \lVert \mat{z}^{\perp} \rVert_2 + \lVert \mat{z}^{\parallel} \rVert_2 
    + \lVert \mat{z} \rVert_2 \big( \gamma_{t+1} \alpha (1 - \alpha^{t}) + \sqrt{2 \gamma_{t+1} (\alpha - \alpha^{t+1})} \big) \bigg)} \\
    & \;\;\;\; + \bigg( \sum_{i=0}^t \delta_{\mat{x}_i} (k_{\mat{x}_i} \lVert \boldsymbol\theta_i \rVert_2 + 2 \beta_i) \prod_{j = i+1}^t (\lVert \boldsymbol\theta_j \rVert_2 + k_{\mat{x}_j} \lVert \boldsymbol\theta_j \rVert_2 + 2 \beta_j) \bigg) \epsilon^2.
\end{align*}
\end{proof}

\section{Additional Numerical Experiments}
\label{exp:additional numerical experiments}
\subsection{Linear Closed-loop Control}
Here, we consider the closed-loop control method in linear setting.
Specifically, the embedding manifolds are linear subspaces, and the embedding functions are orthogonal projections onto those linear subspaces.
We follow the same experimental setting as in Sec \ref{sec:numerical experiments} for baseline models, robustness evaluations and PMP parameters.
We perform principle component analysis on clean training data to obtain embedding subspaces and embedding functions.

\begin{table}[t]
\caption{CIFAR-10 accuracy measure: baseline model / linear control}
\begin{center}
\begin{tabular}{ p{1.9cm}|p{2.3cm}p{2.3cm}p{2.3cm}p{2.3cm} }
\hline
\multicolumn{5}{c}{$\ell_{\infty}:\epsilon=8 / 255$, $\ell_2:\epsilon=0.5$, $\ell_1:\epsilon=12$}\\
\hline
\multicolumn{5}{c}{Standard models}\\
\hline
& None & AA ($\ell_{\infty}$) & AA ($\ell_2$) & AA ($\ell_1$) \\
\hline
RN-18 & \textbf{94.71} / 87.04 & 0. / \textbf{59.98} & 0. / \textbf{73.01} & 0. / \textbf{73.04} \\
\hline
RN-34 & \textbf{94.91} / 87.13 & 0. / \textbf{62.64} & 0. / \textbf{75.28} & 0. / \textbf{74.89} \\
\hline
RN-50 & \textbf{95.08} / 87.83 & 0. / \textbf{61.63} & 0. / \textbf{75.37} & 0. / \textbf{75.39} \\
\hline
WRN-28-8 & \textbf{95.41} / 87.72 & 0. / \textbf{68.11} & 0. / \textbf{77.97} & 0. / \textbf{78.33} \\
\hline
WRN-34-8 & \textbf{94.05} / 88.06 & 0. / \textbf{50.1} & 0. / \textbf{67.16} & 0. / \textbf{66.64} \\
\hline
\hline
\multicolumn{5}{c}{Robust models}\\
\hline
RN-$18$ & \textbf{82.39} / 81.0 & 48.72 / \textbf{53.06} & 58.8 / \textbf{69.94} & 9.86 / \textbf{50.41} \\
\hline
RN-$34$ & \textbf{84.45} / 83.0 & 49.31 / \textbf{52.75} & 57.27 / \textbf{70.51} & 7.21 / \textbf{51.44} \\
\hline
RN-$50$ & \textbf{83.99} / 82.99 & 48.68 / \textbf{52.23} & 57.25 / \textbf{70.31} & 6.83 / \textbf{51.26} \\
\hline
WRN-$28$-$8$ & \textbf{85.09} / 84.04 & 48.13 / \textbf{51.09} & 54.38 / \textbf{70.0} & 5.38 / \textbf{54.56} \\
\hline
WRN-$34$-$8$ & \textbf{84.95} / 83.7 & 48.47 / \textbf{51.4} & 54.36 / \textbf{70.31} & 4.67 / \textbf{55.02} \\
\hline
\end{tabular}
\end{center}
\label{cifar10 oblivious control linear}
\end{table}


As shown in Table \ref{cifar10 oblivious control linear},
on CIFAR-10 dataset, applying linear control on both standard trained and robustly trained baseline models can consistently improve the robustness against various perturbations.
Notably, the controlled models have more than $60 \%$, $80 \%$ and $70 \%$ accuracies against autoattack measured by $\ell_{\infty}, \ell_2$ and $\ell_1$ norms respectively.


\subsection{Robustness Improvement Under White-box Setting}
Here, we test our method in a fully white-box setting, where an attacker has full access to both pre-trained models and our control method.
We summarize our experimental settings below.
\begin{itemize}[leftmargin=*]
    \item {\bf Baseline models.}
    we use Pre-activated ResNet-$18$ (\textbf{RN-18}), -$34$ (\textbf{RN-34}), -$50$ (\textbf{RN-50}) as the testing benchmarks.
    All baseline models are trained with TRADES \citep{zhang2019theoretically}.
    \item {\bf Robustness evaluations.}
    In the white-box setting,
    the objective function of an attack algorithm is defined as follows,
    \begin{equation*}
    \tilde{\mat{x}}_{\epsilon} = \arg\max \limits_{\tilde{\mat{x}}_0 \sim \mathcal{B} (\mat{x}_0, \epsilon)} {\rm CE}(\tilde{\mat{x}}_T, \mat{q}), \quad {\rm s.t.}\; 
    \tilde{\mat{x}}_{t+1} = F_t(\tilde{\mat{x}}_{t} + \mat{u}_t^{\mathcal{M}}).
    \end{equation*}
    \item {\bf Embedding functions.}
    We use a $2$-layer denoising auto-encoder to realize the embedding function.
    In the white-box setting,
    we designing the embedding function to have access to both model and attack information.
    The training objective function of the embedding function at the $t^{th}$ layer is
    \begin{gather*}
        \mathcal{E}_t^{\ast} = \arg \min \limits_{\mathcal{E}_t}
        \max \limits_{\tilde{\mat{x}}_0 \sim \mathcal{B} (\mat{x}_0, \epsilon)}
        \frac{1} {N} \sum_{i=1}^N
        {\rm CE}(\tilde{\mat{x}}_{i, T}, \mat{q}_i)
        + \lVert \mathcal{E}_t (\tilde{\mat{x}}_t) - \mat{x}_t \rVert_2^2, \\
        {\rm s.t.}\; 
        \tilde{\mat{x}}_{t+1} = F_t \circ \mathcal{E}_t (\tilde{\mat{x}}_{t}),
        \;\;\;\; 
        \mat{x}_{t+1} = F_t (\mat{x}_{t}),
    \end{gather*}
    where $\mathcal{B}(\mat{x}, \epsilon)$ is a set of points centered at $\mat{x}$ with radius $\epsilon$ measured by the $\ell_{\infty}$ norm.
    The above objective function enforces the embedding function to encode an embedding manifold as a set of states that result in high robustness.
    The second term penalizes the magnitude of applied control adjustment.
\end{itemize}

\begin{table}[t]
\caption{Accuracy measure: baseline / controlled}
\begin{center}
\begin{tabular}{ p{3.2cm}|p{2.7cm}p{2.7cm}p{2.7cm}  }
\hline
\multicolumn{4}{c}{$\ell_{\infty}$ perturbation with $\epsilon = 8 / 255$}\\
\hline
& None & PGD & AA \\
\hline
RN-18 & \textbf{82.39} / 82.34 & 51.66 / \textbf{51.82} & 48.72 / \textbf{48.77} \\
\hline
RN-34 & \textbf{84.45} / 84.12 & 51.35 / \textbf{51.87} & 49.31 / \textbf{49.96} \\
\hline
RN-50 & \textbf{83.99} / 83.73 & 50.88 / \textbf{50.91} & 48.68 / \textbf{48.98} \\
\hline
\end{tabular}
\end{center}
\label{cifar10 white-box}
\end{table}

Although in practice, the control information is not released to the public, 
we evaluate the performance of the proposed framework in this worst case.
Table \ref{cifar10 white-box} shows that applying closed-loop control method can consistently improve the robustness of all adversarially pre-trained baseline models.
Specifically,
on average, $0.5 \% \sim 1 \%$ accuracy improvements have been achieved across all baseline models against auto-attack.
When the control information is released to an attacker, the robustness improvements are significantly decreased compared with that in the oblivious setting.
In this experiment, we consider simple $2$-layer convolutional auto-encoders.
In future works, we will employ embedding functions that have stronger expressive power, which allows the closed-loop control method to achieve better performance.

\newpage
\vskip 0.2in
\bibliography{jmlr}

\end{document}